%% file: main.tex
\pdfoutput=1
\documentclass[twoside,11pt]{article}

\usepackage{jmlr2e_modified}

\ShortHeadings{SHAP-Score:
Knowledge Compilation and Non-Approximability Results}{Arenas, Barcel\'o, Bertossi, \& Monet}

\firstpageno{1}

\usepackage{tikz}
\usetikzlibrary{shapes.geometric}
\usepackage{amsmath}
\usepackage{amssymb}
\usepackage{mathtools}
\usepackage{nicefrac}
\usepackage[linesnumbered,algoruled,commentsnumbered,boxruled]{algorithm2e}
\usepackage{enumitem}
\allowdisplaybreaks

\usepackage{hhline}

\input{macros}
\newtheorem{fact}[theorem]{Fact}
\newcommand{\rev}[1]{#1}
\newcommand{\todo}{\textcolor{red}{todo}}

\usepackage{lastpage}
\jmlrheading{24}{2023}{1-\pageref{LastPage}}{4/21; Revised
10/22}{3/23}{21-0389}{Marcelo Arenas, Pablo Barcel\'o, Leopoldo Bertossi, Mik\"ael Monet}
\ShortHeadings{Complexity of SHAP-Score-Based Explanations}{Arenas, Barcel\'o, Bertossi, Monet}

\begin{document}

\title{On the Complexity of SHAP-Score-Based Explanations:
Tractability via Knowledge Compilation and Non-Approximability Results}

\author{\name Marcelo Arenas \email marenas@ing.puc.cl \\
       \addr (a) Department of Computer Science \& Institute for Mathematical and Computational Engineering, \\
       School of Engineering \& Faculty of Mathematics, Universidad Cat\'olica de Chile \\ \addr (b) IMFD Chile\\
       \AND
       \name Pablo Barcel\'o \email pbarcelo@uc.cl  \\
       \addr (a) Institute for Mathematical and Computational Engineering, \\
       School of Engineering \& Faculty of Mathematics, Universidad Cat\'olica de Chile \\ (b) IMFD Chile (c) National Center for Artificial Intelligence, CENIA Chile \\
       \AND
       \name Leopoldo Bertossi \email leopoldo.bertossi@skema.edu \\
       \addr SKEMA Business School,  Montreal, Canada\\
       \AND
       \name Mika\"el Monet \email mikael.monet@inria.fr \\
       \addr Univ. Lille, Inria, CNRS, Centrale Lille, UMR 9189 - CRIStAL, F-59000 Lille, France}

\editor{}

\maketitle

\begin{abstract}%
\input{abstract}

\end{abstract}

\begin{keywords}
  Explainable AI, Shapley values, SHAP score, knowledge compilation, FPRAS
\end{keywords}

\section{Introduction}
\label{sec:intro}

\input{intro}

\section{Preliminaries}
\label{sec:preliminaries}
\input{preliminaries}

\section{Tractable Computation of the~$\shap$-Score}
\label{sec:shapscore-d-Ds}
\input{algo}

\section{Extension to Non-Binary Deterministic and Decomposable Circuits}
\label{sec:nonbinary}
\input{nonbinary.tex}

\section{Limits on the Tractable Computation of the $\shap$-Score}
\label{sec:limits}
\input{limits}

\section{Non-Approximability of the $\shap$-Score}
\label{sec:approx}

\input{approx}

\section{On Comparing the Values of the $\shap$-Score}
\label{sec:comparing}
\input{comparing}

\section{Final Remarks}
\label{sec:discussion}

\input{discussion}

\acks{Part of this work has been funded by ANID - Millennium Science
Initiative Program - Code ICN17002. \ M. Arenas has been funded by Fondecyt grant 1191337. \  P. Barcel\'o has been funded by Fondecyt Grant 1200967, and also  by the National Center for Artificial Intelligence CENIA FB210017, Basal ANID. \ L. Bertossi is a Senior Researcher at the IMFD, Chile, and a Professor Emeritus at Carleton University, Ottawa, Canada. }

\vskip 0.2in
\bibliography{main}

\appendix
\clearpage


\medskip

\section{ \ Encoding Binary Decision Trees and FBDDs into\\ \hspace*{3cm} Deterministic and Decomposable Boolean Circuits}
\label{app:KC} \vspace{2mm}
\input{app_KC}

\section{ \ Proof of a folklore fact}
\label{app:pras-distinguishes-zeros} \vspace{2mm}
\input{app_pras-distinguishes-zeros}

\section{ \ Proofs of Intermediate Results}\vspace{2mm}
\subsection{ \ Proof of Lemma~\ref{lem-one-p-e-one-m-e}}
\label{app:tech-lem}
\input{app_tech_lem}

\subsection{ \ Proof of Lemma~\ref{lem-diff-shap}}
\label{app:diff-shap} \vspace{2mm}
\input{app_diff_shap}

\subsection{ \ The summation in Lemma \ref{lem-graph-k-p}}
\label{app:sum-prod}
\input{app_sum_prod}

\end{document}

%% file: macros.tex
\newcommand\dom{\mathsf{dom}}

\newcommand\ssat{\mathsf{\#SAT}}

\newcommand\var{\mathsf{var}}

\newcommand\invpol[1]{\mathsf{InvPol}(#1)}

\newcommand\out{\mathsf{out}}

\renewcommand\H{\mathsf{H}}
\newcommand\diff{\mathsf{Diff}}

\newcommand{\C}{\mathcal{C}}
\newcommand{\A}{\mathcal{A}}
\newcommand{\B}{\mathcal{B}}
\newcommand{\hrulealg}[0]{\vspace{1mm} \hrule \vspace{1mm}}

\newcommand{\pr}{\mathrm{p}}
\newcommand{\prd}{\Pi_\pr}
\newcommand{\uni}{\mathcal{U}}

\newcommand*{\qed}{\null\nobreak\hfill\ensuremath{\square}}

\newcommand{\np}{\text{\rm NP}}
\newcommand{\conp}{\text{\rm coNP}}
\newcommand{\rp}{\text{\rm RP}}
\newcommand{\bpp}{\text{\rm BPP}}
\newcommand{\compshap}{\mathsf{CompSHAP}}
\newcommand{\rshap}{\mathsf{DistCompSHAP}}

\newcommand{\thpdnf}{\C_{\text{\rm 3-POS-DNF}}}
\newcommand{\twncnf}{\C_{\text{\rm 2-NEG-CNF}}}

\newcommand{\es}{\mathbf{e}}
\newcommand{\eset}{\mathsf{ent}}

\newcommand{\shap}{\mathsf{SHAP}}

\newcommand{\shp}{\#\text{\textsc{P}}}

\newcommand{\asm}{\mathsf{cw}}

\usetikzlibrary{arrows,positioning} 
\tikzset{
	rect/.style={
		rectangle,
		rounded corners,
		draw=black, 
		thick,
		text centered},
	rectw/.style={
		rectangle,
		rounded corners,
		draw=white, 
		thick,
		text centered},
	sq/.style={
		rectangle,
		draw=black, 
		thick,
		text centered},
	sqw/.style={
		rectangle,
		thick,
		text centered},
	arrout/.style={
		->,
		-latex,
		thick,
	},
	arrin/.style={
		<-,
		latex-,
		thick,
		El         },
	arrd/.style={
		<->,
		>=latex,
		thick,
	},
	arrw/.style={
		thick,
	}
}

\tikzset{
	circ/.style={
		circle,
		draw=black, 
		thick,
		text centered,
	},
	circw/.style={
		circle,
		draw=white, 
		thick,
		text centered,
	},
	arrout/.style={
		->,
		-latex,
		thick,
	},
	arrin/.style={
		<-,
		latex-,
		thick,
	},
	arrw/.style={
		-,
		thick,
	},
	arrww/.style={
		-,
		thick,
		draw=white, 
	}
}

\newcommand{\ignore}[1]{}

\newcommand{\feat}[1]{\mbox{\bf \footnotesize #1}}

\newcommand{\gclique}{\mathsf{GapClique}}
\newcommand{\one}{\text{\rm {\bf 1}}}
\newcommand{\sclique}{\mathsf{\#CLIQUE}}
\newcommand{\wit}{\mathsf{Wit}}

%% file: abstract.tex
Scores based on Shapley values are 
widely used for providing
explanations to classification results over machine learning models.  A prime
example of this is the influential~$\shap$-score, a version of the Shapley
value that can help explain the result of a learned model on a specific entity
by assigning a score to every feature.
While in general computing Shapley values is a computationally intractable
problem, we prove a strong positive result stating that the $\shap$-score can be
computed in polynomial time over \emph{deterministic and decomposable Boolean
circuits} under the so-called product distributions on entities.
Such
circuits are
studied in the field of Knowledge Compilation and generalize a
wide range of Boolean circuits and binary decision diagrams classes, including
binary decision trees, Ordered Binary Decision Diagrams (OBDDs) and Free Binary Decision Diagrams (FBDDs). 
Our positive result extends even beyond binary classifiers, as it continues to hold if each feature is associated with a finite domain of possible values. 

We also establish the computational limits of the notion of
SHAP-score by observing that, under a mild condition, computing it over a
class of Boolean models is always polynomially as hard as the model counting
problem for that class.
This implies that 
both determinism and decomposability are essential properties for the circuits that we
consider, as removing one or the other renders the problem of
computing the~$\shap$-score intractable (namely,~$\shp$-hard). It also implies that 
computing $\shap$-scores is $\shp$-hard even over the class of propositional formulas in DNF. 
Based on this negative result, we look for the existence of fully-polynomial randomized approximation schemes (FPRAS) 
for computing $\shap$-scores over such class. In stark contrast to the model counting problem for DNF formulas, which admits an FPRAS, we prove that  
no such FPRAS exists (under widely believed complexity assumptions) for the computation of $\shap$-scores. Surprisingly, this negative result holds even for the class of monotone formulas in DNF. These techniques can be further 
extended to prove another strong negative result: Under widely believed
complexity assumptions, there is no polynomial-time algorithm that checks, given a monotone DNF formula $\varphi$ and features $x,y$, whether the $\shap$-score of $x$ in $\varphi$ is smaller than the $\shap$-score of $y$ in $\varphi$.

%% file: intro.tex
\paragraph{Context.}
Explainable artificial intelligence
has become an active area of research. Central to it is the
observation
that artificial intelligence (AI) and machine learning (ML) models cannot always be blindly applied without being able to
interpret or explain their results.
For example, someone who applies for a loan
and sees
the application rejected
by
an algorithmic decision-making system would like to receive from the system an explanation for this decision. In ML, explanations have been commonly considered for classification algorithms, and there are different approaches.
In particular,
explanations can be {\em global} -- focusing on the general input/output relation of the model --, or {\em local}
-- focusing on how individual
features
affect
the decision of the model
for a specific input, as in the loan example above \citep{Ribeiro0G16,lundberg2017unified,DBLP:journals/corr/abs-1811-12615,BLSSV20}.
Recent literature has strengthened the importance of the latter by showing their ability to provide explanations
that are often overlooked by global explanations~\citep{molnar}. In this work we concentrate on local explanations.

One way to define local explanations is by considering feature values as players in a coalition game that jointly contribute to the outcome. More concretely, one treats a feature value's contribution from the viewpoint of game theory, and ties it to the question of how to properly distribute wealth (profit) among collaborating players. This problem was approached in general terms in \citet{shapley1953value}. One  can use the  established concepts and techniques  that he introduced in the context of cooperative game theory; and, more specifically, use the popular {\em Shapley value} as a measure of the contribution of a player to the common wealth associated with a multi-player game.
It is well known that the Shapley value possesses properties that cast it as natural and intuitive. Actually, the Shapley value emerges as the only function that enjoys those desirable properties~\citep{roth1988shapley}.
The Shapley value has been widely applied in different disciplines, in particular in computer science \citep{hunter2010measure,michalak2013efficient,cesari2018application,DBLP:conf/icdt/LivshitsBKS20,DBLP:journals/lmcs/LivshitsBKS21}; and in machine learning it has been applied to the explanation of classification results, in its incarnation as the  \emph{$\shap$-score} \citep{lundberg2017unified,lundberg2020local}. Here, the players are the feature values of an entity under classification.

In this paper, we concentrate on the \emph{$\shap$-score} for classification models.  It has a clear, intuitive, combinatorial meaning,
and inherits all the good properties of the Shapley value.
Accordingly, an explanation for a classification result takes the form of a set of feature values that have a high, hopefully maximum, \emph{$\shap$-score}.
We remark that $\shap$-scores have attracted the attention of the ML community and have found several applications and extensions  \citep{DBLP:journals/corr/abs-1906-09293,DBLP:conf/ijcnn/FidelBS20,BLSSV20,DBLP:conf/cdmake/MerrickT20,DBLP:journals/corr/abs-2004-04464,DBLP:journals/corr/abs-2012-01536,DBLP:conf/icml/KumarVSF20}.
However, its fundamental and computational properties have not been investigated much.

\paragraph{Problems studied in the paper.}
For a given
classification model~$M$, entity~$\es$ and feature~$x$,
the~$\shap$-score~$\shap(M,\es,x)$ intuitively represents the importance of the
feature value~$\es(x)$ to the classification result~$M(\es)$.  In its general
formulation,~$\shap(M,\es,x)$ is a weighted average of differences of expected
values of the outcomes
(c.f.~Section~\ref{sec:preliminaries} for its formal
definition).  Unfortunately, computing quantities that are based on the notion
of Shapley value is in general intractable. Indeed, in many scenarios the
computation turns out to be
$\shp$-hard~\citep{faigle1992shapley,deng1994complexity,DBLP:journals/lmcs/LivshitsBKS21,BLSSV20}, which makes the
notion difficult to use -- if not impossible -- for practical purposes \citep{AB09}. Therefore,
natural questions are: ``For what kinds of classification models the computation of the~$\shap$-score can be
done efficiently?", ``Can one obtain lower computational complexity if one has access to the internals of the classification model?", or again
“In cases where exact computation is intractable, can we efficiently approximate the~$\shap$-score?”.

In \citet{lundberg2020local} the claim is made that for certain models based on decision trees the computation of the $\shap$-score is tractable. In this work we go deeper into these results, in particular, formulating and establishing them in precise terms,  and extending them for a larger class of classification models. We also identify classes of models for which $\shap$-score computation is intractable. In such cases, we investigate the problem of existence and computation of a good approximation in the form of a  {\em fully polynomial-time randomized approximation scheme} (FPRAS). Recall that FPRAS are tractable procedures that return an answer that is, with high probability, close to the correct answer.

Given the high computational complexity of the $\shap$-score, one might try to
solve related problems, other than the exact and approximate computation of all
the features' scores, that could still be useful in practice.
For instance, we consider the problem that consists in deciding, for a pair of
feature values, which of the two has the highest score. This problem
could indeed be used to compute a ranking of (all or the highest) $\shap$-scores, without
computing them explicitly. We also address this
problem in this paper.

\paragraph{Model studied in the paper.}
We focus mainly on binary classifiers with \emph{binary
feature values} (i.e., propositional features that can take the values ~$0$ or~$1$), and
that return~$1$ (accept) or~$0$ (reject) for each entity.  We will call these \emph{Boolean classifiers}. The restriction to binary inputs can be relevant
in many practical scenarios where the features are of a propositional nature. Still, we consider classifiers with possibly non-binary features, but binary outcomes, in Section \ref{sec:nonbinary}.
The second assumption that we make is that the underlying probability
distributions on the population of entities are what we call \emph{product
distributions}, where each binary feature~$x$ has a probability~$\pr(x)$ of being
equal to~$1$, independently of the other feature values. This includes, as a special case, the \emph{uniform probability distribution}
when each~$\pr(x)$ is~$\frac{1}{2}$. Product distributions are also known as {\em fully-factorized} in the literature \citep{VdBAAAI21,van2022tractability}. They have received considerable attention
in the context of computing score-based explanations, as they combine good computational properties with enough flexibility to model relevant practical scenarios
\citep{StrumbeljK10,DBLP:conf/sp/DattaSZ16,lundberg2017unified}.

Positive results on the complexity of computation of $\shap$-scores in the paper are obtained for
Boolean classifiers defined as
\emph{deterministic and decomposable Boolean circuits}. This is
a widely studied model
in \emph{knowledge compilation}~\citep{darwiche2001tractability,DBLP:journals/jair/DarwicheM02}.
Such circuits encompass a wide range of Boolean circuits and binary
decision diagrams classes that are considered in knowledge compilation, and more generally in AI.  For instance, they generalize {\em binary decision trees}, {\em ordered
binary decision diagrams} (OBDDs), {\em free binary decision diagrams} (FBDDs), and
{\em deterministic and decomposable negation normal norms} (d-DNNFs)
\citep{darwiche2001tractability,ACMS20,DBLP:conf/ecai/DarwicheH20}.
These circuits are also known under the name of {\em tractable Boolean circuits},
that is used in recent
literature~\citep{shih2019verifying,DBLP:conf/kr/ShiSDC20,shih2018formal,shih2018symbolic,shih2019smoothing,peharz2020einsum}.
Readers who are not familiar with knowledge
compilation can simply think about deterministic and decomposable
circuits as a
tool for analyzing in a uniform manner the computational complexity
 of the $\shap$-score on several Boolean
classifier classes.

In turn, our negative results on the complexity of computation of $\shap$-scores in the paper are obtained over the class of propositional formulas in
DNF. In addition to being a well-known restriction of the class of propositional formulas for which the satisfiability problem is tractable, DNF formulas
define an
extension of deterministic and decomposable Boolean circuits. In fact, DNF formulas can be seen as decomposable, although not necessarily
deterministic, Boolean circuits.

\paragraph{Our results.}
Our main contributions are the following.

\begin{enumerate}

\item {\em Tractability for a large class of Boolean classifiers.}
We provide a polynomial time algorithm that computes the $\shap$-score for
deterministic and decomposable
Boolean circuits under product distributions over the entity population (Theorem~\ref{thm:shapscore-d-Ds-prod}).
We  obtain as a corollary that the $\shap$-score for
Boolean classifiers given as binary decision trees,
OBDDs, FBDDs and d-DNNFs
can
be computed in polynomial time.

\item {\em \rev{Tractability for non-binary classifiers.}}
We extend the aforementioned tractability result
to the case of classifiers with non-binary features, which take the form of {\em non-binary Boolean circuits} (Theorem~\ref{thm:shapscore-d-Ds-prod-nonbinary}). In these circuits, the nodes may now contain equalities of the form $x = v$, where $x$ is a feature, and $v$ is a value in the domain of $x$. The outcome of the classifier is still binary.

\item {\em Limits of tractability.}
We observe that, under a mild condition, computing the $\shap$-score on Boolean classifiers in a
class is always polynomially as hard as the {\em model counting}
problem
for that class (Lemma~\ref{lem:limits}).
This leads to intractability for the problem of computing the $\shap$-score for all classes of Boolean classifiers for which model counting is intractable.
An important example of this corresponds to the class of propositional formulas in~DNF.
As a corollary, we obtain that each one of the {\em
determinism} assumption and the {\em decomposability} assumption is necessary
for tractability (Theorem \ref{thm:shapscore-limits}). These results even hold for the uniform distribution.

\ignore{\item
We show that the results above (and most interestingly, the
polynomial-time algorithm) can be extended to the $\shap$-score defined on
product distributions for the entity population.  }

\item {\em Non-approximability for DNF formulas.}
We give a simple proof that, under
widely believed complexity assumptions, there is no
FPRAS for the computation of the $\shap$-score with Boolean classifiers
represented as DNF formulas (Proposition~\ref{prp:non-FPRAS-DNFs}).  This holds
even under the uniform distribution. This result establishes a stark contrast
with the model counting problem for DNF formulas, which admits an FPRAS
\citep{karp1989monte}.  We further strengthen this by showing that the
non-approximability result even holds for Boolean formulas represented as
2-POS-DNF, i.e., with every conjunct containing at most two positive literals (and no
negative literal),
and for the uniform distribution (Theorem~\ref{theo:non-FPRAS-2-POS-DNF}).
The proof of this last result is quite involved and is based on a
non-approximability result in
relation to the size of cliques in graphs \citep{FGLSS96,AS98,AMSS98}.

\item {\em Impossibility of comparing $\shap$-scores for DNF formulas.}
Consider the problem of verifying, given a DNF formula $\varphi$ and
features $x,y$, whether the $\shap$-score of $x$ in $\varphi$ is
smaller than the $\shap$-score of $y$ in $\varphi$. Under widely
believed complexity assumptions, we establish that this problem of
comparing two $\shap$-scores cannot be
\rev{approximated}
in polynomial time, even
for the case of monotone DNF formulas (Theorem~\ref{theo-comparing-bpp-np-rp}).

\end{enumerate}

\paragraph{Related work.}
Our contributions should be compared to the results obtained in
contemporaneous papers by Van den Broeck et al. \citep{VdBAAAI21,van2022tractability}. There, the authors
establish the following theorem: for every class~$\mathcal{C}$ of
classifiers and under product distributions, the problem of computing
the~$\shap$-score for~$\mathcal{C}$ is polynomial-time equivalent to the
problem of computing the expected value for the models in~$\mathcal{C}$ (at least under mild assumptions on $\cal C$).
Since computing expectations is in polynomial time for tractable
Boolean circuits, this in particular implies that computing
the~$\shap$-score is in polynomial time for the circuits that we
consider; in other words, their results capture our main positive result. However, there
is a fundamental difference in the approach taken to show
tractability: their reduction uses multiple oracle calls to the
problem of computing expectations, whereas we provide a more direct
algorithm to compute the~$\shap$-score on these circuits.

The exact computation or
approximation of the Shapley value applied to database tuples is investigated in \citet{DBLP:journals/lmcs/LivshitsBKS21,deutch2022computing}.  In \citet{BLSSV20}, approximations of the $\shap$-score are used for experimental purposes and comparisons with other scores, such as RESP and Rudin's FICO-score \citep{DBLP:journals/corr/abs-1811-12615}. However, an empirical distribution is used for the approximate computation of the $\shap$-score, which leads to a simple, non-probabilistic weighted average, and to the restriction of the counterfactual versions of an entity to those in the available sample.  In \citet{DBLP:journals/corr/abs-2303-06516}, the algorithm presented in our work (c.f. Section \ref{subsec:time}) has been used to efficiently compute $\shap$-scores, under the uniform distribution,  for outcomes from binary neural networks, after compiling the latter into deterministic and decomposable Boolean circuits.

\rev{Building on~\cite{DBLP:journals/lmcs/LivshitsBKS21} and on the conference
version of the current article, the authors of~\citet{deutch2022computing}
use knowledge compilation techniques that are similar to ours to
develop a polynomial-time algorithm that is able to compute a version of the Shapley value
tailored to a database context. They moreover propose an algorithm for
computing approximations of the Shapley values of tuples, but this algorithm does not come with
any theoretical guarantees.}

\rev{Last, regarding approximation, there is a large body of work that aims at
approximating Shapley values using Monte Carlo techniques, see for
instance~\cite{castro2009polynomial,okhrati2021multilinear,deutch2021explanations,mitchell2022sampling}.
The results of such works are upper bounds for diverse settings of Shapley
values -- i.e., considering different game functions -- but we are not aware of
such results for the specific game function (the SHAP-score) that we study
here, for instance, theoretical results that would yield an FPRAS for the
SHAP-score.  Besides, these studies do not usually contain lower bounds,
since their focus is on obtaining tractable approximations via Monte Carlo or other
sampling schemes. By contrast, obtaining lower bounds is the topic of our
Sections~\ref{sec:approx} and~\ref{sec:comparing}, where we show the
non-existence of an FPRAS for approximating the SHAP-score of DNF formulas, as well as
the non-membership in~\bpp~of comparing SHAP-scores for such
formulas.}

This paper is a considerable extension of the conference
paper \citep{ABBM21}. In addition to containing full proofs of all the
results of \citep{ABBM21}, we present here several
new results. Among them, we provide a detailed analysis of approximability of the $\shap$-score, a complexity analysis of the problem of comparing $\shap$-scores, and an extension of the results for deterministic and decomposable Boolean circuits to the case of non-binary features.

\paragraph{Paper structure.}
We give preliminaries in Section~\ref{sec:preliminaries}.
In
Section~\ref{sec:shapscore-d-Ds}, we prove that~the $\shap$-score can be computed
in polynomial time for deterministic and decomposable Bool\-ean circuits under product probability distributions.
We extend this tractability result in Section \ref{sec:nonbinary} to non-binary deterministic and decomposable Boolean circuits.
In
Section~\ref{sec:limits}, we establish the computational limits of the exact computation of the~$\shap$-score, while Section~\ref{sec:approx} studies
the (non-) approximability properties of this score. The problem of comparing the $\shap$-scores of different features is studied in Section \ref{sec:comparing}.
We conclude and discuss future work in Section~\ref{sec:discussion}.

%% file: preliminaries.tex
\label{sec:prelim}

\subsection{Entities, distributions and classifiers}
Let~$X$ be a finite set of \emph{binary\footnote{We will come back to this assumption in Section~\ref{sec:nonbinary}, where we will consider non-binary features.} features}, also called
\emph{variables}. An \emph{entity} over~$X$ is a function~$\es:X \to
\{0,1\}$.\footnote{\rev{Equivalently, one can see an entity as a vector of binary values, with each coordinate corresponding to a given feature. We will however always use the functional point of view as it simplifies the notation in our proofs.}}
We denote by~$\eset(X)$ the set of all entities over~$X$.
On this set, we consider probability distributions that we call \emph{product distributions}, defined as follows. Let~$\pr:X \to [0,1]$ be a function that associates to
every feature~$x \in X$ a probability value~$\pr(x) \in [0,1]$.
Then, the
\emph{product distribution generated by~$\pr$} is the
probability distribution~$\prd$ over~$\eset(X)$ such that, for every~$\es\in
\eset(X)$ we have
\[\prd(\es) \ \ \coloneqq \ \ \bigg(\prod_{\substack{x\in X\\ \es(x)=1}} \pr(x)\bigg) \cdot \bigg(\prod_{\substack{x\in X\\ \es(x)=0}} (1-\pr(x))\bigg).\]
That is,~$\prd$ is the product distribution that is determined by pre-specified marginal distributions, and that makes the features take values independently from each other.
We denote by~$\uni$ the \emph{uniform probability distribution}, i.e., for every entity $\es \in \eset(X)$, we have that $\uni(\es) := \frac{1}{2^{|X|}}$.
Note that the uniform distribution can be obtained as a special case of product distribution, with~$\prd$ invoking~$\pr(x) := \nicefrac{1}{2}$ for every~$x\in X$.

A \emph{Boolean classifier}~$M$ over~$X$ is
a
function~$M : \eset(X) \to \{0,1\}$ that maps every entity over~$X$
to~$0$ or~$1$.  We say that~$M$ \emph{accepts} an entity~$\es$
when~$M(\es)=1$, and that it \emph{rejects} it if~$M(\es)=0$. Since we consider $\eset(X)$ to be a probability space, $M$ can be regarded
as a random variable.

\subsection{The $\mathbf{\shap}$-score over Boolean classifiers}
\label{subsec:shapdef}

Let $M : \eset(X) \to \{0,1\}$ be a Boolean classifier over the set~$X$ of features.  Given an entity~$\es$ over~$X$ and a
subset~$S \subseteq X$ of features, we define the set
$\asm(\es, S)$ of entities that are \emph{consistent with $\es$ on $S$} as
$\asm(\es, S) \coloneqq \{ \es' \in \eset(X) \mid \es'(x) =\es(x) $ for each~$x \in S\}$.
Then, given an entity~$\es \in
\eset(X)$, a probability distribution~$\mathcal{D}$ over~$\eset(X)$, and~$S \subseteq X$, we define the {\em
expected value of $M$ over~$X \setminus S$ with respect to~$\es$ under~$\mathcal{D}$} as
\begin{align*}
\phi_\mathcal{D}(M,\es,S) \ &\coloneqq \
\mathbb{E}_{\es' \sim \mathcal{D}}\big[M(\es') \mid \es'\ \in \asm(\es, S) \big].
\end{align*}
In other words,~$\phi_\mathcal{D}(M,\es,S)$ is
the expected value of~$M$, conditioned on the
inputs to coincide with~$\es$ over each feature
in~$S$.
For instance,
if we take~$\mathcal{D}$ to be the uniform distribution~$\uni$ over~$\eset(X)$,
this expression simplifies to
\begin{align*}
\phi_\uni(M,\es,S)= \sum_{\es' \in \asm(\es,S)} \frac{1}{2^{|X\setminus S|}} M(\es').
\end{align*}
The function~$\phi_\mathcal{D}$ is then used in the general formula of the Shapley value \citep{shapley1953value,roth1988shapley} to obtain the $\shap$-score for feature values in
$\es$, as follows.
\begin{definition}[\cite{lundberg2017unified}]
\label{def:Shapley}
Given a Boolean classifier~$M$ over a set of features~$X$, a probability distribution~$\mathcal{D}$ on~$\eset(X)$, an entity~$\es$
over~$X$, and a feature~$x \in X$, the {\em $\shap$ score of
feature~$x$ on~$\es$ with respect to~$M$ under~$\mathcal{D}$} is defined as
\begin{multline}\label{eq:shapscoredef}
\shap_\mathcal{D}(M,\es,x) \ \coloneqq \ \sum_{S \subseteq X\setminus\{x\}}
\frac{|S|! \, (|X| - |S| - 1)!}{|X|!} \bigg(
\phi_\mathcal{D}(M, \es,S \cup \{x\}) - \phi_\mathcal{D}(M, \es,S)\bigg).
\end{multline}
\end{definition}
In Section~\ref{sec:limits}, we will use another equivalent expression of the~$\shap$-score, that we introduce now.
For a permutation~$\pi:X\to \{1,\ldots,n\}$
and~$x\in X$, let~$S_\pi^x$ denote the set of features that appear
before~$x$ in~$\pi$.  Formally, $S_\pi^x \coloneqq \{y \in
X \mid \pi(y) < \pi(x)\}$. Then, letting~$\Pi(X)$ be the set of all
permutations~$\pi:X\to \{1,\ldots,n\}$, observe that Equation~\eqref{eq:shapscoredef} can be rewritten~as
\begin{align}
\label{def:Shapley-alt}
 \shap_\mathcal{D}(M,\es,x) \ = \ \frac{1}{|X|!}\sum_{\pi \in \Pi(X)}  \big(\phi_\mathcal{D}(M,\es,S_\pi^x \cup \{x\}) - \phi_\mathcal{D}(M,\es,S_\pi^x)\big).
\end{align}
Thus, $\shap_\mathcal{D}(M,\es,x)$ is a weighted average of the contribution of feature $x$ on $\es$
to the classification result, i.e., of the differences between having it and
not, under all possible permutations of the other feature
values.
Observe that, from this definition, a high positive value
of~$\shap_\mathcal{D}(M,\es,x)$ intuitively means that setting~$x$ to~$\es(x)$
strongly leans the classifier towards acceptance, while a high
negative value of~$\shap_\mathcal{D}(M,\es,x)$ means that setting~$x$ to~$\es(x)$
strongly leans the classifier towards rejection.

\subsection{Deterministic and decomposable Boolean circuits}
\label{subsec:circuitsprelims}

A Boolean circuit over a set of variables~$X$ is a directed
acyclic graph~$C$ such that
\begin{enumerate}
\item[(i)] Every node without incoming edges is either a {\em variable
  gate} or a {\em constant gate}. A variable gate is labeled with a
  variable from~$X$, and a constant gate is labeled with either~$0$ or~$1$;

\item[(ii)] Every node with incoming edges is a {\em
  logic gate}, and is labeled with a symbol~$\land$,~$\lor$
  or~$\lnot$. If it is labeled with the symbol~$\lnot$, then
  it has exactly one incoming edge;\footnote{Recall that the fan-in of a gate is the number of its input gates. In our definition of Boolean circuits, we allow
    unbounded fan-in~$\land$- and~$\lor$-gates.}

\item[(iii)] Exactly one node does not have any outgoing edges, and
  this node is called the {\em output gate of~$C$}.
\end{enumerate}
Such a Boolean circuit~$C$ represents a Boolean classifier in the expected way -- we assume the reader to be familiar with Boolean logic --,
and we write~$C(\es)$ for the value in~$\{0,1\}$ of the output gate of~$C$ when we evaluate~$C$ over the entity~$\es$.
We consider the \emph{size}~$|C|$ of the circuit to be its number of edges. Observe that, thanks to condition (iii), the number of gates of~$C$ is at most
its number of edges plus one.

Several restrictions of Boolean circuits with good computational
properties have been studied. Let~$C$ be a Boolean circuit over a set
of variables~$X$ and~$g$ a gate of~$C$. The Boolean circuit~$C_g$
over~$X$ is defined by considering the subgraph of~$C$ induced by the
set of gates~$g'$ in~$C$ for which there exists a path from~$g'$
to~$g$ in~$C$. Notice that~$g$ is the output gate of~$C_g$. Then, an~$\lor$-gate~$g$ of~$C$ is said to be
\emph{deterministic} if for every pair~$g_1$,~$g_2$ of distinct input
gates of~$g$, the Boolean circuits~$C_{g_1}$ and~$C_{g_2}$ are
disjoint in the sense that there is no entity~$\es$ that is accepted
by both~$C_{g_1}$ and~$C_{g_2}$ (that is, there is no entity~$\es \in
\eset(X)$ such that~$C_{g_1}(\es) = C_{g_2}(\es) = 1$). The
circuit~$C$ is called \emph{deterministic} if every~$\lor$-gate of~$C$
is deterministic. The
set~$\var(g)$ is defined as the set of variables~$x \in X$ such that
there exists a variable gate with label~$x$ in~$C_g$.  An~$\land$-gate~$g$ of~$C$ is said to be
\emph{decomposable}, if for every pair~$g_1$,~$g_2$ of distinct input
gates of~$g$, we have that~$\var(g_1) \cap \var(g_2) =
\emptyset$. Then $C$ is called \emph{decomposable} if every
$\land$-gate of~$C$ is decomposable.

\begin{example} \label{ex:class}
{\em We want to classify papers submitted to a conference as rejected
(Boolean value $0$) or accepted (Boolean value $1$). Papers are
described by propositional features \feat{fg}, \feat{dtr}, \feat{nf} and \feat{na},
which stand for ``follows guidelines", ``deep theoretical result", ``new
framework" and ``nice applications", respectively.  The Boolean classifier
for the papers is given by the Boolean circuit in
Figure~\ref{fig:ddbc-exa}. The input of this circuit are the
features \feat{fg}, \feat{dtr}, \feat{nf} and \feat{na}, each of which
can take value either $0$ or $1$, depending on whether the feature is
present~($1$) or absent~($0$). The nodes with labels~$\neg$, $\lor$ or~$\land$ are logic gates, and the associated Boolean value of
each one of them depends on the logical connective represented by its
label and the Boolean values of its inputs. The output
value
of the circuit is given by
the top node in the figure.

The Boolean circuit in
Figure~\ref{fig:ddbc-exa} is decomposable, because for each
$\land$-gate the sets of features of its inputs are pairwise
disjoint. For instance, in the case of the top node in
Figure~\ref{fig:ddbc-exa}, the left-hand side input has
$\{\feat{fg}\}$ as its set of features, while its right-hand side
input has $\{\feat{dtr}, \feat{nf}, \feat{na}\}$ as its set of
features, which are disjoint. Also, this circuit is deterministic as
for every $\lor$-gate two (or more) of its inputs cannot be given
value 1 by the same Boolean assignment for the features. For instance,
in the case of the only $\lor$-gate in Figure~\ref{fig:ddbc-exa}, if a
Boolean assignment for the features gives value 1 to its left-hand side
input, then feature \feat{dtr} has to be given value 1 and, thus, such
an assignment gives value $0$ to the right-hand side input of the~$\lor$-gate. In the same way, it can be shown that if a
Boolean assignment for the features gives value 1 to the right-hand side input of this $\lor$-gate, then it gives value $0$ to its left-hand side input.
\qed}
\end{example}

\begin{figure}
\begin{center}
\begin{center}
\begin{tikzpicture}
  \node[circ, minimum size=7mm, inner  sep=-2] (n1) {\feat{dtr}};
  \node[circ, right=12mm of n1, minimum size=7mm] (n2) {\feat{nf}};
  \node[circ, above=1.3mm of n2, minimum size=7mm] (nneg) {$\neg$}
    edge[arrin] (n1);
  \node[circ, right=12mm of n2, minimum size=7mm] (n3) {\feat{na}};
  \node[circ, above=10mm of n3, minimum size=7mm] (n4) {$\land$}
  edge[arrin] (nneg)
  edge[arrin] (n2)
  edge[arrin] (n3);
  \node[circ, above=18mm of n2, minimum size=7mm] (n5) {$\lor$}
  edge[arrin] (n4)
  edge[arrin] (n1);
  \node[circ, above=27mm of n1, minimum size=7mm] (n6) {$\land$}
  edge[arrin] (n5);
  \node[circw, left=12mm of n5, minimum size=7mm] (n5a) {};
  \node[circ, left=12mm of n5a, minimum size=7mm, inner sep=-2] (n0) {\feat{fg}}
  edge[arrout] (n6);
\end{tikzpicture}
\end{center}
\caption{A deterministic and decomposable Boolean Circuit as a classifier. \label{fig:ddbc-exa}}
\end{center}
\end{figure}
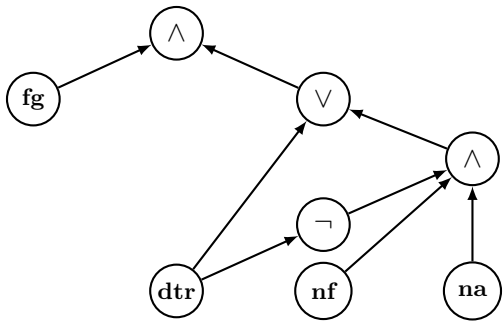

We will use the fact that deterministic and decomposable Boolean circuits are closed under \emph{conditioning}. Let~$C$ be a Boolean circuit
over variables~$X$, and let~$x \in X$. We denote by~$C_{+x}$ (resp.,~$C_{-x}$) the Boolean circuit that is obtained from~$C$
by replacing every variable gate that is labeled with~$x$ by a constant~$1$-gate (resp, by a constant~$0$-gate).
In the literature, $C_{+x}$ (resp., $C_{-x}$) is called the
conditioning by $x$ (resp., by $\lnot x$) of $C$.
It is easy to check that, if~$C$ is deterministic and decomposable, then so are~$C_{+x}$ and~$C_{-x}$.

As mentioned in the introduction,
deterministic and decomposable Boolean circuits
generalize many decision diagrams and Boolean circuits classes.
We refer
to~\citep{darwiche2001tractability,ACMS20} for detailed studies of
knowledge compilation classes and of their precise relationships.
For the reader's convenience, we explain in Appendix~\ref{app:KC}
how FBDDs and binary decision trees can be encoded in linear time as
deterministic and decomposable Boolean circuits.

\subsection{Complexity classes \rev{and encoding of probability values}}
\label{subsec:compl}

In Section~\ref{sec:limits}, we will consider the counting complexity
class \shp~\citep{V79} of problems that can be
expressed as the number of accepting paths of a nondeterministic
Turing machine running in polynomial
time. Following~\cite{V79}, we
define \#P-hardness using Turing reductions. Prototypical examples
of \#P-complete problems are counting the number of assignments that
satisfy a propositional formula and counting the number of
three-colorings of a graph. While it is known
that~$\mathrm{FPTIME} \subseteq \shp$, where $\mathrm{FPTIME}$ is the
class of functions that can be computed in polynomial time, this
inclusion is widely believed to be strict. Therefore, proving that a
problem is \shp-hard implies, under such an assumption, that it cannot be solved in polynomial
time.

In Sections~\ref{sec:approx} and~\ref{sec:comparing} we will use the complexity classes~$\rp$ and~$\bpp$.
Recall that~$\rp$ is the class of decision problems~$L$ for which there exists a polynomial-time probabilistic Turing Machine~$M$ such that: (a) if~$x \in L$, then~$M$ accepts with probability at least~$\nicefrac{3}{4}$; and (b) if~$x \not\in L$, then~$M$ does not accept~$x$. Moreover,~$\bpp$ is defined exactly as~$\rp$ but with condition (b) replaced by: (b') if~$x \not\in L$, then~$M$ accepts with probability at most~$\nicefrac{1}{4}$. Thus,~$\bpp$ is defined as~$\rp$ but allowing errors for both the elements that are and are not in~$L$. Hence,~$\mathrm{PTIME} \subseteq \rp \subseteq \bpp$ by definition. Besides, it is known that~$\rp \subseteq \np$, and this inclusion is widely believed to be strict. Finally, it is not known whether~$\bpp \subseteq \np$ or~$\np \subseteq \bpp$, but it is widely believed that~$\np$ is not included in~$\bpp$.

\rev{Finally, when considering problems where probabilities can be part of the
input (such as in Theorem~\ref{thm:shapscore-d-Ds-prod} or
Theorem~\ref{thm:shapscore-d-Ds-prod-nonbinary}), it will always be implicit
that such probabilities are given as rational numbers~$\nicefrac{p}{q}$
for~$p,q\in \mathbb{N}$, encoded as ordered pairs~$(p,q)$ where~$p$ and~$q$ are
themselves encoded in binary.}

%% file: algo.tex
In this section we prove our main tractability result, namely, that
computing the~$\shap$-score for Boolean classifiers given as
deterministic and decomposable Boolean circuits can be done in
polynomial time for product probability distributions.

\begin{theorem}
\label{thm:shapscore-d-Ds-prod}
Given as input a
deterministic and decomposable circuit~$C$ over a set of features~$X$, rational
probability values~$\pr(x)$ for every feature~$x\in X$, an entity~$\es:X\to
\{0,1\}$, and a feature~$x\in X$, the value~$\shap_{\prd}(C,\es,x)$ can be computed in polynomial time.
\end{theorem}

In particular, since binary decision trees, OBDDs, FBDDs and d-DNNFs are all
restricted types of deterministic and decomposable circuits, we obtain as a
consequence of Theorem~\ref{thm:shapscore-d-Ds-prod} that this problem is also in
polynomial time for these classes. For instance, for binary decision trees we
obtain the following result.

\begin{corollary}
\label{cor:shapscore-decision-trees}
Given as input a binary decision tree~$T$ over a set of features~$X$, rational
probability values~$\pr(x)$ for every feature~$x\in X$,
an entity~$\es:X\to \{0,1\}$, and a feature~$x\in X$, the value~$\shap_{\prd}(T,\es,x)$ can be computed in polynomial time.
\end{corollary}

The authors of \citep{lundberg2020local} provide an algorithm for computing the $\shap$-score in polynomial time for decision trees, but, unfortunately, as stated in
\citep{VdBAAAI21},
this algorithm is not correct.
Moreover,
it is important to notice that Theorem~\ref{thm:shapscore-d-Ds-prod}
is a nontrivial extension of the result for decision trees, as it is known that
deterministic and decomposable circuits can be exponentially more succinct
than binary decision trees (in fact, than FBDDs)
at representing Boolean
classifiers~\citep{darwiche2001tractability,ACMS20}.

We prove Theorem~\ref{thm:shapscore-d-Ds-prod} in the remaining of this section.
First, we prove in Section~\ref{subsec:proof} that the problem can
be solved in polynomial time, and we extract from this
proof a first version of our
algorithm.
Then, in Section~\ref{subsec:final-algo}, we provide an optimized version of this algorithm, and argue why the proposed optimization reduces the complexity of the algorithm.

\subsection{Polynomial time computability proof}
\label{subsec:proof}

In this section we prove that the SHAP-score can be computed in polynomial time.
For a Boolean classifier~$M$ over a set of variables~$X$, probability distribution~$\mathcal{D}$ over~$\eset(X)$,
entity~$\es\in \eset(X)$, and natural number~$k \leq |X|$, we define the quantity
\[\H_{\mathcal{D}}(M,\es, k) \ \coloneqq \ \sum_{\substack{S \subseteq X\\|S|=k}} \, \, \, \mathbb{E}_{\es' \sim \mathcal{D}}[M(\es') \mid \es' \in \asm(\es,S)].\]

Our proof of
Theorem~\ref{thm:shapscore-d-Ds-prod} is
divided into two modular parts. The first part, which is developed in
Section \ref{subsubsec:shap-to-H}, consists in showing that the
problem of computing~$\shap_{\Pi_\cdot}(\cdot,\cdot,\cdot)$ can be reduced in
polynomial time to that of computing~$\H_{\Pi_\cdot}(\cdot,\cdot,\cdot)$. This
part of the proof is a sequence of formula manipulations, and it only
uses the fact that deterministic and decomposable circuits are closed under conditioning on a variable value (recall the definition of conditioning from Section~\ref{subsec:circuitsprelims}).
In the second part of the proof, which is developed in
Section \ref{subsubsec:H}, we show that
computing~$\H_{\Pi_\cdot}(\cdot,\cdot,\cdot)$ can be done in polynomial time
for deterministic and decomposable Boolean circuits.
It is in this part that the properties of
deterministic and decomposable circuits are  used. Finally, we extract Algorithm \ref{algo:main} from this proof in Section \ref{subsec:time}.

\subsubsection{Reducing in polynomial-time from~$\shap_{\Pi_\cdot}(\cdot,\cdot,\cdot)$ to~$\H_{\Pi_\cdot}(\cdot,\cdot,\cdot)$}
\label{subsubsec:shap-to-H}

In this section we show that for deterministic and decomposable
Boolean circuits, and under product distributions, the computation of the~$\shap$-score can be reduced
in polynomial time to the computation
of~$\H_{\Pi_\cdot}(\cdot,\cdot,\cdot)$.

Suppose then that we wish to compute~$\shap_{\Pi_p}(C,\es,x)$, for a given deterministic and decomposable
circuit~$C$ over a set of variables~$X$, probability mapping~$p:X\to [0,1]$,
entity~$\es\in \eset(X)$, and feature~$x\in X$. Define
\[\diff_k(C,\es,x) \ \coloneqq \ \sum_{\substack{S\subseteq X\setminus \{x\}\\|S|=k}} (\phi_{\Pi_p}(C,\es,S\cup \{x\}) - \phi_{\Pi_p}(C,\es,S)),\]
and let~$n = |X|$. We then have, by definition, that
\begin{align}
\shap_{\Pi_p}(C,\es,x) \ &= \ \sum_{S \subseteq X\setminus \{x\}} \frac{|S|!(n-|S|-1)!}{n!}(\phi_{\Pi_p}(C,\es,S\cup \{x\}) - \phi_{\Pi_p}(C,\es,S))\nonumber \\
&= \ \sum_{k=0}^{n-1} \ \sum_{\substack{S\subseteq X\setminus \{x\}\\|S|=k}} \frac{k!(n-k-1)!}{n!}(\phi_{\Pi_p}(C,\es,S\cup \{x\}) - \phi_{\Pi_p}(C,\es,S))\nonumber
\\
&= \ \sum_{k=0}^{n-1} \ \frac{k!(n-k-1)!}{n!} \diff_k(C,\es,x). \label{eq:shap-to-diff}
\end{align}
Observe that all arithmetical terms (such as $k!$ or $n!$) can
be computed in polynomial time: this is simply because $n$
is given in unary, as it is bounded by the size of the circuit.
Therefore, it is good enough to show how to compute in polynomial time the
quantities~$\diff_k(C,\es,x)$ for each~$k\in \{0,\ldots,n-1\}$. By definition
of~$\phi_{_{\Pi_\cdot}}(\cdot,\cdot,\cdot)$ we have that
\begin{align}
\nonumber
\diff_k(C,\es,x) \ =& \ \bigg[\sum_{\substack{S\subseteq X\setminus \{x\}\\|S|=k}} \mathbb{E}_{\es' \sim \Pi_p}[C(\es') \mid \es' \in \asm(\es,S\cup \{x\})]\bigg]\\
& \hspace{100pt} - \bigg[\sum_{\substack{S\subseteq X\setminus \{x\}\\|S|=k}} \mathbb{E}_{\es' \sim \Pi_p}[C(\es') \mid \es' \in \asm(\es,S)]\bigg]. \label{eq:diffk-H}
\end{align}
\noindent In this expression, let~$L$ and~$R$ be the left- and
right-hand side terms in the subtraction. Looking closer at~$R$, we
have by standard properties of conditional expectation that
\begin{align*}
R \ =& \ \sum_{\substack{S\subseteq X\setminus \{x\}\\|S|=k}} \mathbb{E}_{\es' \sim \Pi_p}[C(\es') \mid \es' \in \asm(\es,S)]\\
=& \ \, \pr(x) \cdot \sum_{\substack{S\subseteq X\setminus \{x\}\\|S|=k}} \mathbb{E}_{\es' \sim \Pi_p}[C(\es') \mid \es' \in \asm(\es,S) \text{ and $\es'(x)=1$}]\\
& \hspace{80pt} + (1-\pr(x)) \cdot \sum_{\substack{S\subseteq X\setminus \{x\}\\|S|=k}} \mathbb{E}_{\es' \sim \Pi_p}[C(\es') \mid \es' \in \asm(\es,S) \text{ and $\es'(x)=0$}].
\end{align*}
To continue, we need to introduce some notation. For a set of features~$X$
and~$S \subseteq X$, we write~$\pr_{|S}:S\to [0,1]$ for
the mapping that is the restriction of~$\pr$ to~$S$, and~$\Pi_{\pr_{|S}}: \eset(S)
\to [0,1]$ for the corresponding product distribution on~$\eset(S)$.
Similarly, for an entity~$\es \in \eset(X)$
and~$S\subseteq X$, let~$\es_{|S}$ be the entity over~$S$ that is
obtained by restricting~$\es$ to the domain~$S$ (that
is, formally~$\es_{|S} \in \eset(S)$ and~$\es_{|S}(y) \coloneqq \es(y)$ for
every~$y \in S$).
Now, remembering from Section~\ref{subsec:circuitsprelims} the definition of~$C_{+x}$ and~$C_{-x}$, we obtain that
\begin{align}
\nonumber
R \ =& \ \, p(x) \cdot \sum_{\substack{S\subseteq X\setminus \{x\}\\|S|=k}} \mathbb{E}_{\es'' \sim \Pi_{p_{|X\setminus \{x\}}}} [C_{+x}(\es'') \mid \es'' \in \asm(\es_{|X\setminus \{x\}},S)]\\
& \hspace{50pt} + (1-p(x)) \cdot \sum_{\substack{S\subseteq X\setminus \{x\}\\|S|=k}} \mathbb{E}_{\es'' \sim \Pi_{p_{|X\setminus \{x\}}}} [C_{-x}(\es'') \mid \es'' \in \asm(\es_{|X\setminus \{x\}},S)]\nonumber\\
=& \ \, p(x) \cdot \H_{\Pi_{p_{|X\setminus \{x\}}}}(C_{+x},\es_{|X\setminus \{x\}},k) \, + \, (1-p(x)) \cdot \H_{\Pi_{p_{|X\setminus \{x\}}}}(C_{-x},\es_{|X\setminus \{x\}},k), \label{eq:R}
\end{align}

\noindent where the last equality is obtained simply by using the definition
of~$\H_{\Pi_\cdot}(\cdot,\cdot,\cdot)$. Hence, if we could compute in polynomial
time~$H_{\Pi_\cdot}(\cdot,\cdot,\cdot)$ for deterministic and decomposable Boolean
circuits, then we could compute~$R$ in polynomial time as~$C_{+x}$
and~$C_{-x}$ can be computed in linear time from~$C$, and they
are deterministic and decomposable Boolean circuits as well.

We now inspect the term~$L$, which we recall is
\begin{align*}
L \ =& \ \sum_{\substack{S\subseteq X\setminus \{x\}\\|S|=k}} \mathbb{E}_{\es' \sim \Pi_p}[C(\es') \mid \es' \in \asm(\es,S\cup \{x\})].
\end{align*}

\noindent Observe that, for~$S\subseteq X\setminus \{x\}$ and~$\es' \in \asm(\es,S\cup \{x\})$, it holds that
\begin{equation*}
C(\es') \ =
\begin{cases}
C_{+x}(\es'_{|X\setminus\{x\}}) & \text{if } \es(x)=1\\
C_{-x}(\es'_{|X\setminus\{x\}}) & \text{if } \es(x)=0
\end{cases}.
\end{equation*}

\noindent Therefore, if~$\es(x)=1$, we have that
\begin{align}
\nonumber
L \ =& \ \sum_{\substack{S\subseteq X\setminus \{x\}\\|S|=k}} \mathbb{E}_{\es'' \sim \Pi_{p_{|X\setminus \{x\}}}}[C_{+x}(\es'') \mid \es'' \in \asm(\es_{|X\setminus \{x\}},S)]\\
 =& \ \ \H_{\Pi_{p_{|X\setminus \{x\}}}}(C_{+x},\es_{|X\setminus \{x\}},k) \label{eq:L-one}
\end{align}
whereas if~$\es(x)=0$, we have that
\begin{align}
\label{eq:L-zero}
L \ = \ \H_{\Pi_{p_{|X\setminus \{x\}}}}(C_{-x},\es_{|X\setminus \{x\}},k).
\end{align}
Hence, again, if we were able to compute in polynomial
time~$\H_{\Pi_\cdot}(\cdot,\cdot,\cdot)$ for deterministic and decomposable Boolean
circuits, then we could compute~$L$ in polynomial time (as
deterministic and decomposable Boolean circuits~$C_{+x}$ and~$C_{-x}$
can be computed in linear time from~$C$).
But then we deduce from~\eqref{eq:diffk-H} that~$\diff_k(C,\es,x)$ could
be computed in polynomial time for each~$k\in \{0,\ldots,n-1\}$, from
which we have that~$\shap_{\Pi_p}(C,\es,x)$ could be computed in polynomial
time (by Equation~\eqref{eq:shap-to-diff}), therefore concluding the existence of the reduction claimed in this section.

\subsubsection{Computing~$\H_{\Pi_\cdot}(\cdot,\cdot,\cdot)$ in polynomial time}
\label{subsubsec:H}

We now take care of the second part of the proof of
tractability, i.e., proving that
computing~$\H_{\Pi_\cdot}(\cdot,\cdot,\cdot)$ for deterministic and
decomposable Boolean circuits can be done in polynomial time. Formally, in this section we prove the following lemma:

\begin{lemma}
\label{lem:H}
The following problem can be solved in polynomial time. Given as input
a deterministic and decomposable Boolean circuit~$C$ over
a set of variables~$X$, rational probability values~$\pr(x)$ for each~$x\in X$,
an entity~$\es\in \eset(X)$, and a natural number ~$k \leq |X|$, compute
the quantity~$\H_{\Pi_p}(C,\es,k)$.
\end{lemma}

We first perform two preprocessing steps on~$C$.

\begin{description}
    \item[Rewriting to fan-in~$2$.] First, we modify the
		circuit~$C$ so that the fan-in of every~$\lor$-
		and~$\land$-gate is exactly~$2$. This can be
		done in linear time simply by rewriting every~$\land$-gate (resp.,
		and~$\lor$-gate) of fan-in~$m > 2$ with a chain
		of~$m-1$ $\land$-gates (resp.,~$\lor$-gates) of
		fan-in~$2$, and by attaching to each~$\land$ or~$\lor$-gate of fan-in~$1$ a constant gate
		of the appropriate type. It is
		clear that the resulting Boolean circuit is
		deterministic and decomposable. Hence, from now on we
		assume that the fan-in of every~$\lor$-
		and~$\land$-gate of~$C$ is exactly~$2$.

\item[Smoothing the circuit.]
	A deterministic and decomposable circuit~$C$ is
	\emph{smooth}~\citep{darwiche2001tractability,shih2019smoothing} if for
	every~$\lor$-gate~$g$ and input gates~$g_1,g_2$ of~$g$, we have that~$\var(g_1)
	= \var(g_2)$ (we call such an~$\lor$-gate smooth).
		Recall that by the previous paragraph, we assume that
		the fan-in of every~$\lor$-gate is exactly~$2$.
		\rev{We then repeat the following
                operation until~$C$ becomes smooth.} For an~$\lor$-gate~$g$ of~$C$ having two input
		gates~$g_1,g_2$ violating the smoothness condition,
		define $S_1 \coloneqq \var(g_1) \setminus \var(g_2)$
		and~$S_2 \coloneqq \var(g_2) \setminus \var(g_1)$, and
		let~$d_{S_1}$, $d_{S_2}$ be Boolean circuits defined
		as follows. If~$S_1 = \emptyset$, then~$d_{S_1}$
		consist of the single constant $1$-gate. Otherwise,
		$d_{S_1}$ encodes the propositional formula
		$\land_{x\in S_1} (x \lor \lnot x)$ but it is
		constructed as a circuit in such a way that every~$\land$- and
		$\lor$-gate has fan-in exactly~$2$. Boolean circuit
		$d_{S_2}$ is constructed exactly as~$d_{S_1}$ but
		considering the set of variables~$S_2$ instead of
		$S_1$.  Observe that~$\var(d_{S_1})=S_1$,
		$\var(d_{S_2})=S_2$, that~$d_{S_1}$ and~$d_{S_2}$ always
		evaluate to~$1$, \rev{and that all~$\lor$-gates appearing in~$d_{S_1}$ and in~$d_{S_2}$ are deterministic}.  Then, we transform~$g$ into a smooth
		$\lor$-gate by replacing gate~$g_1$ by a
		decomposable~$\land$-gate~$(g_1 \land d_{S_2})$, and
		gate~$g_2$ by a decomposable~$\land$-gate~$(g_2 \land
		d_{S_1})$. Clearly, this does not change the Boolean classifier
		computed, \rev{and~$g$ is again deterministic because
$g_1$ and $(g_1 \land d_{S_2})$ (resp., $g_2$ and $(g_2 \land d_{S_1})$) capture the same Boolean
classifier.}
  Moreover, since~$\var(g_1 \land d_{S_2})
		= \var(g_2 \land d_{S_1}) = \var(g_1) \cup \var(g_2)$,
		we have that~$g$ is now smooth. \rev{Therefore, the circuit that we obtain
                after this operation is equivalent to the one we started from, is again
		deterministic and
		decomposable, and contains one less non-smooth~$\lor$-gate. Hence, by repeating this operation
		for each non-smooth~$\lor$-gate, which we can do in polynomial time, we
obtain an equivalent smooth deterministic and decomposable circuit where
		each~$\lor$- and~$\land$-gate has fan-in exactly~$2$.}
Thus, from now on we also assume
		that~$C$ is smooth.
\end{description}

For a gate~$g$ of~$C$, let~$R_g$ be the Boolean circuit over~$\var(g)$ that is
defined by considering the subgraph of~$C$ induced by the set of gates~$g'$
in~$C$ for which there exists a path from~$g'$ to~$g$ in~$C$.\footnote{The only
difference between~$R_g$ and~$C_g$ (defined in Section~\ref{sec:preliminaries})
is that we formally regard~$R_g$ as a Boolean classifier over~$\var(g)$, while
we formally regarded~$C_g$ as a Boolean classifier over~$X$.} Notice that~$R_g$
is a deterministic and decomposable Boolean circuit with output gate~$g$.
Moreover, for a gate~$g$ and natural number~$\ell \leq |\var(g)|$,
define~$\alpha_g^\ell \coloneqq \H_{\Pi_{p_{|\var(g)}}}(R_g,\es_{|\var(g)},\ell)$, which we recall is equal, by definition of~$\H_{\Pi_\cdot}(\cdot,\cdot,\cdot)$, to
\[ \H_{\Pi_{p_{|\var(g)}}}(R_g,\es_{|\var(g)},\ell) \ = \ \sum_{\substack{S \subseteq \var(g)\\|S|=\ell}} \, \, \mathbb{E}_{\es' \sim \Pi_{p_{|\var(g)}}}[R_g(\es') \mid \es' \in \asm(\es_{|\var(g)},S)].\]
We will show how to
compute all the values~$\alpha_g^\ell$ for every gate~$g$ of~$C$ and~$\ell \in
\{0,\ldots,|\var(g)|\}$ in polynomial time. This will conclude the proof of Lemma~\ref{lem:H} since,
for the output gate~$g_\out$ of~$C$, we have that~$\alpha_{g_\out}^k =
\H_{\Pi_p}(C,\es,k)$.  Next we explain how to compute these values by bottom-up
induction on~$C$.

\begin{description}
    \item[Variable gate.] $g$ is a variable gate with label~$y \in X$,
        so that~$\var(g)=\{y\}$. Then for~$\es' \in \eset(\{y\})$ we have~$R_g(\es')=\es'(y)$, therefore
		\begin{align}
\nonumber
		\alpha^0_g &= \ \sum_{\substack{S \subseteq \{y\}\\|S|=0}} \, \, \mathbb{E}_{\es' \sim \Pi_{p_{|\{y\}}}}[\es'(y) \mid \es' \in \asm(\es_{|\{y\}},S)]\\
		&= \ \mathbb{E}_{\es' \sim \Pi_{p_{|\{y\}}}}[\es'(y) \mid \es' \in \asm(\es_{|\{y\}},\emptyset)]\nonumber\\
		&= \ \mathbb{E}_{\es' \sim \Pi_{p_{|\{y\}}}}[\es'(y)]\nonumber\\
		&= \ 1 \cdot p(y) + 0 \cdot (1-p(y))\nonumber\\
		&= \ p(y) \label{eq:var1}
		\end{align}
		\noindent and
		\begin{align}
\nonumber
		\alpha^1_g &= \ \sum_{\substack{S \subseteq \{y\}\\|S|=1}} \, \, \mathbb{E}_{\es' \sim \Pi_{p_{|\{y\}}}}[\es'(y) \mid \es' \in \asm(\es_{|\{y\}},S)]\\
		&= \ \mathbb{E}_{\es' \sim \Pi_{p_{|\{y\}}}}[\es'(y) \mid \es' \in \asm(\es_{|\{y\}},\{y\})]\nonumber\\
		&= \ \es(y).		\label{eq:var2}
		\end{align}

    \item[Constant gate.] $g$ is a constant gate with
		label~$a\in \{0,1\}$, and~$\var(g) = \emptyset$. We recall the mathematical
		convention that there is a unique function with the
		empty domain and, hence, a unique entity
		over~$\emptyset$. But then
		\begin{align}
\nonumber
		\alpha_g^0 &= \ \sum_{\substack{S \subseteq \emptyset \\|S|=0}} \, \, \mathbb{E}_{\es' \sim \Pi_{p_{|\emptyset}}}[a \mid \es' \in \asm(\es_{|\emptyset},S)]\\
		&= \  \mathbb{E}_{\es' \sim \Pi_{p_{|\emptyset}}}[a \mid \es' \in \asm(\es_{|\emptyset},\emptyset)]\nonumber\\
		&= \ a.		\label{eq:constant}
		\end{align}
    \item[$\lnot$-gate.] $g$ is a~$\lnot$-gate with input
			gate~$g'$. Notice
			that~$\var(g)=\var(g')$. Then, since for~$\es' \in \eset(\var(g))$ we have that~$R_g(\es') = 1- R_{g'}(\es')$, we have
            \begin{align*}
            \alpha_g^\ell \ &= \ \sum_{\substack{S \subseteq \var(g)\\|S|=\ell}} \, \, \mathbb{E}_{\es' \sim \Pi_{p_{|\var(g)}}}[1- R_{g'}(\es') \mid \es' \in \asm(\es_{|\var(g)},S)].
            \end{align*}
            \noindent By linearity of expectations we deduce that
            \begin{align}
\nonumber
             \alpha_g^\ell \ &= \ \sum_{\substack{S \subseteq \var(g)\\|S|=\ell}} \, \, \mathbb{E}_{\es' \sim \Pi_{p_{|\var(g)}}}[1\mid \es' \in \asm(\es_{|\var(g)},S)]\\
            & \hspace{60pt} - \sum_{\substack{S \subseteq \var(g)\\|S|=\ell}} \, \, \mathbb{E}_{\es' \sim \Pi_{p_{|\var(g)}}}[R_{g'}(\es')\mid \es' \in \asm(\es_{|\var(g)},S)]\nonumber\\
            &= \ \big(\sum_{\substack{S \subseteq \var(g)\\|S|=\ell}} 1\big) - \alpha_{g'}^\ell \nonumber\\
            &= \ \binom{|\var(g)|}{\ell} - \alpha_{g'}^\ell 		\label{eq:negation}
            \end{align}
                        for every~$\ell \in \{0, \ldots, |\var(g)|\}$.
			By induction, the values~$\alpha_{g'}^\ell$ for~$\ell \in\{0,\ldots,|\var(g)|\}$ have
			already been computed.
            Thus, we can compute all the values~$\alpha_g^\ell$ for~$\ell \in
			\{0,\ldots,|\var(g)|\}$ in polynomial time.

\item[$\lor$-gate.] $g$ is an~$\lor$-gate. By assumption, recall that~$g$ is deterministic, smooth, and has fan-in exactly~$2$.
Let $g_1$ and~$g_2$ be the input gates of~$g$, and recall that~$\var(g_1)
= \var(g_2) = \var(g)$, because~$g$ is smooth. Given that~$g$
is deterministic, observe that for every~$\es' \in \eset(\var(g))$ we have~$R_g(\es') = R_{g_1}(\es') + R_{g_2}(\es')$.
But then for~$\ell \in \{0,\ldots,|\var(g)|\}$ we have
\begin{align}
\nonumber
\alpha_g^\ell &= \ \sum_{\substack{S \subseteq \var(g)\\|S|=\ell}} \, \, \mathbb{E}_{\es' \sim \Pi_{p_{|\var(g)}}}[R_{g_1}(\es') + R_{g_2}(\es') \mid \es' \in \asm(\es_{|\var(g)},S)]\\
&= \ \sum_{\substack{S \subseteq \var(g)\\|S|=\ell}} \, \, \mathbb{E}_{\es' \sim \Pi_{p_{|\var(g)}}}[R_{g_1}(\es') \mid \es' \in \asm(\es_{|\var(g)},S)]\nonumber\\
& \hspace{60pt} + \sum_{\substack{S \subseteq \var(g)\\|S|=\ell}} \, \, \mathbb{E}_{\es' \sim \Pi_{p_{|\var(g)}}}[R_{g_2}(\es') \mid \es' \in \asm(\es_{|\var(g)},S)]\nonumber\\
&= \ \alpha_{g_1}^\ell + \alpha_{g_2}^\ell, \label{eq:or}
\end{align}
where the second equality is by linearity of the expectation, and the last equality is valid because~$g$ is smooth (in particular, we have that $\var(g_1) = \var(g_2) = \var(g)$).
By induction, the
values~$\alpha_{g_1}^\ell$ and~$\alpha_{g_2}^\ell$, for
each~$\ell \in\{0,\ldots,|\var(g)|\}$, have already been
computed. Therefore, we can compute all the values~$\alpha_g^\ell$
for~$\ell \in \{0,\ldots,|\var(g)|\}$ in polynomial time.

\item[$\land$-gate.] $g$ is an~$\land$-gate. By assumption, recall that~$g$ is decomposable and has fan-in exactly~$2$.
Let~$g_1$ and~$g_2$ be the input gates of~$g$. For~$\es' \in \eset(\var(g))$ we have that~$R_g(\es') = R_{g_1}(\es'_{|\var(g_1)}) \cdot R_{g_2}(\es'_{|\var(g_2)})$.
Moreover, since~$\var(g) = \var(g_1) \cup \var(g_2)$ and~$\var(g_1)\cap \var(g_2) = \emptyset$ (because~$g$ is decomposable), observe that every~$S\subseteq \var(g)$ can be uniquely decomposed into~$S_1 \subseteq \var(g_1)$, $S_2 \subseteq \var(g_2)$ such that~$S = S_1\cup S_2$ and $S_1 \cap S_2 = \emptyset$.
Therefore, for~$\ell \in \{0,\ldots,|\var(g)|\}$ we have
\begin{align*}
\alpha_g^\ell = \ \sum_{\substack{S_1 \subseteq \var(g_1)\\|S_1| \leq \ell}} \, \sum_{\substack{S_2 \subseteq \var(g_2)\\|S_2| = \ell - |S_1|}}\, \mathbb{E}_{\es' \sim \Pi_{p_{|\var(g)}}}[R_{g_1}(\es'_{|\var(g_1)}) \cdot R_{g_2}(\es'_{|\var(g_2)}) &\mid\\
& \hspace{-1cm}\es' \in \asm(\es_{|\var(g)},S_1\cup S_2)].
\end{align*}
But, by definition of the product distribution~$\Pi_{p_{|\var(g)}}$ \rev{and because~$g$ is decomposable}, we have that~$R_{g_1}(\es'_{|\var(g_1)})$ and~$R_{g_2}(\es'_{|\var(g_2)})$ are independent random variables, hence we deduce
\begin{align*}
\alpha_g^\ell &= \ \sum_{\substack{S_1 \subseteq \var(g_1)\\|S_1| \leq \ell}} \, \sum_{\substack{S_2 \subseteq \var(g_2)\\|S_2| = \ell - |S_1|}}\, \bigg(\mathbb{E}_{\es' \sim \Pi_{p_{|\var(g)}}}[R_{g_1}(\es'_{|\var(g_1)})\mid \es' \in \asm(\es_{|\var(g)},S_1\cup S_2)]\\
& \hspace{120pt} \cdot \mathbb{E}_{\es' \sim \Pi_{p_{|\var(g)}}}[R_{g_2}(\es'_{|\var(g_2)})\mid \es' \in \asm(\es_{|\var(g)},S_1\cup S_2)]\bigg).\\
&= \ \sum_{\substack{S_1 \subseteq \var(g_1)\\|S_1| \leq \ell}} \, \sum_{\substack{S_2 \subseteq \var(g_2)\\|S_2| = \ell - |S_1|}}\, \bigg(\mathbb{E}_{\es'' \sim \Pi_{p_{|\var(g_1)}}}[R_{g_1}(\es'')\mid \es'' \in \asm(\es_{|\var(g_1)},S_1)]\\
& \hspace{120pt} \cdot \mathbb{E}_{\es'' \sim \Pi_{p_{|\var(g_2)}}}[R_{g_2}(\es'')\mid \es'' \in \asm(\es_{|\var(g_2)},S_2)]\bigg),
\end{align*}
where the last equality is simply by definition of the product distributions, and because~$R_{g_1}(\es'_{|\var(g_1)})$ is independent of the value~$\es'_{|\var(g_2)}$, and similarly for~$R_{g_2}(\es'_{|\var(g_2)})$. But then, \rev{using the convention that} $\alpha_{g_i}^{\ell_i} = 0$ when $\ell_i > |\var(g_i)|$, for $i = 1,2$, we obtain that
\begin{align}
\nonumber
\alpha_g^\ell &= \ \sum_{\substack{S_1 \subseteq \var(g_1)\\|S_1| \leq \ell}} \, \mathbb{E}_{\es'' \sim \Pi_{p_{|\var(g_1)}}}[R_{g_1}(\es'')\mid \es'' \in \asm(\es_{|\var(g_1)},S_1)] \, \cdot\\
& \hspace{3cm}\sum_{\substack{S_2 \subseteq \var(g_2)\\|S_2| = \ell - |S_1|}} \mathbb{E}_{\es'' \sim \Pi_{p_{|\var(g_2)}}}[R_{g_2}(\es'')\mid \es'' \in \asm(\es_{|\var(g_2)},S_2)]\nonumber\\
&= \ \sum_{\substack{S_1 \subseteq \var(g_1)\\|S_1| \leq \ell}} \, \mathbb{E}_{\es'' \sim \Pi_{p_{|\var(g_1)}}}[R_{g_1}(\es'')\mid \es'' \in \asm(\es_{|\var(g_1)},S_1)] \cdot  \alpha_{g_2}^{\ell -|S_1|}\nonumber\\
&= \sum_{\ell_1 = 0}^{\ell} \alpha_{g_2}^{\ell -\ell_1} \cdot \sum_{\substack{S_1 \subseteq \var(g_1)\\|S_1| = \ell_1}} \, \mathbb{E}_{\es'' \sim \Pi_{p_{|\var(g_1)}}}[R_{g_1}(\es'')\mid \es'' \in \asm(\es_{|\var(g_1)},S_1)]\nonumber\\
&= \sum_{\ell_1 = 0}^{\ell} \alpha_{g_2}^{\ell -\ell_1} \cdot \alpha_{g_1}^{\ell_1} \nonumber\\
&= \sum_{\substack{\ell_1 \in \{0,\ldots,\min(\ell,|\var(g_1)|)\}\\\ell_2 \in \{0,\ldots,\min(\ell,|\var(g_2)|)\}\\ \ell_1 + \ell_2 = \ell}} \alpha_{g_1}^{\ell_1} \cdot \alpha_{g_2}^{\ell_2}. \label{eq:and}
\end{align}

By induction, the values~$\alpha_{g_1}^{\ell_1}$
and~$\alpha_{g_2}^{\ell_2}$, for
each~$\ell_1 \in\{0,\ldots,|\var(g_1)|\}$
and~$\ell_2 \in \{0,\ldots,|\var(g_2)|\}$, have already been
computed. Therefore, we can compute all the values~$\alpha_g^\ell$
for~$\ell \in \{0,\ldots,|\var(g)|\}$ in polynomial time.
\end{description}
This concludes
the proof of Lemma~\ref{lem:H} and, hence, the proof that~$\shap$-score can be computed in polynomial time for our circuits.

\begin{algorithm}
\caption{$\shap$-scores for deterministic and decomposable Boolean circuits (intermediate)}
\label{algo:main}
\SetKwBlock{Induction}{Compute values~$\gamma_g^{\ell}$ and~$\delta_g^{\ell}$ for every gate $g$ in $D$ and~$\ell\in \{0,\ldots,|\var(g) \setminus \{x\}|\}$ by bottom-up induction on~$D$ as follows:\label{line:begin1}}{end\label{line:end1}}
\SetKwInOut{Input}{Input}\SetKwInOut{Output}{Output}
\Input{Deterministic and decomposable Boolean circuit~$C$ over features~$X$ with output gate~$g_\out$, rational probability values~$\pr(x)$ for all~$x\in X$, entity~$\es\in \eset(X)$, and feature~$x\in X$.}
\Output{The value $\shap(C,\es,x)$ under the probability distribution~$\prd$.}
\BlankLine
\hrulealg
\BlankLine
\BlankLine
Transform $C$ into an equivalent smooth circuit $D$ where each $\lor$-gate and $\land$-gate has fan-in exactly 2\;
\BlankLine
Compute the set~$\var(g)$ for every gate~$g$ in $D$\;
\BlankLine
\Induction{
    \uIf{$g$ is a constant gate with label~$a\in \{0,1\}$}{
        $\gamma_g^0,\, \delta_g^0 \leftarrow a$\;
    }
    \uElseIf{$g$ is a variable gate with~$\var(g)=\{x\}$\label{line:mark11}}{
        $\gamma_g^0 \leftarrow 1$\;
        $\delta_g^0 \leftarrow 0$\label{line:mark21}\;
    }
    \uElseIf{$g$ is a variable gate with~$\var(g)=\{y\}$ and $y\neq x$}{
        $\gamma_g^0,\, \delta_g^0 \leftarrow \pr(y)$\;
        $\gamma_g^1,\, \delta_g^1 \leftarrow \es(y)$\;
    }
    \uElseIf{$g$ is a $\lnot$-gate with input gate $g'$}{
        \For{$\ell\in \{0,\ldots,|\var(g)\setminus \{x\}|\}$}{
            $\gamma_g^{\ell} \leftarrow \binom{|\var(g) \setminus \{x\}|}{\ell} - \gamma_{g'}^{\ell}$\;
            $\delta_g^{\ell} \leftarrow \binom{|\var(g) \setminus \{x\}|}{\ell} - \delta_{g'}^{\ell}$\;
        }
    }

    \uElseIf{$g$ is an $\lor$-gate with input gates $g_1,g_2$\label{line:lor1}}{
        \For{$\ell\in \{0,\ldots,|\var(g)\setminus \{x\}|\}$\label{line:for1}}{
            $\gamma_g^{\ell} \leftarrow \gamma_{g_1}^{\ell} + \gamma_{g_2}^{\ell}$\;
            $\delta_g^{\ell} \leftarrow \delta_{g_1}^{\ell} + \delta_{g_2}^{\ell}$\;
        }\label{line:endlor1}
    }
    \uElseIf{$g$ is an $\land$-gate with input gates $g_1,g_2$}{
        \For{$\ell\in \{0,\ldots,|\var(g)\setminus \{x\}|\}$}{
            $\gamma_g^{\ell} \leftarrow \sum_{\substack{\ell_1 \in \{0,\ldots,\min(\ell,|\var(g_1) \setminus \{x\}|)\}\\\ell_2 \in \{0,\ldots,\min(\ell,|\var(g_2) \setminus \{x\}|)\}\\ \ell_1 + \ell_2 = \ell}} \gamma_{g_1}^{\ell_1} \cdot \gamma_{g_2}^{\ell_2}$\;
            $\delta_g^{\ell} \leftarrow \sum_{\substack{\ell_1 \in \{0,\ldots,\min(\ell,|\var(g_1) \setminus \{x\}|)\}\\\ell_2 \in \{0,\ldots,\min(\ell,|\var(g_2)\setminus\{x\}|)\}\\ \ell_1 + \ell_2 = \ell}} \delta_{g_1}^{\ell_1} \cdot \delta_{g_2}^{\ell_2}$\;
        }

    }
}
\Return ${\displaystyle \sum_{k=0}^{|X|-1} \frac{k! \, (|X| - k - 1)!}{|X|!} \cdot \big[(\es(x) - \pr(x))(\gamma_{g_\out}^k - \delta_{g_\out}^k) \big]}$\label{line:return1}\;
\end{algorithm}

\subsubsection{Extracting an algorithm from the proof}
\label{subsec:time}

From the previous proof,
it is possible to
extract Algorithm~\ref{algo:main},
which is more amenable to implementation.
The main idea of this algorithm
is that values~$\gamma_g^\ell$ correspond
to values~$\alpha^g_\ell$ of the proof for the circuit~$D_{+x}$, while
values~$\delta_g^\ell$ correspond to values~$\alpha_g^\ell$ of the proof for
the circuit~$D_{-x}$.  In lines \ref{line:begin1}--\ref{line:end1}, these values
are computed by bottom-up induction over the circuit~$D$, following the relations
that we obtained in Equations~\eqref{eq:var1}--\eqref{eq:and} (but specialized
to the circuits~$D_{+x}$ and~$D_{-x}$, as can be seen from
lines~\ref{line:mark11}--\ref{line:mark21}, and by the fact that indices are
always in~$\{0,\ldots,|\var(g) \setminus \{x\}|\}$). Hence,
it only remains to show that
the returned value of the algorithm is correct. To see that,
observe that we can rewrite Equations~\eqref{eq:L-one}
and~\eqref{eq:L-zero}~into
\[L = \es(x) \cdot \H_{\Pi_{p_{|X\setminus \{x\}}}}(D_{+x},\es_{|X\setminus \{x\}},k) + (1-\es(x)) \cdot \H_{\Pi_{p_{|X\setminus \{x\}}}}(D_{-x},\es_{|X\setminus \{x\}},k), \]
which works no matter the value of~$\es(x) \in \{0,1\}$.  We can then directly combine this
expression for~$L$ with Equation~\eqref{eq:diffk-H} and the expression of~$R$ from Equation~\eqref{eq:R} to
obtain
\begin{align*}
\diff_k(D,\es,x) \ &= \ L - R\\
& = \ (\es(x) - \pr(x)) \big[\H_{\Pi_{p_{|X\setminus \{x\}}}}(D_{+x},\es_{|X\setminus \{x\}},k) - \H_{\Pi_{p_{|X\setminus \{x\}}}}(D_{-x},\es_{|X\setminus \{x\}},k)\big].
\end{align*}
But the rightmost factor is precisely~$(\gamma_{g_\out}^k - \delta_{g_\out}^k)$ in
Algorithm~\ref{algo:main}, so that the returned value is indeed correct by Equation~\eqref{eq:shap-to-diff}.

\begin{example}\label{exa-ce-algo-1} \em
\rev{We describe in this example a complete execution of Algorithm~\ref{algo:main} over the deterministic and decomposable Boolean circuit $C$ in depicted Figure \ref{fig:ddbc-exa} (see Example \ref{ex:class}). More precisely, we show in Figure \ref{fig-ce-alg} how $\shap(C,\es,\feat{nf})$ is computed under the uniform distribution (that is, $\pr(\feat{fg}) = \pr(\feat{dtr}) = \pr(\feat{nf}) = \pr(\feat{na}) = \nicefrac{1}{2}$), and assuming that $\es(\feat{fg}) = 1$, $\es(\feat{dtr}) = 0$, $\es(\feat{nf}) = 1$, and $\es(\feat{na}) = 1$. Notice that $C(\es) = 1$.}

\rev{In the first step of the algorithm, the gate in gray in Figure \ref{fig-ce-alg} is added to $C$ so that the fan-in of every $\lor$- and $\land$-gate is exactly 2, and the gates in blue in Figure \ref{fig-ce-alg} are added to obtain a deterministic, decomposable and smooth circuit $D$ (in particular, for every $\lor$-gate~$g$ of $D$ with input gates~$g_1$ and~$g_2$, it holds that $\var(g_1) = \var(g_2)$). Then in the main loop of the algorithm, the values of $\gamma_g^i$ and $\delta_g^i$ are computed in a bottom-up fashion for every gate~$g$ and $i \in \{0, \ldots, |\var(g) \setminus \{\feat{nf}\}|\}$. Finally, assuming that $g_\out$ is the top $\land$-gate in $D$, the value $\shap(C,\es,\feat{nf})$ is computed as follows:
\begin{eqnarray*}
\shap(C,\es,\feat{nf}) &=& \sum_{k=0}^{3} \frac{k! \, (3 - k)!}{4!} \cdot \big[(\es(\feat{nf}) - \pr(\feat{nf}))(\gamma_{g_\out}^k - \delta_{g_\out}^k) \big]\\
&=& \frac{1}{2} \cdot \bigg(\frac{1}{4} \cdot \frac{1}{8} + \frac{1}{12} \cdot \frac{3}{4} + \frac{1}{12} \cdot \frac{3}{2} + \frac{1}{4} \cdot 1\bigg)\\
&=& \frac{15}{64}.
\end{eqnarray*}
As a final comment, we notice that by following the same procedure it can be shown that
$\shap(C,\es,\feat{fg}) = \nicefrac{23}{64}$,
$\shap(C,\es,\feat{dtr}) = -\nicefrac{9}{64}$, and
$\shap(C,\es,\feat{na}) = \nicefrac{15}{64}$. \qed}
\end{example}

\input{fig-smooth}

\subsection{An optimized version of the algorithm}
\label{subsec:final-algo}

In this section, we present
an optimized version of the algorithm to compute the~$\shap$-score.
The observation behind this algorithm
is that it is
possible to bypass the smoothing step and directly work on the input
circuit~$C$.\footnote{We first presented a version that used smoothing
because it made the tractability proof easier to understand.}
Our
procedure can be found in Algorithm~\ref{algo:main-final}.

Observe that Algorithm~\ref{algo:main-final} is identical to
Algorithm~\ref{algo:main}, except that: (1) we do not smooth the
circuit on line~\ref{line:preprocess}; and (2) the expressions
for~$\lor$-gates on lines~\ref{line:or1} and~\ref{line:or2} have
changed. Before showing the correctness of
Algorithm \ref{algo:main-final}, notice that the smoothing
step \rev{on line}~\ref{line:preprocess} of Algorithm~\ref{algo:main} takes
time~$O(|C|\cdot|X|)$, as described in Section~\ref{subsubsec:H}, and
the resulting circuit, called~$D$, has size~$O(|C|\cdot|X|)$. As
Algorithm~\ref{algo:main-final} does not execute such a smoothing
step, it works directly with a circuit of size $O(|C|)$, obtained
after preprocessing input circuit $C$ to ensure that all~$\lor$-
and~$\land$-gate have fan-in~$2$, and it has a lower complexity.

The only thing that changed between Algorithm~\ref{algo:main} and~\ref{algo:main-final} is how we treat an~$\lor$-gate, which can now be non-smoothed.
Therefore, to show that Algorithm \ref{algo:main-final} is correct, we need to revisit
Equation~\eqref{eq:or}. The relation
\begin{align}
\nonumber
\alpha_g^\ell &= \ \sum_{\substack{S \subseteq \var(g)\\|S|=\ell}} \, \, \mathbb{E}_{\es' \sim \Pi_{p_{|\var(g)}}}[R_{g_1}(\es') + R_{g_2}(\es') \mid \es' \in \asm(\es_{|\var(g)},S)]\\
&= \ \sum_{\substack{S \subseteq \var(g)\\|S|=\ell}} \, \, \mathbb{E}_{\es' \sim \Pi_{p_{|\var(g)}}}[R_{g_1}(\es') \mid \es' \in \asm(\es_{|\var(g)},S)]\label{eq:or2}\\
& \hspace{60pt} + \sum_{\substack{S \subseteq \var(g)\\|S|=\ell}} \, \, \mathbb{E}_{\es' \sim \Pi_{p_{|\var(g)}}}[R_{g_2}(\es') \mid \es' \in \asm(\es_{|\var(g)},S)]\nonumber
\end{align}
is still true for an arbitrary deterministic~$\lor$-gate~$g$ with
children~$g_1,g_2$, however this is not equal to~$\alpha_{g_1}^\ell +
\alpha_{g_2}^\ell$ anymore when~$g$ is not smooth.  This is because, in this
case, one of~$\var(g_1)$ or~$\var(g_2)$ (or both) is not equal to~$\var(g)$.
Assuming without loss of generality that we have $\var(g_1) \neq \var(g)$, by
induction hypothesis we have $\alpha_{g_1}^\ell =
\H_{\Pi_{p_{|\var(g_1)}}}(R_{g_1},\es_{|\var(g_1)},\ell)$, which is not the
expression \eqref{eq:or2} that we obtain above. To fix this, we will do as if the gate
had been smoothed, by considering \rev{the} virtual circuits~$d_{S_1}$ and $d_{S_2}$ that
we use in the smoothing process, but without materializing them.  We recall that $S_1
\coloneqq \var(g_1) \setminus \var(g_2)$ and~$S_2 \coloneqq \var(g_2) \setminus
\var(g_1)$, and that~$d_{S_1}$ and $d_{S_2}$ are Boolean circuits that always
evaluate to~$1$ over sets of variables~$S_1$ and~$S_2$, respectively. Let~$g'_1$
and~$g'_2$ be the output gate of those circuits.  We will show how
to compute the values~$\alpha_{g'_1}^{\ell_1}$ and~$\alpha_{g'_2}^{\ell_2}$,
for~$\ell_1 \in \{0,\ldots,|S_1|\}$ and~$\ell_2 \in \{0,\ldots,|S_2|\}$,
directly with a closed form expression and use Equation~\eqref{eq:and}
for~$\land$-gates to fix the algorithm.  We have
\begin{multline*}
\alpha_{g'_1}^{\ell_1}\ = \ \H_{\Pi_{p_{|\var(g'_1)}}}(R_{g'_1},\es_{|\var(g'_1)},\ell)
\ = \ \sum_{\substack{S \subseteq S_1\\|S|=\ell_1}} \mathbb{E}_{\es' \sim \Pi_{p_{|S_1}}}[R_{g'_1}(\es') \mid \es' \in \asm(\es_{|S_1},S)] \ = \\
\sum_{\substack{S \subseteq S_1\\|S|=\ell_1}} \mathbb{E}_{\es' \sim \Pi_{p_{|S_1}}}[1 \mid \es' \in \asm(\es_{|S_1},S)]
\ = \ \sum_{\substack{S \subseteq S_1\\|S|=\ell_1}} 1
\ = \ \binom{|S_1|}{\ell_1}.
\end{multline*}
Hence, by virtually considering that the circuit has been smoothed (that is, that we replaced gate~$g_1$ with~$(g_1\land d_{S_2})$ and gate~$g_2$ with~$(g_2\land d_{S_1})$), and
by using the relation for~$\land$-gates we can correct Equation~\eqref{eq:or} as follows:

\begin{align*}
\alpha_g^\ell \ = \ &\sum_{\ell' \in \{0,\ldots,\min(\ell,|\var(g_1)|)\}} \alpha_{g_1}^{\ell'} \cdot \binom{|\var(g_2) \setminus \var(g_1)|}{\ell - \ell'} \\
& + \sum_{\ell' \in \{0,\ldots,\min(\ell,|\var(g_2)|)\}} \alpha_{g_2}^{\ell'} \cdot \binom{|\var(g_1) \setminus \var(g_2)|}{\ell - \ell'}.
\end{align*}
And this is precisely what we use in Algorithm~\ref{algo:main-final}
for~$\lor$-gates, so this concludes the proof.

We note that, if we ignore the complexity of arithmetic operations (that is, if
we consider that arithmetic operations over rationals take constant time and
that rationals can be stored in constant space),
Algorithm~\ref{algo:main-final} runs in time~$O(|C| \cdot |X|^2)$.

\begin{algorithm}
\caption{$\shap$-scores for deterministic and decomposable Boolean circuits}
\label{algo:main-final}
\SetKwBlock{Induction}{Compute values~$\gamma_g^{\ell}$ and~$\delta_g^{\ell}$ for every gate $g$ in $C$ and~$\ell\in \{0,\ldots,|\var(g) \setminus \{x\}|\}$ by bottom-up induction on~$C$ as follows:\label{line:begin2}}{end\label{line:end2}}
\SetKwInOut{Input}{Input}\SetKwInOut{Output}{Output}
\Input{Deterministic and decomposable Boolean circuit~$C$ over features~$X$ with output gate~$g_\out$, rational probability values~$\pr(x)$ for all~$x\in X$, entity~$\es\in \eset(X)$, and feature~$x\in X$.}
\Output{The value $\shap(C,\es,x)$ under the probability distribution~$\prd$.}
\BlankLine
\hrulealg
\BlankLine
\BlankLine
Preprocess $C$ so that each $\lor$-gate and $\land$-gate has fan-in exactly 2\label{line:preprocess}\;
\BlankLine
Compute the set~$\var(g)$ for every gate~$g$ in $C$\;
\BlankLine
\Induction{
    \uIf{$g$ is a constant gate with label~$a\in \{0,1\}$}{
        $\gamma_g^0,\, \delta_g^0 \leftarrow a$\;
    }
    \uElseIf{$g$ is a variable gate with~$\var(g)=\{x\}$\label{line:mark12}}{
        $\gamma_g^0 \leftarrow 1$\;
        $\delta_g^0 \leftarrow 0$\label{line:mark22}\;
    }
    \uElseIf{$g$ is a variable gate with~$\var(g)=\{y\}$ and $y\neq x$}{
        $\gamma_g^0,\, \delta_g^0 \leftarrow \pr(y)$\;
        $\gamma_g^1,\, \delta_g^1 \leftarrow \es(y)$\;
    }
    \uElseIf{$g$ is a $\lnot$-gate with input gate $g'$}{
        \For{$\ell\in \{0,\ldots,|\var(g)\setminus \{x\}|\}$}{
            $\gamma_g^{\ell} \leftarrow \binom{|\var(g) \setminus \{x\}|}{\ell} - \gamma_{g'}^{\ell}$\;
            $\delta_g^{\ell} \leftarrow \binom{|\var(g) \setminus \{x\}|}{\ell} - \delta_{g'}^{\ell}$\;
        }
    }

    \uElseIf{$g$ is an $\lor$-gate with input gates $g_1,g_2$\label{line:lor2}}{
        \For{$\ell\in \{0,\ldots,|\var(g)\setminus \{x\}|\}$\label{line:for2}}{
            $\gamma_g^{\ell} \leftarrow \sum_{\ell' \in \{0,\ldots,\min(\ell,|\var(g_1)\setminus \{x\}|)\}} \gamma_{g_1}^{\ell'} \cdot \binom{|\var(g_2) \setminus (\var(g_1) \cup \{x\})|}{\ell - \ell'}+ \sum_{\ell' \in \{0,\ldots,\min(\ell,|\var(g_2)\setminus \{x\}|)\}} \gamma_{g_2}^{\ell'} \cdot \binom{|\var(g_1) \setminus (\var(g_2) \cup \{x\})|}{\ell - \ell'} $  \label{line:or1}\;
                        $\delta_g^{\ell} \leftarrow \sum_{\ell' \in \{0,\ldots,\min(\ell,|\var(g_1)\setminus \{x\}|)\}} \delta_{g_1}^{\ell'} \cdot \binom{|\var(g_2) \setminus (\var(g_1) \cup \{x\})|}{\ell - \ell'}+ \sum_{\ell' \in \{0,\ldots,\min(\ell,|\var(g_2)\setminus \{x\}|)\}} \delta_{g_2}^{\ell'} \cdot \binom{|\var(g_1) \setminus (\var(g_2) \cup \{x\})|}{\ell - \ell'} $\label{line:or2}\;
        }\label{line:endlor2}
    }
    \uElseIf{$g$ is an $\land$-gate with input gates $g_1,g_2$}{
         \For{$\ell\in \{0,\ldots,|\var(g)\setminus \{x\}|\}$}{
            $\gamma_g^{\ell} \leftarrow \sum_{\substack{\ell_1 \in \{0,\ldots,\min(\ell,|\var(g_1) \setminus \{x\}|)\}\\\ell_2 \in \{0,\ldots,\min(\ell,|\var(g_2) \setminus \{x\}|)\}\\ \ell_1 + \ell_2 = \ell}} \gamma_{g_1}^{\ell_1} \cdot \gamma_{g_2}^{\ell_2}$\;
            $\delta_g^{\ell} \leftarrow \sum_{\substack{\ell_1 \in \{0,\ldots,\min(\ell,|\var(g_1) \setminus \{x\}|)\}\\\ell_2 \in \{0,\ldots,\min(\ell,|\var(g_2)\setminus\{x\}|)\}\\ \ell_1 + \ell_2 = \ell}} \delta_{g_1}^{\ell_1} \cdot \delta_{g_2}^{\ell_2}$\;
        }
    }
}
\Return ${\displaystyle \sum_{k=0}^{|X|-1} \frac{k! \, (|X| - k - 1)!}{|X|!} \cdot \big[(\es(x) - \pr(x))(\gamma_{g_\out}^k - \delta_{g_\out}^k) \big]}$\label{line:return2}\;
\end{algorithm}

%% file: fig-smooth.tex
\begin{figure}
\begin{center}
\begin{tikzpicture}
  \node[circ, minimum size=7mm, inner sep=-2] (n1) {\feat{dtr}};
  \node[sqw, below=0mm of n1, text width=10mm] (vn1) {{\footnotesize $\gamma^0 = \frac{1}{2}$ $\gamma^1 = 0$ $\delta^0 = \frac{1}{2}$ $\delta^1 = 0$}};
  \node[circ, above=15mm of n1, minimum size=7mm, inner sep=-2, fill=blue!30] (n1a) {$\land$}
  edge[arrin] (n1);
    \node[sqw, above left= -2mm and 0mm of n1a, text width=22mm] (vn1a) {{\footnotesize $\gamma^0 = \frac{1}{2}$ \ $\delta^0 = \frac{1}{2}$ $\gamma^1 = \frac{1}{2}$ \ $\delta^1 = \frac{1}{2}$ $\gamma^2 = 0$ \ $\delta^2 = 0$}};
  \node[circ, left=30mm of n1, minimum size=7mm, inner sep=-2, fill=blue!30] (jn1o) {$\land$}
  edge[arrout] (n1a);  
  \node[sqw, left=0mm of jn1o, text width=22mm] (vn1o) {{\footnotesize $\gamma^0 = 1$ \ $\delta^0 = 1$ $\gamma^1 = 1$ \ $\delta^1 = 1$}};

  \node[circ, below right=18mm and 16mm of jn1o, minimum size=7mm, inner sep=-2, fill=blue!30] (n1o) {$\lor$}
  edge[arrout] (jn1o);
  \node[sqw, left=0mm of n1o, text width=22mm] (vn1o) {{\footnotesize $\gamma^0 = 1$ \ $\delta^0 = 1$ $\gamma^1 = 1$ \ $\delta^1 = 1$}};
  \node[circw, below=12mm of n1o, minimum size=7mm, inner sep=-2] (n1oa) {};
  \node[circ, left=1mm of n1oa, minimum size=7mm, inner sep=-2, fill=blue!30] (n1ol) {\feat{na}}
  edge[arrout] (n1o);
  \node[sqw, below=0mm of n1ol, text width=10mm] (vn1ol) {{\footnotesize $\gamma^0 = \frac{1}{2}$ $\gamma^1 = 1$ $\delta^0 = \frac{1}{2}$ $\delta^1 = 1$}};
  \node[circ, right=1mm of n1oa, minimum size=7mm, inner sep=-2, fill=blue!30] (n1or) {$\neg$}
  edge[arrout] (n1o);
  \node[sqw, right=0mm of n1or, text width=10mm] (vn1or) {{\footnotesize $\gamma^0 = \frac{1}{2}$ $\gamma^1 = 0$ $\delta^0 = \frac{1}{2}$ $\delta^1 = 0$}};
  \node[circ, below=10mm of n1or, minimum size=7mm, inner sep=-2, fill=blue!30] (n1oru) {\feat{na}}
  edge[arrout] (n1or);
  \node[sqw, below=0mm of n1oru, text width=10mm] (vn1oru) {{\footnotesize $\gamma^0 = \frac{1}{2}$ $\gamma^1 = 1$ $\delta^0 = \frac{1}{2}$ $\delta^1 = 1$}};

  \node[circ, below left=18mm and 16mm of jn1o, minimum size=7mm, inner sep=-2, fill=blue!30] (kn1o) {$\lor$}
  edge[arrout] (jn1o);
  \node[sqw, left=0mm of kn1o, text width=10mm] (kvn1o) {{\footnotesize $\gamma^0 = 1$ $\delta^0 = 1$}};
  \node[circw, below=12mm of kn1o, minimum size=7mm, inner sep=-2] (kn1oa) {};
  \node[circ, left=1mm of kn1oa, minimum size=7mm, inner sep=-2, fill=blue!30] (kn1ol) {\feat{nf}}
  edge[arrout] (kn1o);
  \node[sqw, below=0mm of kn1ol, text width=10mm] (kvn1ol) {{\footnotesize $\gamma^0 = 1$ $\delta^0 = 0$}};
  \node[circ, right=1mm of kn1oa, minimum size=7mm, inner sep=-2, fill=blue!30] (kn1or) {$\neg$}
  edge[arrout] (kn1o);
  \node[sqw, right=0mm of kn1or, text width=10mm] (kvn1or) {{\footnotesize $\gamma^0 = 0$ $\delta^0 = 1$}};
  \node[circ, below=10mm of kn1or, minimum size=7mm, inner sep=-2, fill=blue!30] (kn1oru) {\feat{nf}}
  edge[arrout] (kn1or);
  \node[sqw, below=0mm of kn1oru, text width=10mm] (kvn1oru) {{\footnotesize $\gamma^0 = 1$ $\delta^0 = 0$}};

  \node[circ, left=45mm of n1a, minimum size=7mm, inner sep=-2] (n0) {\feat{fg}};
  \node[sqw, left=0mm of n0, text width=10mm] (vn0) {{\footnotesize $\gamma^0 = \frac{1}{2}$ $\gamma^1 = 1$ $\delta^0 = \frac{1}{2}$ $\delta^1 = 1$}};
  \node[circw, right=12mm of n1, minimum size=7mm] (n2) {};
  \node[circ, right=12mm of n2, minimum size=7mm, fill=gray!30] (n3) {$\land$};
  \node[sqw, right=0mm of n3, text width=10mm] (vn3) {{\footnotesize $\gamma^0 = \frac{1}{2}$ $\gamma^1 = 1$ $\delta^0 = 0$ $\delta^1 = 0$}};
  \node[circ, below=15mm of n3, minimum size=7mm] (n3r) {\feat{na}}
  edge[arrout] (n3);
  \node[sqw, below=0mm of n3r, text width=10mm] (vn3r) {{\footnotesize $\gamma^0 = \frac{1}{2}$ $\gamma^1 = 1$ $\delta^0 = \frac{1}{2}$ $\delta^1 = 1$}};
  \node[circ, left=12mm of n3r, minimum size=7mm] (n3l) {\feat{nf}}
  edge[arrout] (n3);
  \node[sqw, below=0mm of n3l, text width=10mm] (v3l) {{\footnotesize $\gamma^0 = 1$ $\delta^0 = 0$}};
  \node[circ, above=15mm of n3, minimum size=7mm] (n4) {$\land$}
  edge[arrin] (n3);
  \node[circ, above=25.6mm of n3l, minimum size=7mm] (nnegaux) {$\neg$}
    edge[arrin] (n1)
    edge[arrout] (n4);
  \node[sqw, below=0mm of nnegaux, text width=10mm] (ven4) {{\footnotesize $\gamma^0 = \frac{1}{2}$ $\gamma^1 = 1$ $\delta^0 = \frac{1}{2}$ $\delta^1 = 1$}};
  \node[sqw, right=0mm of n4, text width=22mm] (vn4) {{\footnotesize $\gamma^0 = \frac{1}{4}$ \ $\delta^0 = 0$ $\gamma^1 =1$ \ $\delta^1 =0$ $\gamma^2 = 1$ \ $\delta^2 = 0$}};

  \node[circ, above=35mm of n2, minimum size=7mm] (n5) {$\lor$}
  edge[arrin] (n4)
  edge[arrin] (n1a);
  \node[sqw, above right = -2mm and 0mm of n5, text width=22mm] (vn5) {{\footnotesize $\gamma^0 = \frac{3}{4}$ \ $\delta^0 = \frac{1}{2}$ $\gamma^1 = \frac{3}{2}$ \ $\delta^1 = \frac{1}{2}$ $\gamma^2 = 1$ \ $\delta^2 = 0$ }};
  \node[circ, above=53mm of n1, minimum size=7mm] (n6) {$\land$}
  edge[arrin] (n0) edge[arrin] (n5);
  \node[sqw, above left = -3mm and 0mm of n6, text width=22mm] (vn6) {{\footnotesize $\gamma^0 = \frac{3}{8}$ \ $\delta^0 = \frac{1}{4}$ $\gamma^1 = \frac{3}{2}$ \ $\delta^1 = \frac{3}{4}$ $\gamma^2 = 2$ \ $\delta^2 = \frac{1}{2}$ $\gamma^3 = 1$ \ $\delta^3 = 0$}};
\end{tikzpicture} 
\caption{\rev{Execution of Algorithm~\ref{algo:main} over the deterministic and decomposable Boolean circuit depicted in Figure \ref{fig:ddbc-exa}.} \label{fig-ce-alg} }
\end{center}
\end{figure}
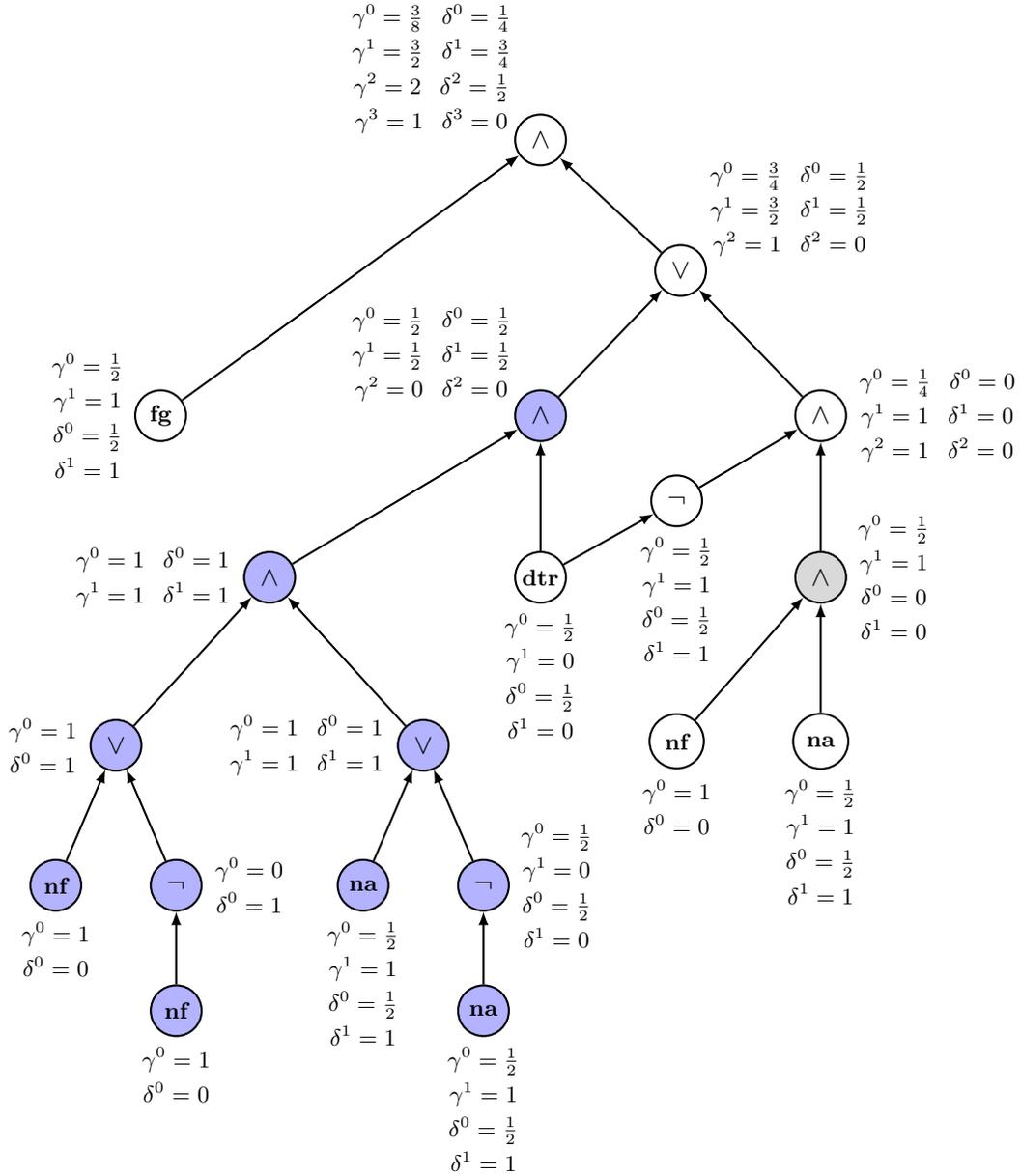

%% file: nonbinary.tex
In this section, we show how to extend the result of
\rev{the previous section}
to non-binary classifiers. First, we need to
redefine the notions of entities, product distributions and~$\shap$-score to
account for non-binary features.  
\rev{We point out that these new definitions are to be considered for this
section only, as in the rest of the paper we will again consider binary classifiers.}
Let~$X$ be a finite set of features,
and~$\dom$ be a function that associates to every feature~$x\in X$ a finite
domain~$\dom(x)$. An entity over~$(X,\dom)$ is a function~$\es$ that associates
to every feature an element~$\es(x) \in \dom(X)$.  We denote by~$\eset(X,\dom)$
the set of all entities over~$(X,\dom)$.  We then consider product
distributions on~$\eset(X,\dom)$ as follows. For every~$x\in X$,
let~$\pr_x:\dom(X) \to [0,1]$ be such that~$\sum_{v\in \dom(X)}\pr_x(v) = 1$.
Then the \emph{product distribution generated by~$(\pr_x)_{x\in X}$} is the
probability distribution~$\prd$ over~$\eset(X,\dom)$ such that, for
every~$\es\in \eset(X,\dom)$ we have 
\[\prd(\es) \ \ \coloneqq \ \ \prod_{x\in X} \pr_x(\es(x)).\]
That is, again,~$\prd$ is the product distribution that is
determined by pre-specified marginal distributions, and that makes the features
take values independently from each other.  A \emph{Boolean classifier}~$M$
over~$X$ is a function~$M : \eset(X,\dom) \to \{0,1\}$ that maps every entity
over~$X$ to~$0$ or~$1$. The~$\shap$ score over such classifiers is then defined
just like in Section~\ref{subsec:shapdef}.

\paragraph*{Non-binary Boolean circuits.}
We now define a variant of deterministic and decomposable circuits for
non-binary variables, that we will call \emph{deterministic and decomposable
non-binary Boolean circuits}. Just like a Boolean circuit captures a set of
binary entities (those that satisfy the circuit), a \emph{non-binary Boolean
circuit} will capture a set of entities over~$(X,\dom)$.  We define a
non-binary Boolean circuit over~$(X,\dom)$ like a Boolean circuit
(recall the definition from Section~\ref{subsec:circuitsprelims}), except
that each variable gate is now labeled with an equality of the form~$x=v$,
where~$x\in X$ and~$v\in \dom(X)$.  Such a circuit then maps every entity
over~$X$ to~$0$ or~$1$ in the expected way, and can thus be seen as a Boolean
classifier.
Again, an~$\land$-gate is decomposable if for every pair~$g_1$,~$g_2$ of distinct input
gates of~$g$, we have that~$\var(g_1) \cap \var(g_2) =
\emptyset$; an~$\lor$-gate is deterministic if for every pair~$g_1$,~$g_2$ of distinct input
gates of~$g$ there is no entity that satisfies them both; and the circuit is called deterministic and decomposable
if all its~$\land$-gates are decomposable and all its~$\lor$-gates are deterministic.

The main result of this section is that we can generalize Theorem~\ref{thm:shapscore-d-Ds-prod} to these kind of circuits (we presented
the result for Boolean circuits over binary variables first for clarity of presentation):
\begin{theorem}
\label{thm:shapscore-d-Ds-prod-nonbinary}
Given as input \rev{a set of features~$X$, finite domains~$\dom(x)$ for every~$x\in X$,} a
deterministic and decomposable non-binary Boolean circuit~$C$ over~$(X,\dom)$, rational
probability values~$\pr_x(v)$ for every~$x\in X$ and~$v\in \dom(x)$, an entity~$\es\in \eset(X,\dom)$, and a feature~$x\in X$, the value~$\shap_{\prd}(C,\es,x)$ can be computed in
polynomial time.
\end{theorem}
\begin{proof}
\rev{First, notice that, since probability values~$\pr_x(v)$ for every~$x\in X$
and~$v\in \dom(x)$ are anyway given as part of the input, it can safely be
considered that~$\dom(x)$ is of linear size.}
We then go through the proof of Theorem~\ref{thm:shapscore-d-Ds-prod} and
only explain what changes.  For~$x\in X$ and~$v\in \dom(x)$, we denote by~$C_{x=v}$
the non-binary Boolean circuit that is obtained from~$C$ by replacing every variable gate that
is labeled with $x=v$ by a constant 1-gate, and every variable gate that is
labeled with $x=v'$ for~$v'\neq v$ by a constant 0-gate (it is clear that~$C_{x=v}$ is
again deterministic and decomposable if~$C$ satisfies these properties).
Then, the reduction from~$\shap_{\Pi_\cdot}(\cdot,\cdot,\cdot)$ to~$\H_{\Pi_\cdot}(\cdot,\cdot,\cdot)$ from
Section~\ref{subsubsec:shap-to-H} still works, for instance the term~$R$ from Equation~\eqref{eq:diffk-H} becomes
\begin{align*}
R \ =& \ \sum_{\substack{S\subseteq X\setminus \{x\}\\|S|=k}} \mathbb{E}_{\es' \sim \Pi_p}[C(\es') \mid \es' \in \asm(\es,S)]\\
=& \ \, \sum_{v\in \dom(X)} \pr_x(v) \cdot \H_{\Pi_{p_{|X\setminus \{x\}}}}(C_{x=v},\es_{|X\setminus \{x\}},k).
\end{align*}
We now look at the computation of~$\H_{\Pi_\cdot}(\cdot,\cdot,\cdot)$ and
inspect what changes in Lemma~\ref{lem:H}. We can rewrite to fan-in~$2$ and
smooth the circuit in the same way,\footnote{For smoothing, we use the
circuits~$d_{S} \coloneqq \bigwedge_{x\in S} (\bigvee_{v \in \dom(x)} x=v )$.} and
the quantities~$\alpha^\ell_g$ are also defined in the same way. The relations
that we used to compute the values~$\alpha^\ell_g$ do not change, except the
ones for variable gates: for a variable gate~$g$ labeled with~$x=v$,
Equations~\eqref{eq:var1} and~\eqref{eq:var2} become, respectively,~$\alpha^0_g
= \pr_x(v)$ and
\begin{eqnarray*}
\alpha^1_g &=&
\begin{cases}
1 & \text{if } \es(x) = v\\
0 & \text{otherwise}
\end{cases}.
\end{eqnarray*}
\end{proof}

This in particular allows us to prove that the~$\shap$-score can be computed in
polynomial time for (not necessarily binary) decision trees, or for variants of
OBDDs/FBDDs that use non-binary features, as deterministic and decomposable 
non-binary Boolean circuits generalize them.

%% file: limits.tex
We have shown in \rev{the previous sections} that the $\shap$-score can be computed in polynomial time for deterministic and decomposable circuits under product distributions.
A natural question, then, is whether both determinism and decomposability are necessary for this positive result to hold.
In this section we show that this is the case, at least under standard complexity assumptions, and even when we consider the uniform distribution.
\rev{Recall that we are now back to considering Boolean classifiers.}
Formally, we prove the following:

\begin{theorem}
\label{thm:shapscore-limits}
The following problems are~$\shp$-hard.
\begin{enumerate}
\item Given as input a decomposable (but not necessarily deterministic) Boolean circuit~$C$
over a set of features~$X$, an entity~$\es:X\to \{0,1\}$, and a feature~$x\in
X$, compute the value~$\shap_\mathcal{U}(C,\es,x)$.
\item
Given as input a deterministic (but not necessarily decomposable) Boolean circuit~$C$
over a set of features~$X$, an entity~$\es:X\to \{0,1\}$, and a feature~$x\in
X$, compute the value~$\shap_\mathcal{U}(C,\es,x)$.
\end{enumerate}
\end{theorem}

Intuitively, for the first item, this comes from the fact that an arbitrary
Boolean circuit can always be transformed in polynomial time (in fact in linear
time) into an equivalent decomposable (but not necessarily deterministic)
Boolean circuit, simply by applying De Morgan's laws to eliminate
all~$\land$-gates. Hence the problem on those circuits it at least as hard as
on unrestricted Boolean circuits; and the argument is similar for the second item.
Therefore, to show Theorem \ref{thm:shapscore-limits}, it is good enough to prove that
the problem is indeed intractable over unrestricted Boolean circuits. We now
prove these claims formally.

We start by showing that there is
a \rev{general} polynomial-time reduction from the problem of computing the number of
entities that satisfy $M$, for $M$ an arbitrary Boolean classifier, to the
problem of computing the $\shap$-score over $M$ under the uniform distribution.
This holds
 under the mild condition that $M(\es)$ can be computed in polynomial time for an input entity $\es$,
 which is satisfied for all the Boolean
circuits and binary decision diagrams classes considered in this
paper.
The proof of this result follows from well-known properties of Shapley values, and a closely related result can be found as
Theorem~5.1 in~\citep{BLSSV20}.

\begin{lemma}
\label{lem:limits}
Let~$M$ be a Boolean classifier over a set of features~$X$, and let~$\ssat(M) \coloneqq |\{\es \in \eset(X) \mid M(\es)=1\}|$.
Then for every~$\es \in \eset(X)$ we have:
\begin{align*}
\ssat(M) \ = \ 2^{|X|} \bigg( M(\es) - \sum_{x\in X} \shap_\mathcal{U}(M,\es,x) \bigg).
\end{align*}
\end{lemma}
\begin{proof}
The validity of this equation will be \rev{a} consequence of the following property of the $\shap$-score:
for every Boolean classifier~$M$ over~$X$, entity~$\es\in \eset(X)$ and feature~$x\in X$, it holds that
\begin{align}
\label{eq:efficiency}
\sum_{x\in X} \shap_\mathcal{U}(M,\es,x) \ = \ \phi_\mathcal{U}(M,\es,X) - \phi_\mathcal{U}(M,\es,\emptyset).
\end{align}
This property is often called the \emph{efficiency} property of the Shapley value.
Although this is folklore,
we prove Equation~\eqref{eq:efficiency} here for
the reader's convenience.  Recall from Equation~\eqref{def:Shapley-alt} that the~$\shap$-score can be written as
\[
 \shap_\mathcal{U}(M,\es,x) \ = \frac{1}{|X|!}\sum_{\pi \in \Pi(X)}  \big(\phi_\mathcal{U}(M,\es,S_\pi^x \cup \{x\}) - \phi_\mathcal{U}(M,\es,S_\pi^x)\big).
\]
Hence, we have that
\begin{align*}
\sum_{x\in X} \shap_\mathcal{U}(M,\es,x)\ &= \ \frac{1}{|X|!} \sum_{x\in X} \ \sum_{\pi \in \Pi(X)}  \big(\phi_\mathcal{U}(M,\es,S_\pi^x \cup \{x\}) - \phi_\mathcal{U}(M,\es,S_\pi^x)\big)\\
&= \ \frac{1}{|X|!}\sum_{\pi \in \Pi(X)} \ \sum_{x\in X}  \big(\phi_\mathcal{U}(M,\es,S_\pi^x \cup \{x\}) - \phi_\mathcal{U}(M,\es,S_\pi^x)\big)\\
&= \ \frac{1}{|X|!}\sum_{\pi \in \Pi(X)} \big(\phi_\mathcal{U}(M,\es,X) - \phi_\mathcal{U}(M,\es,\emptyset)\big)\\
&= \ \phi_\mathcal{U}(M,\es,X) - \phi_\mathcal{U}(M,\es,\emptyset),
\end{align*}
where the second to last equality is obtained by noticing that the inner sum
is a telescoping sum.
This establishes Equation~\eqref{eq:efficiency}.
Now, by definition of~$\phi_\mathcal{U}(\cdot,\cdot,\cdot)$ we have that~$\phi_\mathcal{U}(M,\es,X) = M(\es)$ and~$\phi_\mathcal{U}(M,\es,\emptyset) = \frac{1}{2^{|X|}} \sum_{\es' \in \eset(X)} M(\es')$, so we obtain
\begin{align*}
\sum_{x\in X} \shap_\mathcal{U}(M,\es,x)\ = \ M(\es) - \frac{\ssat(M)}{2^{|X|}},
\end{align*}
thus proving Lemma~\ref{lem:limits}.
\end{proof}

\rev{We stress out here that this result holds for \emph{any} Boolean
classifier and is not restricted to the classes of Boolean circuits that we
consider in this paper.}

Theorem~\ref{thm:shapscore-limits} can then easily be deduced from
Lemma~\ref{lem:limits} and the following two facts: (a) counting the
number of satisfying assignments of an arbitrary Boolean circuit
is a~$\shp$-hard problem~\citep{provan1983complexity}; and (b)
every Boolean circuit can be transformed in linear time into a Boolean
circuit that is deterministic (resp, decomposable), simply by using De
Morgan's laws to get rid of the~$\lor$-gates (resp., $\land$-gates).
\rev{For instance, the reduction to prove the first item of
Theorem~\ref{thm:shapscore-limits} is as follows: on input an arbitrary
Boolean circuit~$C$, use De Morgan's laws to remove all~$\land$-gates,
thus obtaining an equivalent circuit~$C'$ which is vacuously decomposable
(since it does not have any~$\land$-gates), and then use the oracle to
computing SHAP-scores together with Lemma~\ref{lem:limits} with an
arbitrary entity (for instance, the one that assigns zero to all features) in order to
compute~$\ssat(C)$.}

%% file: approx.tex
We now study the approximability of the~$\shap$-score.  As
we have shown in the previous section \rev{with Lemma~\ref{lem:limits}}, computing the~$\shap$-score is
\rev{generally} intractable for classes of models for which model counting is intractable. Yet, it could be the case that one can
efficiently approximate the~$\shap$-score, just like in some cases one
can efficiently approximate the number of models of a formula even though computing this quantity exactly is intractable.
For instance, model counting of DNF formulas is~$\shp$-hard~\citep{provan1983complexity} but admits a
\emph{Fully Polynomial-time Randomized Approximation Scheme}, or FPRAS~\citep{karp1989monte}.
Unfortunately, as we show next,
intractability of computing the~$\shap$-score continues to hold when one considers approximability, and this even for very
restricted kinds of Boolean classifiers and under the uniform probability distribution.

To simplify the notation in this section, we will drop the subscript~$\mathcal{U}$ for the uniform distribution,
and write~$\shap(\cdot,\cdot,\cdot)$ instead of~$\shap_\mathcal{U}(\cdot,\cdot,\cdot)$, and~$\phi(\cdot,\cdot,\cdot)$ instead of~$\phi_\mathcal{U}(\cdot,\cdot,\cdot)$.
Let $\C$ be a class of Boolean classifiers and $\varepsilon \in
(0,1)$. We say that the problem of computing the $\shap$ score for $\C$ admits an
{\em $\varepsilon$ polynomial-time randomized approximation}
($\varepsilon$-PRA), if there exists a randomized algorithm $\A$
satisfying the following conditions. For every Boolean classifier
$M\in\C$ over a set of features $X$, entity $\es$ over $X$, and feature
$x \in X$, it holds that:
\begin{eqnarray*}
\Pr\big(|\A(M,\es,x) - \shap(M,\es,x)| \leq \varepsilon \cdot |\shap(M,\es,x)|\big) & \geq & \frac{3}{4}.
\end{eqnarray*}
Moreover, there exists a polynomial $p(\cdot)$ such that $\A(M,\es,x)$ works in
time $O(p(\|M\| + \|\es\|))$, where $\|M\|$ and $\|\es\|$ are the sizes
of $M$ and $\es$ represented as input strings, respectively.
We recall that the notion of FPRAS mentioned before is defined as the concept of PRA, but imposing the stronger requirements that~$\varepsilon$ is part of the input and the algorithm is polynomial in~$\frac{1}{\varepsilon}$ as well.

\begin{sloppypar}
We start by presenting a simple proof that for every $\varepsilon \in (0,1)$, the
problem of computing the $\shap$ score for Boolean classifiers given
as DNF formulas does not admit an $\varepsilon$-PRA, unless~\mbox{$\np
= \rp$}.\footnote{Recall that it is widely believed that $\rp$ is
\rev{properly}
contained in $\np$, as mentioned in
Section~\ref{sec:preliminaries}.}

\begin{proposition}
\label{prp:non-FPRAS-DNFs}
For every $\varepsilon \in (0,1)$, the problem of computing the $\shap$
score for Boolean classifiers given as DNF formulas does not
admit an $\varepsilon$-PRA, unless~$\np = \rp$. This result holds even
if we restrict to the uniform distributions on the entities.
\end{proposition}
\begin{proof}
We first recall the following fact: if a function~$f$ admits an
$\varepsilon$-PRA, then the problem of determining, given a string $x$,
whether~$f(x)=0$ is in \bpp\ (although this is folklore, we provide a
proof in Appendix~\ref{app:pras-distinguishes-zeros}).  We use this to
prove that if we could approximate the $\shap$-score over DNF
formulas, then we could solve the validity problem over DNF formulas
in~$\bpp$.
\rev{Recall that the validity problem over DNF formulas is the decision problem
that, given as input a DNF formula~$\varphi$, accepts if all valuations satisfy~$\varphi$,
and rejects otherwise (in other words, it rejects if~$\lnot\varphi$ is
satisfiable and accepts otherwise).}
Since this problem is~$\conp$-complete and
$\bpp$ is closed under complement, this would imply
that~$\np \subseteq
\bpp$, and thus that~$\np = \rp$~\citep{K82}.  Let~$\varphi$ be a DNF formula
over a set of variables~$X$. We consider the DNF formula~$\varphi'\coloneqq \varphi
\lor x$ with~$x\notin X$, and the uniform probability distribution
over~$\eset(X\cup \{x\})$.  Let~$\es$ be an arbitrary entity over~$X \cup \{x\}$
such that~$\es(x)=1$.  We show that~$\varphi$ is valid if and only
if~$\shap(\varphi',\es,x)=0$, which, by
the previous remarks, is good enough to conclude the proof.  By definition we have
\begin{equation*}
\shap(\varphi',\es,x) \ = \ \sum_{S \subseteq X}
\frac{|S|! \, (|X| - |S|)!}{(|X|+1)!} \, \bigg(\phi(\varphi', \es,S \cup \{x\})
- \phi(\varphi', \es,S)\bigg).
\end{equation*}
Observe that for each~$S\subseteq X$, it holds that~$\phi(\varphi', \es,S
\cup \{x\})=1$
(given the definition of $\varphi'$ and the fact that~$\es(x)=1$), and
that~$0 \leq \phi(\varphi', \es,S) \leq 1$.  Now, if~$\varphi$ is valid, then
it is clear that~$\phi(\varphi', \es,S)=1$ for every~$S\subseteq X$, so
that indeed~$\shap(\varphi',\es,x)=0$.  Assume now that~$\varphi$ is not valid.
By what preceded, it is good enough to show that for some~$S\subseteq X$ we
have~$\phi(\varphi', \es,S) < 1$. But this clearly holds
for~$S=\emptyset$, because~$\varphi$ is not valid and all
entities have the same probability (so that no entity has probability zero).
\end{proof}

Hence, this result already establishes an important difference with
model counting: for the case of DNF formulas, the model counting problem
admits an FPRAS~\citep{karp1989monte} and, thus, an
$\varepsilon$-PRA for every $\varepsilon \in (0,1)$.
\end{sloppypar}

It is important to mention that the proof of
Proposition~\ref{prp:non-FPRAS-DNFs} uses, in a crucial way, the fact
that the validity problem for DNF formulas is intractable. We prove
next a
strong negative result, which establishes that not even
for \emph{positive} DNF formulas -- for which the validity problem is
trivial -- it is possible to obtain an $\varepsilon$-PRA for computing
the~$\shap$-score.
Let $\varphi = D_1 \vee D_2 \vee \cdots \vee D_k$ be a formula in DNF,
that is, each formula $D_i$ is a conjunction of positive literals (propositional variables) and negative literals (negations of propositional variables).
Then $\varphi$ is said to be in POS-DNF if each
formula $D_i$ is a conjunction of positive literals,
(hence, $\varphi$ is a monotone formula), and $\varphi$ is said to be
in 2-POS-DNF if $\varphi$ is in POS-DNF and each formula $D_i$
contains at most two (positive) literals.  Our main result of this
section is the following.

\begin{theorem}
\label{theo:non-FPRAS-2-POS-DNF}
For every $\varepsilon \in (0,1)$, the problem of computing the $\shap$
score for Boolean classifiers given as 2-POS-DNF formulas does not
admit an $\varepsilon$-PRA, unless~$\np = \rp$. This result holds even
if we restrict to the uniform distributions on the entities.
\end{theorem}
Furthermore, we can show that the same intractability result also
holds for closely related classes of Boolean classifiers. As before, let $\varphi = D_1 \vee D_2 \vee \cdots \vee D_k$ be a formula in DNF. Then $\varphi$ is in
NEG-DNF if
each formula $D_i$ is a conjunction of negative literals, and
$\varphi$ is in 2-NEG-DNF if it is in NEG-DNF and each formula $D_i$
is a conjunction of
at most two (negative) literals. We define similarly formulas in
2-POS-CNF and in 2-NEG-CNF. We then obtain the following result.

\begin{corollary}
\label{cor:cor}
Let~$\mathcal{C} \in \{\text{2-POS-DNF}, \text{2-NEG-DNF}, \text{2-POS-CNF}, \text{2-NEG-CNF}\}$.
Then for every $\varepsilon \in (0,1)$, the problem of computing the $\shap$
score for Boolean classifiers given as formulas in~$\mathcal{C}$ does not
admit an $\varepsilon$-PRA, unless~$\np = \rp$. This result holds even
if we restrict to the uniform distributions on the entities.
\end{corollary}
It should be noticed that 2-NEG-CNF formulas are special cases of
HORN-SAT formulas, so that the previous result also applies for
HORN-SAT.

We prove Theorem~\ref{theo:non-FPRAS-2-POS-DNF} in Section~\ref{subsec:proof-approx}, and then show how the proof can be adapted to prove~Corollary~\ref{cor:cor} in Section~\ref{subsec:corollaries}

\subsection{Proof of Theorem~\ref{theo:non-FPRAS-2-POS-DNF}}
\label{subsec:proof-approx}

We will use some results and techniques related to
approximating cliques in graphs. More specifically, all graphs $G =
(N,E)$ considered in this proof are assumed to be undirected and
loop-free (that is, edges of the form $(a,a)$ are not
allowed). Besides, we assume that each graph $G = (N,E)$ satisfies the
following condition:
\begin{enumerate}[label=(\Alph*)]
\item there exist at least two isolated nodes in $G$, that is, two distinct nodes $a,b \in N$ such that $(a,c) \not\in E$ and $(b,c) \not\in E$ for every node $c \in N$. \label{cond:u}
\end{enumerate}
The assumption that condition \ref{cond:u} is satisfied will allow us to
simplify some calculations. Besides, it is clear that
condition \ref{cond:u} can be checked in polynomial time.

We define the problem $\gclique$ as
follows. The input of $\gclique$ is a graph $G = (N,E)$ and a number
$m \leq |N|$, and its output is:
\begin{itemize}
\item {\bf yes} if $G$ contains a clique with $m$ nodes,

\item {\bf no} if every clique of $G$ contains at most $\lfloor \frac{m}{3} \rfloor$ nodes.
\end{itemize}
From the PCP theorem and its applications to hardness of
approximation \citep{FGLSS96,AS98,AMSS98}, it is known that $\gclique$
is $\np$-hard (in fact even if $\frac{1}{3}$ is replaced by any
$\delta \in (0,1)$).  In other words, there exists a polynomial-time reduction that takes
as input a Boolean formula~$\varphi$ and that outputs a graph~$G$ and an integer~$m$ such that (1) if~$\varphi$ is satisfiable then~$G$
contains a clique with~$m$ nodes; and (2) if~$\varphi$ is not satisfiable then every clique of~$G$ contains at most $\lfloor \frac{m}{3} \rfloor$ nodes.

The idea of our proof of Theorem~\ref{theo:non-FPRAS-2-POS-DNF} is then to show that if the problem of
computing the $\shap$ score for Boolean classifiers given as 2-POS-DNF
formulas admits an $\varepsilon$-PRA, then there exists a $\bpp$ algorithm for
$\gclique$. Hence, we would conclude that $\np \subseteq \bpp$, which in
turn implies that $\np = \rp$
\citep{K82}. The proof is technical, and it is divided into five modular parts, highlighted in bold in the remaining of this section.

\paragraph*{The~$\shap$-score of a Boolean classifier of the form~$M \lor x$.}
Given a set of
features~$X$, a Boolean classifier $M$ over $X$, and an $S \subseteq X$,
define
\[\ssat(M, S) \ \coloneqq \ |\{ \es \in \eset(X) \mid  M(\es) = 1 \text{ and } \es(y) = 1 \text{ for every } y \in S\}|.\]
We can relate the $\shap$-score of a Boolean classifier of the form~$M \lor x$ to the quantities~$\ssat(\lnot M, S)$ as follows:

\begin{lemma}\label{lem-fpras-disj}
  Let $X$ be a set of features, $n = |X|$, $x \in X$, $M$ be a Boolean
  classifier over $X \setminus \{x\}$, and $\one$ be the entity over $X$
  such that $\one(y) = 1$ for every $y \in X$. Then
\begin{eqnarray*}
    \shap(M \vee x, \one, x) & = & \sum_{k=0}^{n-1} \frac{k! (n-k-1)!}{n!} \cdot \frac{1}{2^{n-k}} \sum_{S \subseteq X \setminus \{x\} \,:\, |S| = k} \ssat(\neg M, S).
  \end{eqnarray*}
  \end{lemma}

\begin{proof}  Let $M'$ be a Boolean classifier over $X$. Given $S \subseteq X \setminus \{x\}$, we have that:
\begin{eqnarray*}
  \phi(M',\one,S \cup \{x\}) &=& \sum_{\es \in \asm(\one,S \cup \{x\})}
  \frac{1}{2^{|X \setminus (S \cup \{x\})|}} \cdot M'(\es)\\
  &=& \frac{1}{2^{|X \setminus (S \cup \{x\})|}} \sum_{\es' \in \asm(\one_{|{X \setminus \{x\}}},S)} M'_{+x}(\es')\\
  &=& \frac{2}{2^{|X \setminus S|}} \sum_{\es' \in \asm(\one_{|{X \setminus \{x\}}},S)}
  M'_{+x}(\es').
\end{eqnarray*}
Besides, we have that:
\begin{eqnarray*}
  \phi(M',\one,S) &=& \sum_{\es \in \asm(\one,S)} \frac{1}{2^{|X \setminus S|}} \cdot M'(\es)\\
  &=& \sum_{\es \in \asm(\one,S) \,:\, \es(x)=0} \frac{1}{2^{|X \setminus S|}} \cdot M'(\es) \ + \ \sum_{\es \in \asm(\one,S) \,:\, \es(x)=1} \frac{1}{2^{|X \setminus S|}} \cdot M'(\es) \\
  &=& \frac{1}{2^{|X \setminus S|}} \cdot \bigg(\sum_{\es' \in \asm(\one_{|X \setminus \{x\}},S)} M'_{-x}(\es') \ + \ \sum_{\es' \in \asm(\one_{|X \setminus \{x\}},S)} M'_{+x}(\es')\bigg).
\end{eqnarray*}
Therefore, we conclude that:
\begin{eqnarray*}
  \phi(M',\one,S \cup \{x\}) - \phi(M',\one,S) \ = \ \frac{1}{2^{|X \setminus S|}} \cdot \bigg(\sum_{\es \in \asm(\one_{|X \setminus \{x\}},S)} M'_{+x}(\es) - \sum_{\es \in \asm(\one_{|X \setminus \{x\}},S)} M'_{-x}(\es)\bigg).
\end{eqnarray*}
Considering this equation with $M' \coloneqq M \vee x$, we deduce that:
\begin{align*}
  \phi(M \vee x&,\one,S \cup \{x\}) -  \phi(M \vee x,\one,S) \ =\\
  &\frac{1}{2^{|X \setminus S|}} \cdot \bigg(\sum_{\es \in \asm(\one_{|X \setminus \{x\}},S)} (M \vee x)_{+x}(\es) - \sum_{\es \in \asm(\one_{|X \setminus \{x\}},S)} (M \vee x)_{-x}(\es)\bigg) \ =\\
  &\frac{1}{2^{|X \setminus S|}} \cdot \bigg(\sum_{\es \in \asm(\one_{|X \setminus \{x\}},S)} 1 - \sum_{\es \in \asm(\one_{|X \setminus \{x\}},S)} M(\es)\bigg) \ =\\
  &\frac{1}{2^{|X \setminus S|}} \cdot \bigg(\sum_{\es \in \asm(\one_{|X \setminus \{x\}},S)} (1 -  M(\es))\bigg) \ =\\
  &\frac{1}{2^{|X \setminus S|}} \cdot \ssat(\neg M,S).
\end{align*}
By considering the definition of the $\shap$ score, we obtain that:
\begin{align*}
  \shap(M \vee x&, \one, x) \ =\\
  & \sum_{S \subseteq X \setminus \{x\}} \frac{|S|! (|X|-|S|-1)!}{|X|!} \cdot \bigg(\phi(M \vee x,\one,S \cup \{x\}) -  \phi(M \vee x,\one,S)\bigg) \ =\\
  & \sum_{k=0}^{n-1} \, \sum_{S \subseteq X \setminus \{x\} \,:\, |S|=k} \frac{|S|! (|X|-|S|-1)!}{|X|!} \cdot \frac{1}{2^{|X \setminus S|}} \cdot \ssat(\neg M,S) \ =\\
  & \sum_{k=0}^{n-1} \frac{k! (n-k-1)!}{n!} \cdot \frac{1}{2^{n-k}} \sum_{S \subseteq X \setminus \{x\} \,:\, |S|=k} \ssat(\neg M,S).
\end{align*}
This concludes the proof of Lemma~\ref{lem-fpras-disj}.
\end{proof}

\paragraph*{A 2-POS-DNF formula for cliques.}
Given a graph $G = (N,E)$ we define the formula~$\theta(G)$ in 2-POS-DNF by
\begin{eqnarray}\label{eq-def-theta}
\theta(G) \coloneqq \bigvee_{\substack{(a,b) \in (N \times N) \\ a \neq b \text{ and } (a,b) \not\in E} } a \wedge b.
\end{eqnarray}
Notice that the set of propositional variables occurring in $\theta(G)$
is the same as the set $N$ of nodes of $G$, given that $G$ satisfies
condition~\ref{cond:u}.
For $S\subseteq N$ we define
\begin{eqnarray*}
  \sclique(G,S) &\coloneqq& |\{ Y \mid Y \text{ is a clique of } G \text{ and } S \subseteq Y \}|.
 \end{eqnarray*}
Using Lemma~\ref{lem-fpras-disj}, we can relate the~$\shap$-score of the 2-POS-DNF formula~$\theta(G)\lor x$ to the quantities~$\sclique(G,S)$ for $S \subseteq N$ as follows.

\begin{lemma}
\label{lem:shap-to-cliques}
For every graph~$G=(N,E)$ and~$x \notin N$, letting~$n = |N|$
\rev{and $\one$ be the entity over $N\cup \{x\}$
  such that $\one(y) = 1$ for every $y \in N\cup \{x\}$},
we have:
\begin{align*}
\shap(\theta(G) \vee x, \one, x)  =  \frac{1}{2(n+1)} \sum_{k=0}^{n} \frac{k! (n-k)!}{n!} \cdot \frac{1}{2^{n-k}} \sum_{S \subseteq N \,:\, |S| = k} \sclique(G, S).
\end{align*}
\end{lemma}

\begin{proof}
The crucial observation is that the following relation holds for every $S \subseteq N$:
\begin{eqnarray*}
\sclique(G,S) &=& \ssat(\neg \theta(G), S).
\end{eqnarray*}
Indeed, this can be seen by defining the bijection that to every subset~$Y$ of~$N$ associates the valuation~$\nu_Y$ of~$\theta(G)$ such that~$\nu_Y(x) = 1$ if and only if~$x\in Y$,
and then checking that [$Y$ is a clique of~$G$ with~$S\subseteq Y$] if and only if [$\nu_Y \models \neg \theta(G)$ and~$\nu_Y(x) = 1$ for all~$x \in S$].
Therefore, we have from Lemma \ref{lem-fpras-disj} that
\begin{eqnarray}
\notag    \shap(\theta(G) \vee x, \one, x) & = & \sum_{k=0}^{n} \frac{k! (n-k)!}{(n+1)!} \cdot \frac{1}{2^{n+1-k}} \sum_{S \subseteq N \,:\, |S| = k} \ssat(\neg \theta(G), S)\\
\notag    & = & \frac{1}{2(n+1)} \sum_{k=0}^{n} \frac{k! (n-k)!}{n!} \cdot \frac{1}{2^{n-k}} \sum_{S \subseteq N \,:\, |S| = k} \sclique(G, S).
  \end{eqnarray}
\end{proof}

\paragraph*{Working with amplified graphs.}
In this proof, we consider an amplification technique from~\citep{DBLP:books/daglib/0073032}. More precisely, given a graph~$G=(N,E)$ and a natural number~$r \geq 1$, define the amplified
graph $G^r = (N^r, E^r)$ of $G$ as follows. For each $a \in N$, let $N^a \coloneqq \{a_1, \ldots, a_r\}$ and~$N^r  \coloneqq  \bigcup_{a \in N} N^a$,
and let $E^r$ be the following set of edges:
\begin{eqnarray*}
E^r & \coloneqq & \{(a_i,a_j) \mid a_i, a_j \in N^a \text{ and } i \neq
j \} \cup \{(a_i,b_j) \mid a_i \in N^a, b_j \in N^b \text{ and }
(a,b) \in E\}.
\end{eqnarray*}
Thus, the amplified graph $G^r$ is constructed by copying $r$ times
each node of $G$, \rev{by connecting for each node of~$G$ all of its copies as a clique}, and finally by
connecting copies $a_i$, $b_j$ of nodes $a$, $b$ of $G$, respectively,
whenever $a$ and $b$ are connected in $G$ \rev{by an edge}.
\rev{Notice in particular that~$G^1$ is simply~$G$. An example of this
construction for~$r=2$ is illustrated in Figure~\ref{fig:amplification}.}

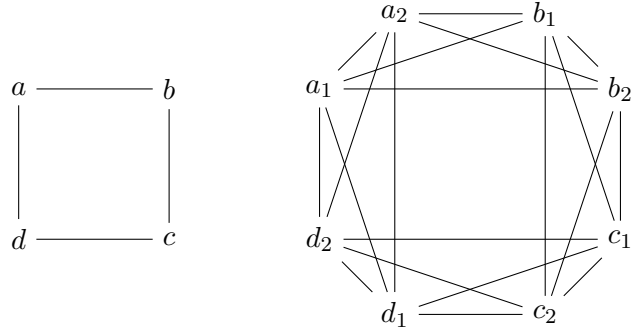
\begin{figure}
\begin{center}
\begin{tikzpicture}
\node (a) at (0,2) {$a$};
\node (b) at (2,2) {$b$};
\node (c) at (2,0) {$c$};
\node (d) at (0,0) {$d$};
\draw (a) -- (b) -- (c) -- (d) -- (a);
\node (a1) at (4,2) {$a_1$};
\node (a2) at (5,3) {$a_2$};
\node (b1) at (7,3) {$b_1$};
\node (b2) at (8,2) {$b_2$};
\node (c1) at (8,0) {$c_1$};
\node (c2) at (7,-1) {$c_2$};
\node (d1) at (5,-1) {$d_1$};
\node (d2) at (4,0) {$d_2$};
\draw (a1) -- (a2) -- (b1) -- (b2) -- (c1) -- (c2) -- (d1) -- (d2) -- (a1);

\draw (a1) -- (b1) -- (c1) -- (d1) -- (a1);
\draw (a2) -- (b2) -- (c2) -- (d2) -- (a2);
\draw (a2) -- (d1);
\draw (b1) -- (c2);
\draw (a1) -- (b2);
\draw (d2) -- (c1);

\end{tikzpicture}
\end{center}
\caption{\rev{Illustration of the amplification construction for~$r=2$. On the
left a graph~$G$, on the right the corresponding graph~$G^2$. We only
illustrate the construction for~$r=2$, as for~$r\geq 3$ this very quickly
becomes unreadable.}}
\label{fig:amplification}
\end{figure}

Let then~$T$ be a clique of~$G^r$. We say that~$T$ \emph{is a witness} of the clique~$\{
a \in N \mid N^a \cap T \neq \emptyset\}$ of~$G$.
Observe that a clique of~$G^r$ witnesses a unique clique of~$G$ \rev{by definition}, but that a clique of~$G$ can have
multiple witnessing cliques in~$G^r$.
Moreover, letting~$S$ be a clique of~$G$, we
write~$\wit(G^r,S)$ for the set of cliques of~$G^r$ that are witnesses of~$S$.
\rev{We then show the following three properties of~$G^r$:}

\begin{enumerate}[label=(\roman*)]
	\item if $S_1$, $S_2$ are distinct cliques of $G$,\label{item:disjoint}
then $\wit(G^r,S_1) \cap \wit(G^r,S_2) = \emptyset$.  \rev{Indeed, assuming
without loss of generality that~$S_1 \not\subseteq S_2$ (the case $S_2
\not\subseteq S_1$ being symmetrical) and letting~$a\in S_1 \setminus S_2$, it
is clear by that, by definition of being a witness, for any~$T \in \wit(G^r,S_1)$ we must have~$N^a \cap T \neq
\emptyset$, whereas for any~$T \in \wit(G^r,S_2)$ we have~$N^a \cap T =
\emptyset$.}
	\item if $S$ is a clique of $G$, then $|\wit(G^r,S)| =(2^r-1)^{|S|}$.\label{item:size}
\rev{Indeed, let~$T \in \wit(G^r,S)$. By definition of being a witness,~$T$
cannot contain any node of the form~$a_i$ for~$a \notin S$. Moreover, for
every~$a\in S$, by definition of being a witness again, the set~$T \cap N^a$
must be a non-empty subset of~$N^a$. Since for every~$a\in N$ there are
exactly~$(2^r-1)$ non-empty subsets of~$N^a$, this shows that $|\wit(G^r,S)|
\leq (2^r-1)^{|S|}$. Additionally, any set of the form~$\bigcup_{a\in S}S_a$
where each~$S_a$ is a non-empty subset of~$N^a$ is in fact in~$\wit(G^r,S)$: this is
simply because~$S$ itself is a clique of~$S$ and by definition of the
graph~$G^r$. This shows that $|\wit(G^r,S)| \geq (2^r-1)^{|S|}$, hence
$|\wit(G^r,S)| = (2^r-1)^{|S|}$ indeed.}
	\item for every natural number~$\ell$, if all cliques of~$G$ have size at most~$\ell$, then all cliques of~$G^r$ have size at most~$\ell \cdot r$.\label{item:upper}
\rev{Indeed, for any clique~$T$ of~$G^r$, letting~$S$ be the clique of~$S$
that~$T$ witnesses, it is clear that~$|T| \leq r\cdot |S|$ (with the
equality being reached when~$T\cap N^a = N^a$ for every~$a\in S$).}
\end{enumerate}
We use these properties to prove our next lemma.

\begin{lemma}\label{lem-ampl-is}
  Let $G = (N,E)$ be a graph,~$x\notin N$,~$n\coloneqq |N|$, and $m,r$ be natural numbers such that
  $1 \leq m \leq n$ and $r \geq 1$. Then:
  \begin{enumerate}
  \item[(a)] If every clique of $G$ contains at most $\lfloor \frac{m}{3} \rfloor$ nodes, then
    \begin{eqnarray*}
      \shap(\theta(G^r) \vee x, \one, x) & \leq &  \frac{2^{2\lfloor \frac{m}{3} \rfloor r + n}}{2^{n\cdot r+1}}.
      \end{eqnarray*}

  \item[(b)] If $G$ contains a clique with $m$ nodes, then
    \begin{eqnarray*}
      \shap(\theta(G^r) \vee x, \one, x) & \geq &  \frac{1}{(n\cdot r+1)} \cdot \frac{2^{mr}}{2^{n\cdot r+1}}.
    \end{eqnarray*}
    \end{enumerate}
    \end{lemma}

\begin{proof}
\begin{enumerate}
\item[(a)] Given that $|N^r| = n \cdot r$, we have from Lemma~\ref{lem:shap-to-cliques} that:
\begin{multline*}
  \shap(\theta(G^r) \vee x, \one, x) \ = \ \frac{1}{2(n\cdot r+1)} \sum_{k=0}^{n\cdot r} \frac{k! (n\cdot r-k)!}{(n\cdot r)!} \ \cdot\\ \frac{1}{2^{n\cdot r-k}} \sum_{S \subseteq N^r \,:\, |S|=k} \sclique(G^r,S).
\end{multline*}
Now, since every clique of~$G$ contains at most $\lfloor \frac{m}{3} \rfloor$ nodes, we have by Item~\ref{item:upper} that all cliques of~$G^r$ have size at most~$\lfloor \frac{m}{3} \rfloor \cdot r$.
Hence when~$k > \lfloor \frac{m}{3} \rfloor \cdot r$ it holds that~$\sclique(G^r,S) = 0$ for any~$S\subseteq N^r$ with~$|S|=k$. Therefore, we have that
\begin{align*}
  \shap(\theta(&G^r) \vee x, \one, x) \ = \\
  &\frac{1}{2(n\cdot r+1)} \sum_{k=0}^{\lfloor \frac{m}{3} \rfloor \cdot r} \frac{k! (n\cdot r-k)!}{(n\cdot r)!} \cdot \frac{1}{2^{n\cdot r-k}} \sum_{S \subseteq N^r \,:\, |S|=k} \sclique(G^r,S) \ \leq\\
  &\frac{1}{2(n\cdot r+1)} \cdot \frac{1}{2^{n\cdot r- \lfloor \frac{m}{3} \rfloor r}} \sum_{k=0}^{\lfloor \frac{m}{3} \rfloor \cdot r} \frac{k! (n\cdot r-k)!}{(n\cdot r)!} \sum_{S \subseteq N^r \,:\, |S|=k} \sclique(G^r,S) \ =\\
  &\frac{1}{2(n\cdot r+1)} \cdot \frac{1}{2^{n\cdot r- \lfloor \frac{m}{3} \rfloor r}} \sum_{k=0}^{\lfloor \frac{m}{3} \rfloor\cdot r} \frac{1}{\binom{n\cdot r}{k}} \sum_{S \subseteq N^r \,:\, |S|=k} \sclique(G^r,S).
\end{align*}

Now, defining $\sclique(G,\ell) \coloneqq |\{ Y \mid Y \text{ is a clique of } G \text{ and } |Y| = \ell \}|$,
observe that we have
\begin{eqnarray*}
\sum_{S \subseteq N^r \,:\, |S|=k} \sclique(G^r,S) &\rev{=}& \rev{\sum_{S \subseteq N^r \,:\, |S|=k} \ \ \sum_{\substack{Y \text{ clique of } G^r \\ S\subseteq Y}} 1}\\
&\rev{=}& \rev{\sum_{S \subseteq N^r \,:\, |S|=k} \ \ \sum_{\ell=k}^{|N^r|} \ \ \sum_{\substack{Y \text{ clique of } G^r \\ |Y| = \ell \\ S\subseteq Y}} 1}\\
&\rev{=}& \rev{\sum_{S \subseteq N^r \,:\, |S|=k} \ \ \sum_{\ell=k}^{\lfloor \frac{m}{3} \rfloor \cdot r}\ \ \sum_{\substack{Y \text{ clique of } G^r \\ |Y| = \ell \\ S\subseteq Y}} 1}\\
&\rev{=}& \rev{\sum_{\ell=k}^{\lfloor \frac{m}{3} \rfloor \cdot r} \ \ \sum_{S \subseteq N^r \,:\, |S|=k}\ \ \sum_{\substack{Y \text{ clique of } G^r \\ |Y| = \ell \\ S\subseteq Y}} 1}\\
&\rev{=}& \rev{\sum_{\ell=k}^{\lfloor \frac{m}{3} \rfloor \cdot r} \ \ \sum_{\substack{Y \text{ clique of } G^r \\ |Y| = \ell}} \ \ \sum_{S \subseteq Y \,:\, |S|=k}1}\\
&\rev{=}& \rev{\sum_{\ell=k}^{\lfloor \frac{m}{3} \rfloor \cdot r} \ \ \sum_{\substack{Y \text{ clique of } G^r \\ |Y| = \ell}} \binom{|Y|}{k}}\\
&\rev{=}& \rev{\sum_{\ell=k}^{\lfloor \frac{m}{3} \rfloor \cdot r} \ \ \binom{\ell}{k} \sum_{\substack{Y \text{ clique of } G^r \\ |Y| = \ell}}1}\\
&=& \sum_{\ell = k}^{\lfloor \frac{m}{3} \rfloor \cdot r} \binom{\ell}{k} \sclique(G^r,\ell),
\end{eqnarray*}
\rev{where the third equality is simply because all cliques of~$G^r$ have size at most $\lfloor \frac{m}{3} \rfloor \cdot r$, again by Item~\ref{item:upper}.}
Hence, we have that
\begin{align*}
\shap(\theta(&G^r) \vee x, \one, x) \ \leq \\
&\frac{1}{2(n\cdot r+1)} \cdot \frac{1}{2^{n\cdot r- \lfloor \frac{m}{3} \rfloor r}} \sum_{k=0}^{\lfloor \frac{m}{3} \rfloor\cdot r} \frac{1}{\binom{n\cdot r}{k}} \sum_{\ell = k}^{\lfloor \frac{m}{3} \rfloor \cdot r} \binom{\ell}{k} \sclique(G^r,\ell) \ \leq\\
& \frac{1}{2(n\cdot r+1)} \cdot \frac{1}{2^{n\cdot r- \lfloor \frac{m}{3} \rfloor r}} \sum_{k=0}^{\lfloor \frac{m}{3} \rfloor\cdot r} \frac{1}{\binom{n\cdot r}{k}} \sum_{\ell = k}^{\lfloor \frac{m}{3} \rfloor \cdot r} \binom{n \cdot r}{k} \sclique(G^r,\ell) \ = \\
& \frac{1}{2(n\cdot r+1)} \cdot \frac{1}{2^{n\cdot r- \lfloor \frac{m}{3} \rfloor r}} \sum_{k=0}^{\lfloor \frac{m}{3} \rfloor\cdot r} \sum_{\ell = k}^{\lfloor \frac{m}{3} \rfloor \cdot r} \sclique(G^r,\ell) \ \leq \\
& \frac{1}{2(n\cdot r+1)} \cdot \frac{1}{2^{n\cdot r- \lfloor \frac{m}{3} \rfloor r}} \sum_{k=0}^{\lfloor \frac{m}{3} \rfloor\cdot r} \sum_{\ell = 0}^{\lfloor \frac{m}{3} \rfloor \cdot r} \sclique(G^r,\ell) \ \leq \\
& \rev{\frac{(\lfloor \frac{m}{3} \rfloor\cdot r +1)}{2(n\cdot r+1)} \cdot \frac{1}{2^{n\cdot r- \lfloor \frac{m}{3} \rfloor r}} \sum_{\ell = 0}^{\lfloor \frac{m}{3} \rfloor \cdot r} \sclique(G^r,\ell) \ \leq }\\
& \rev{\frac{n\cdot r +1}{2(n\cdot r+1)} \cdot \frac{1}{2^{n\cdot r- \lfloor \frac{m}{3} \rfloor r}} \sum_{\ell = 0}^{\lfloor \frac{m}{3} \rfloor \cdot r} \sclique(G^r,\ell) \ =}\\
& \rev{\frac{1}{2} \cdot \frac{1}{2^{n\cdot r- \lfloor \frac{m}{3} \rfloor r}} \sum_{\ell = 0}^{\lfloor \frac{m}{3} \rfloor \cdot r} \sclique(G^r,\ell).}
\end{align*}
But notice that~$\sum_{\ell = 0}^{\lfloor \frac{m}{3} \rfloor \cdot r} \sclique(G^r,\ell)$ is simply the total number of cliques of~$G^r$.
Given a clique $S$ of~$G$ with $\ell$ elements, remember that by Item~\ref{item:size} we have that
$|\wit(G^r,S)| = (2^r - 1)^{\ell}$. Hence, given that each clique of $G$ has at most~$\lfloor \frac{m}{3} \rfloor$ nodes, that~$G$ has at most~$\binom{n}{\ell}$ cliques with $\ell$ nodes, and that every clique of~$G^r$ witnesses a unique clique of~$G$, we
have that the total number of cliques of~$G^r$ is bounded as~follows:
\begin{eqnarray*}
\sum_{\ell = 0}^{\lfloor \frac{m}{3} \rfloor \cdot r} \sclique(G^r,\ell) & \leq & \sum_{\ell' = 0}^{\lfloor \frac{m}{3} \rfloor} \binom{n}{\ell'} (2^r-1)^{\ell'}.
\end{eqnarray*}
Therefore, we conclude that:
\begin{eqnarray*}
 \shap(\theta(G^r) \vee x, \one, x) & \leq & \frac{1}{2} \cdot \frac{1}{2^{n\cdot r- \lfloor \frac{m}{3} \rfloor r}}  \sum_{\ell' = 0}^{\lfloor \frac{m}{3} \rfloor} \binom{n}{\ell'} (2^r-1)^{\ell'}\\
  & \leq & \frac{1}{2} \cdot \frac{1}{2^{n\cdot r- \lfloor \frac{m}{3} \rfloor r}}  \sum_{\ell' = 0}^{\lfloor \frac{m}{3} \rfloor} \binom{n}{\ell'} (2^r-1)^{\lfloor \frac{m}{3} \rfloor}\\
  & = & \frac{1}{2} \cdot \frac{1}{2^{n\cdot r- \lfloor \frac{m}{3} \rfloor r}} \cdot (2^r-1)^{\lfloor \frac{m}{3} \rfloor}  \sum_{\ell' = 0}^{\lfloor \frac{m}{3} \rfloor} \binom{n}{\ell'}\\
  & \leq & \frac{1}{2} \cdot \frac{1}{2^{n\cdot r- \lfloor \frac{m}{3} \rfloor r}} \cdot (2^r-1)^{\lfloor \frac{m}{3} \rfloor}  \sum_{\ell' = 0}^{n} \binom{n}{\ell'}\\
  & = & \frac{1}{2} \cdot \frac{1}{2^{n\cdot r- \lfloor \frac{m}{3} \rfloor r}} \cdot (2^r-1)^{\lfloor \frac{m}{3} \rfloor} \cdot 2^n \\
  & \leq & \frac{1}{2} \cdot \frac{1}{2^{n\cdot r- \lfloor \frac{m}{3} \rfloor r}} \cdot 2^{\lfloor \frac{m}{3} \rfloor r} \cdot 2^n\\
  & = & \frac{2^{2\lfloor \frac{m}{3} \rfloor r + n}}{2^{n\cdot r+1}}.
\end{eqnarray*}

\item[(b)]
First, we claim that~$G^r$ contains at least~$2^{mr}$ cliques. Indeed, let~$S$ be a
clique with~$m$ nodes that~$G$ contains (by hypothesis). Every subset~$S'$
of~$S$ is also a clique of~$G$. Now, by Item~\ref{item:disjoint} \rev{and Item~\ref{item:size}}, we have that
the number of witnesses of all the subcliques of~$S$ is equal to
\begin{align*}
\bigg|\bigcup_{S' \subseteq S} \wit(G^r,S')\bigg| \ &=  \ \sum_{S' \subseteq S} |\wit(G^r,S')|\\
&= \ \sum_{\ell = 0}^m \binom{m}{\ell} (2^r -1)^\ell\\
&= \ \sum_{\ell = 0}^m \binom{m}{\ell} (2^r -1)^\ell 1^{m-\ell}\\
&= \ (2^r -1 +1)^m\\
&= \ 2^{mr}.
\end{align*}

All terms in the summation of $\shap(\theta(G^r) \vee x, \one, x)$ are
non-negative.
Hence,
we have that:
\begin{align*}
  \shap(\theta(&G^r) \vee x, \one, x) \ = \\
  &\frac{1}{2(n\cdot r+1)} \sum_{k=0}^{n\cdot r} \frac{k! (n\cdot r-k)!}{n\cdot r!} \cdot \frac{1}{2^{n\cdot r-k}} \sum_{S \subseteq N^r \,:\, |S|=k} \sclique(G^r,S) \ \geq\\
  &\frac{1}{2(n\cdot r+1)} \cdot \frac{0! (n\cdot r)!}{(n\cdot r)!} \cdot \frac{1}{2^{n\cdot r}} \sum_{S \subseteq N^r \,:\, |S|=0} \sclique(G^r,S)\ = \\
  &\frac{1}{2(n\cdot r+1)} \cdot \frac{1}{2^{n\cdot r}} \cdot \sclique(G^r,\emptyset).
\end{align*}
But~$\sclique(G^r,\emptyset)$ is the total number of cliques of~$G^r$, which we have just proved to be greater than or equal to~$2^{mr}$. Therefore
\begin{eqnarray*}
  \shap(\theta(G^r) \vee x, \one, x) &\geq& \frac{1}{(n\cdot r+1)} \cdot \frac{2^{mr}}{2^{n\cdot r+1}},
\end{eqnarray*}
which concludes the proof of Lemma~\ref{lem-ampl-is}.
\end{enumerate}
\end{proof}

\paragraph*{A technical lemma.}
Last, we will need the following technical lemma.
\begin{lemma}\label{lem-one-p-e-one-m-e}
For every $\varepsilon \in (0,1)$, there exists $n_\varepsilon \in \mathbb{N}$ such that for every
$n,m \in \mathbb{N}$ with $n \geq n_\varepsilon$ and $m \geq 1$, it holds
that:
\begin{eqnarray*}
(1+\varepsilon) \cdot 2^{2\lfloor \frac{m}{3} \rfloor n^2 + n} & < & (1-\varepsilon) \cdot \frac{2^{mn^2}}{(n^3+1)}.
\end{eqnarray*}
\end{lemma}
The proof is straightforward and can be found in Appendix~\ref{app:tech-lem}.

\paragraph*{Putting it all together.}
We now have all the necessary ingredients to prove
Theorem~\ref{theo:non-FPRAS-2-POS-DNF}. Assume that the problem of computing the $\shap$ score for
Boolean classifiers given as 2-POS-DNF formulas admits an $\varepsilon$-PRA, for some fix\rev{ed}~$\varepsilon \in (0,1)$, that we will denote by~$\A$.
By using~$\A$, we define the
following $\bpp$ algorithm $\B$ for $\gclique$. Let $G = (N,E)$ be a
graph with $n = |N|$ and $m \in \{0, \ldots, n\}$. Then $\B(G,m)$
performs the following steps:
\begin{enumerate}
\item If $m = 0$, then return {\bf yes}.
\item Let $n_\varepsilon$ be the constant in Lemma \ref{lem-one-p-e-one-m-e}. If $n \leq n_\varepsilon$, then by performing an exhaustive search, check whether
$G$ contains a clique with $m$ nodes or if every clique of $G$ contains
at most $\lfloor \frac{m}{3} \rfloor$ nodes. In the former case return
{\bf yes}, in the latter case return {\bf no}.
\item Construct the formula $\theta(G^{n^2}) \vee x$, where $x$ is a fresh feature (not occurring in $N^{n^2}$). Notice that $\theta(G^{n^2}) \vee x$ is a formula in 2-POS-DNF.
\item \label{alg-rest-step} Use algorithm~$\A$ to compute $s \coloneqq \A(\theta(G^{n^2}) \vee x, \one, x)$.
\item If
\begin{eqnarray*}
s &>& (1+\varepsilon) \cdot \frac{2^{2\lfloor \frac{m}{3} \rfloor n^2 + n}}{2^{n^3+1}},
\end{eqnarray*}
then return {\bf yes}; otherwise return {\bf no}.
\end{enumerate}
Algorithm $\B$ works in polynomial time since $n_\varepsilon$ is a fixed natural number (given that $\varepsilon$ is a fixed value in $(0,1)$),
amplified graph $G^{n^2}$ can be computed in polynomial time from $G$, formula $\theta(G^{n^2}) \vee x$ can be constructed in polynomial time from $G^{n^2}$, algorithm $\A(\theta(G^{n^2}) \vee x, \one, x)$ works in time $p(\|\theta(G^{n^2}) \vee x\| + \|\one\|)$ for a polynomial $p(\cdot)$, and bound
$(1+\varepsilon) \cdot \frac{2^{2\lfloor \frac{m}{3} \rfloor n^2 + n}}{2^{n^3+1}}$
can be computed in polynomial time in $n$.
Therefore, to conclude the proof we need to show that the error probability of algorithm $\B$ is bounded by $\frac{1}{4}$.
\begin{itemize}
\item Assume that every clique of $G$ contains at most $\lfloor \frac{m}{3} \rfloor$ nodes. Then the error probability of algorithm $\B$ is equal to $\Pr(\B(G,m)$ returns {\bf yes}$)$. By using the definition of $\B$ and Lemma \ref{lem-ampl-is} (a) with $r = n^2$, we obtain that:
\begin{align*}
\Pr(&\B(G,m) \text{ returns {\bf yes}}) \ =\\
&\Pr\bigg(\A(\theta(G^{n^2}) \vee x, \one, x) > (1+\varepsilon) \cdot \frac{2^{2\lfloor \frac{m}{3} \rfloor n^2 + n}}{2^{n^3+1}}\bigg) \ \leq\\
&\Pr\bigg(\A(\theta(G^{n^2}) \vee x, \one, x) > (1+\varepsilon) \cdot \shap(\theta(G^{n^2}) \vee x, \one, x)\bigg) \ \leq\\
&\Pr\bigg(\A(\theta(G^{n^2}) \vee x, \one, x) > (1+\varepsilon) \cdot \shap(\theta(G^{n^2}) \vee x, \one, x) \ \vee\\
&\phantom{\Pr\bigg(}\A(\theta(G^{n^2}) \vee x, \one, x) < (1-\varepsilon) \cdot \shap(\theta(G^{n^2}) \vee x, \one, x)\bigg) \ =\\
&1 - \Pr\bigg(\A(\theta(G^{n^2}) \vee x, \one, x) \leq (1+\varepsilon) \cdot \shap(\theta(G^{n^2}) \vee x, \one, x) \ \wedge\\
&\phantom{1 - \Pr\bigg(}\A(\theta(G^{n^2}) \vee x, \one, x) \geq (1-\varepsilon) \cdot \shap(\theta(G^{n^2}) \vee x, \one, x)\bigg) \ =\\
&1 - \Pr\bigg(\bigg|\A(\theta(G^{n^2}) \vee x, \one, x) - \shap(\theta(G^{n^2}) \vee x, \one, x)\bigg| \ \leq\\
&\hspace{200pt} \varepsilon \cdot \bigg|\shap(\theta(G^{n^2}) \vee x, \one, x)\bigg|\bigg) \ \leq\\
&1 - \frac{3}{4} \ = \ \frac{1}{4},
\end{align*}
where in the last step we use the fact that $\A$ is an $\varepsilon$-PRA for the problem of computing the $\shap$ score for Boolean classifiers given as 2-POS-DNF formulas, and the fact that $\shap(\theta(G^{n^2}) \vee x, \one, x) \geq 0$ (\rev{as can be seen from} Lemma~\ref{lem:shap-to-cliques}).

\item Assume now that $G$ contains a clique with $m$ nodes. We can assume that~$n > n_\varepsilon$, since otherwise we know that~$\B$ returns the correct answer in Step 2. Given that $n > n_\varepsilon$, we know from Lemma \ref{lem-one-p-e-one-m-e} that:
\begin{eqnarray}\label{eq-prop-error-no}
(1+\varepsilon) \cdot \frac{2^{2\lfloor \frac{m}{3} \rfloor n^2 + n}}{2^{n^3+1}} & < & (1-\varepsilon) \cdot \frac{1}{(n^3+1)} \cdot \frac{2^{mn^2}}{2^{n^3+1}}.
\end{eqnarray}
The error probability of algorithm $\B$ is equal to $\Pr(\B(G,m)$
returns {\bf no}$)$. By using the definition of $\B$,
Lemma \ref{lem-ampl-is} (b) with $r = n^2$, and inequality \eqref{eq-prop-error-no}, we
obtain that:
\begin{align*}
\Pr(&\B(G,m) \text{ returns {\bf no}}) \ =\\
&\Pr\bigg(\A(\theta(G^{n^2}) \vee x, \one, x) \leq (1+\varepsilon) \cdot \frac{2^{2\lfloor \frac{m}{3} \rfloor n^2 + n}}{2^{n^3+1}}\bigg) \ \leq\\
&\Pr\bigg(\A(\theta(G^{n^2}) \vee x, \one, x) < (1-\varepsilon) \cdot \frac{1}{(n^3+1)} \cdot \frac{2^{mn^2}}{2^{n^3+1}}\bigg) \ \leq\\
&\Pr\bigg(\A(\theta(G^{n^2}) \vee x, \one, x) < (1-\varepsilon) \cdot \shap(\theta(G^{n^2}) \vee x, \one, x)\bigg) \ \leq\\
&\Pr\bigg(\A(\theta(G^{n^2}) \vee x, \one, x) < (1-\varepsilon) \cdot \shap(\theta(G^{n^2}) \vee x, \one, x) \ \vee\\
&\phantom{\Pr\bigg(}\A(\theta(G^{n^2}) \vee x, \one, x) > (1+\varepsilon) \cdot \shap(\theta(G^{n^2}) \vee x, \one, x)\bigg) \ =\\
&1 - \Pr\bigg(\A(\theta(G^{n^2}) \vee x, \one, x) \geq (1-\varepsilon) \cdot \shap(\theta(G^{n^2}) \vee x, \one, x) \ \wedge\\
&\phantom{1 - \Pr\bigg(}\A(\theta(G^{n^2}) \vee x, \one, x) \leq (1+\varepsilon) \cdot \shap(\theta(G^{n^2}) \vee x, \one, x)\bigg) \ =\\
&1 - \Pr\bigg(\bigg|\A(\theta(G^{n^2}) \vee x, \one, x) - \shap(\theta(G^{n^2}) \vee x, \one, x)\bigg| \ \leq\\
&\hspace{200pt} \varepsilon \cdot \bigg| \shap(\theta(G^{n^2}) \vee x, \one, x)\bigg|\bigg) \ \leq\\
&1 - \frac{3}{4} \ = \ \frac{1}{4},
\end{align*}
where the last step is obtained as in the previous case. This concludes the proof of Theorem~\ref{theo:non-FPRAS-2-POS-DNF}.
\end{itemize}

\subsection{Proof of Corollary~\ref{cor:cor}}
\label{subsec:corollaries}

First, we explain how the result for 2-NEG-DNF can be obtained from the result for 2-POS-DNF (Theorem~\ref{theo:non-FPRAS-2-POS-DNF}).
For an entity~$\es$ over a set of features~$X$, let us write~$\invpol{\es}$ the
entity defined by $\invpol{\es}(x) \coloneqq 1 - \es(x)$ for every~$x\in X$.
For a Boolean classifier~$M$ over variables~$X$, we write~$\invpol{M}$ the
Boolean classifier defined by~$\invpol{M}(\es) \coloneqq M(\invpol{\es})$; in
other words, we have changed the polarity of all the features.
Then the following holds:

\begin{lemma}
\label{lem:invpol}
Let~$M$ be a Boolean classifier over~$X$,~$\es$ an entity over~$X$ and~$x\in X$.
We have
\begin{eqnarray*}
\shap(M,\es,x) &=& \shap(\invpol{M},\invpol{\es},x).
\end{eqnarray*}
\end{lemma}
\begin{proof}
For every~$S \subseteq X$, we have that
\begin{align*}
\phi(\invpol{M},\invpol{\es},S) \ &= \ \sum_{\es' \in \asm(\invpol{\es},S)} \frac{1}{2^{|X\setminus S|}} \invpol{M}(\es')\\
&= \ \sum_{\es' \in \asm(\es,S)} \frac{1}{2^{|X\setminus S|}} \invpol{M}(\invpol{\es'})\\
&= \ \sum_{\es' \in \asm(\es,S)} \frac{1}{2^{|X\setminus S|}} M(\invpol{\invpol{\es'}})\\
&= \ \sum_{\es' \in \asm(\es,S)} \frac{1}{2^{|X\setminus S|}} M(\es')\\
&= \ \phi(M,\es,S),
\end{align*}
which proves the lemma.
\end{proof}
But it is clear that if~$M$ is a Boolean classifier defined with a 2-POS-DNF
formula~$\varphi$, then~$\invpol{M}$ is the Boolean classifier defined by the
formula obtained from~$\varphi$ by replacing every occurrence of a variable~$x$
by the literal~$\lnot x$; that is, a formula in 2-NEG-DNF.  Combining
Theorem~\ref{theo:non-FPRAS-2-POS-DNF} with Lemma~\ref{lem:invpol} then
establishes Corollary~\ref{cor:cor} for the case of 2-NEG-DNF.
The same argument shows that the results for 2-POS-CNF and 2-NEG-CNF are equivalent, so all that
remains to do in this section is to prove the result, say, for 2-NEG-CNF.
But this simply comes from the fact that the negation of a 2-POS-DNF formula is a 2-NEG-CNF formula, and from the fact that,
for a Boolean classifier~$M$ we have that
\begin{eqnarray}
\label{eq:shap-neg}
\shap(M,\es,x) &=& - \shap(\lnot M,\es,x).
\end{eqnarray}
This last property can be easily shown by considering that $\phi(M,\es,S) + \phi(\lnot M, \es, S) = 1$, for every Boolean classifier $M$ over a set of features $X$ and every subset $S$ of $X$.
This concludes the proof of Corollary~\ref{cor:cor}.

%% file: comparing.tex
We now turn our attention to the problem of comparing the $\shap$-score of the features of an
entity.  
The reason we are interested in this problem is that it could be the
case that computing the~$\shap$-score exactly or approximately is
intractable (as we have shown in the last two sections), but yet that we are able to compare the relevance of
features. 
In this section we will again use the uniform distribution and drop
the subscripts~$\mathcal{U}$. We now define the problem \rev{that} we consider.
Given a class~$\C$ of Boolean classifiers, the input of the
problem $\compshap(\C)$ is a Boolean classifier~$M \in \C$ over a set
of features~$X$, an entity~$\es$ over $X$ and two features~$x,y \in
X$, and the question to answer is whether $\shap(M,\es,x) >
\shap(M,\es,y)$.\footnote{Note that~$>$ can be replaced by any of~$\geq, <$, or~$\leq$, as this does not change the problem.}

We prove that this problem is unlikely to be tractable,
even for very restricted Boolean classifiers.
A formula $\varphi$ is said to be in 3-POS-DNF if $\varphi$ is in
POS-DNF and every disjunct in $\varphi$ contains at most three
variables. Then we have that:

\begin{theorem}\label{theo-comparing-bpp-np-rp}
Let $\thpdnf$ be the class of Boolean classifiers given as 3-POS-DNF
formulas. If $\compshap(\thpdnf) \in \bpp$, then $\np = \rp$.
\end{theorem}
Moreover, as in the previous section, Lemma~\ref{lem:invpol} and Equation~\eqref{eq:shap-neg}
directly imply that this result extends to 
the closure of 3-POS-DNF formulas by duals and negations:

\begin{corollary}\label{cor-comparing-bpp-np-rp}
Let
$\mathcal{F} \in \{\text{3-POS-DNF}$, $\text{3-NEG-DNF}$, $\text{3-POS-CNF}$, $\text{3-NEG-CNF}\}$
and $\mathcal{C}_{\mathcal{F}}$ be the class of Boolean classifiers
given as formulas in $\mathcal{F}$. If
$\compshap(\mathcal{C}_{\mathcal{F}}) \in \bpp$, then~$\np = \rp$.
\end{corollary}
Hence, all we have to do is to prove Theorem~\ref{theo-comparing-bpp-np-rp}. The
proof is again technical and divided in two parts, highlighted
in bold in the rest of this section.

\paragraph*{An intermediate problem: $\rshap(\C)$.}
As a first step, we consider a variation of the problem
$\compshap$. Let $\C$ be a class of Boolean classifiers. The input of
the problem $\rshap(\C)$ is a pair of Boolean
classifiers~$M,M' \in \C$ over the same set of features~$X$ and the question
to answer is whether
\begin{eqnarray*}
\shap(M \wedge \neg x,\es,x) &>& \shap(M' \wedge \neg y,\es',y).
\end{eqnarray*}
where $x,y$ are two features not occurring in $X$, $\es$ is the entity
over $X \cup \{x\}$ such that $\es(x) = 0$ and $\es(z) = 1$ for every
$z \in X$, and $\es'$ is the entity over $X \cup \{y\}$ such that
$\es'(y) = 0$ and $\es'(z) = 1$ for every $z \in X$. 
We claim that~$\rshap$ over 3-NEG-CNF formulas can be reduced in polynomial time to~$\compshap$ over 2-POS-DNF formulas:

\begin{lemma}
\label{lem:dist-to-comp}
Let $\twncnf$ be the class of Boolean classifiers given as 2-NEG-CNF
formulas. There exists a polynomial-time many-one reduction from
$\rshap(\twncnf)$ to~$\compshap(\thpdnf)$.
\end{lemma}
This in particular implies that if $\compshap(\thpdnf)$ is in $\bpp$ then so is\linebreak $\rshap(\twncnf)$.
We need three lemmas to prove Lemma~\ref{lem:dist-to-comp}.
\begin{lemma}\label{lem-diff-shap}
Given a Boolean classifier~$M$ over a set of features $X$, an entity $\es$ over $X$ and features $x,y \in X$, we have that:
\begin{multline*}
\shap(M,\es,x) - \shap(M,\es,y) \ = \\
\sum_{S \subseteq X \setminus \{x,y\}} \frac{|S|!\, (|X|-|S|-2)!}{(|X|-1)!} \big(\phi(M, \es,S \cup \{x\}) - \phi(M, \es,S \cup \{y\})\big).
\end{multline*}
\end{lemma}
\begin{proof}
We prove this in Appendix~\ref{app:diff-shap}.
\end{proof}

\begin{lemma}\label{lem-fpras-conj}
  Let $X$ be a set of features, $n = |X|$, $x \in X$, $M$ be a Boolean
  classifier over $X \setminus \{x\}$ and $\es$ be the entity over $X$
  such that $\es(x) = 0$ and $\es(y) = 1$ for every $y \in X\setminus \{x\}$. Then we have that:
\begin{eqnarray*}
    \shap(M \wedge \lnot x, \es, x) & = & \sum_{k=0}^{n-1} \frac{k! (n-k-1)!}{n!} \cdot \frac{1}{2^{n-k}} \sum_{S \subseteq X \setminus \{x\} \,:\, |S| = k} \ssat(M, S).
\end{eqnarray*}
\end{lemma}
\begin{proof}
This can be established in the same way as Lemma~\ref{lem-fpras-disj} is proved.
\end{proof}

\begin{lemma}\label{lem-reduction-comp-dist}
Let $X$ be a set of features, $n = |X|$, $x,y \in X$, $M,M'$ be Boolean
  classifiers over $X \setminus \{x,y\}$ and $\es$ be the entity over $X$
  such that $\es(x) = 0$, $\es(y) = 0$ and $\es(z) = 1$ for every $z \in X \setminus \{x,y\}$. Then we have that:
\begin{multline*}
\shap(M \wedge \neg x, \es_{|X \setminus \{y\}}, x) - \shap(M' \wedge \neg y, \es_{|X \setminus \{x\}}, y) \ = \\
    \shap((\neg M' \wedge x) \vee (\neg M \wedge y), \es, y) -
    \shap((\neg M' \wedge x) \vee (\neg M \wedge y), \es, x).
\end{multline*}
\end{lemma}

\begin{proof}
Let $S \subseteq X \setminus \{x,y\}$. We have that:
\begin{align*}
\phi((\neg M' \wedge x) \vee (\neg M \wedge y&), \es, S \cup \{x\}) \ = \\
&\sum_{\es' \in \asm(\es,S\cup\{x\})} \frac{1}{2^{|X \setminus (S \cup \{x\})|}} \cdot
((\neg M' \wedge x) \vee (\neg M \wedge y))(\es') \ = \\
&\frac{1}{2^{|X \setminus (S \cup \{x\})|}} \sum_{\es' \in \asm(\es,S\cup\{x\})}
(\neg M \wedge y)(\es') \ =\\
&\frac{1}{2^{|X \setminus (S \cup \{x\})|}} \bigg(\sum_{\es' \in \asm(\es,S\cup\{x\})\,:\,\es'(y)=1}
(\neg M \wedge y)(\es') \ +\\
&\hspace{40pt} \sum_{\es' \in \asm(\es,S\cup\{x\})\,:\,\es'(y)=0}
(\neg M \wedge y)(\es')\bigg) \ =\\
&\frac{1}{2^{|X \setminus (S \cup \{x\})|}} \sum_{\es' \in \asm(\es,S\cup\{x\})\,:\,\es'(y)=1}
(\neg M)(\es') \ =\\
&\frac{1}{2^{|X \setminus (S \cup \{x\})|}} \sum_{\es' \in \asm(\es_{|X\setminus \{x,y\}},S)}
(\neg M)(\es') \ =\\
&\frac{1}{2^{|X \setminus (S \cup \{x\})|}} \cdot \ssat(\neg M,S).
\end{align*}
Similarly we have that:
\begin{eqnarray*}
\phi((\neg M' \wedge x) \vee (\neg M \wedge y), \es, S \cup \{y\}) & = &
\frac{1}{2^{|X \setminus (S \cup \{y\})|}} \cdot \ssat(\neg M',S).
\end{eqnarray*}
Hence, we have from Lemma \ref{lem-diff-shap} (applied to the classifier~$(\neg M' \wedge x) \vee (\neg M \wedge y)$) that:
\begin{align*}
    &\shap((\neg M' \wedge x) \vee (\neg M \wedge y), \es, y) -
    \shap((\neg M' \wedge x) \vee (\neg M \wedge y), \es, x) \ = \\
& \hspace{20pt} \sum_{S \subseteq X \setminus \{x,y\}} \frac{|S|! (|X|-|S|-2)!}{(|X|-1)!} 
\cdot \bigg(\frac{1}{2^{|X \setminus (S \cup \{y\})|}} \cdot \ssat(\neg M',S) \
-\\
& \hspace{220.5pt} \frac{1}{2^{|X \setminus (S \cup \{x\})|}} \cdot \ssat(\neg M,S)\bigg) \ =\\
    & \hspace{20pt} \bigg(\sum_{k=0}^{n-2} \frac{k! (n-k-2)!}{(n-1)!} \cdot \frac{1}{2^{n-k-1}} \sum_{S \subseteq X \setminus \{x,y\} \,:\, |S| = k} \ssat(\neg M', S)\bigg)\ -\\ 
    &\hspace{70pt} \bigg(\sum_{k=0}^{n-2} \frac{k! (n-k-2)!}{(n-1)!} \cdot \frac{1}{2^{n-k-1}} \sum_{S \subseteq X \setminus \{x,y\} \,:\, |S| = k} \ssat(\neg M, S)\bigg) \ =\\
    & \hspace{20pt} \bigg(\sum_{k=0}^{n-2} \frac{k! (n-k-2)!}{(n-1)!} \cdot \frac{1}{2^{n-k-1}} \sum_{S \subseteq X \setminus \{x,y\} \,:\, |S| = k} \big(2^{n-2-k} - \ssat(M', S)\big)\bigg)\ -\\ 
    &\hspace{70pt} \bigg(\sum_{k=0}^{n-2} \frac{k! (n-k-2)!}{(n-1)!} \cdot \frac{1}{2^{n-k-1}} \sum_{S \subseteq X \setminus \{x,y\} \,:\, |S| = k} \big(2^{n-2-k} - \ssat(M, S)\big)\bigg) \ =\\
    & \hspace{20pt} \bigg(\sum_{k=0}^{n-2} \frac{k! (n-k-2)!}{(n-1)!} \cdot \frac{1}{2^{n-k-1}} \sum_{S \subseteq X \setminus \{x,y\} \,:\, |S| = k} \ssat(M, S)\bigg)\ -\\ 
    &\hspace{70pt} \bigg(\sum_{k=0}^{n-2} \frac{k! (n-k-2)!}{(n-1)!} \cdot \frac{1}{2^{n-k-1}} \sum_{S \subseteq X \setminus \{x,y\} \,:\, |S| = k} \ssat(M', S))\bigg)
\end{align*}
We can now use Lemma~\ref{lem-fpras-conj} to conclude the proof.

\end{proof}
\begin{proof}{{\bf of Lemma~\ref{lem:dist-to-comp}}}
Let~$M,M'$ be Boolean classifiers in $\twncnf$ over a set of
variables~$X$, which are inputs of~$\rshap$. Remember that we want to
decide whether
\begin{eqnarray*}
\shap(M \wedge \neg x,\es,x) &>& \shap(M' \wedge \neg y,\es',y),
\end{eqnarray*}
where $x,y$ are two features not occurring in $X$, $\es$ is the entity over $X
\cup \{x\}$ such that $\es(x) = 0$ and $\es(z) = 1$ for every $z \in X$, and
$\es'$ is the entity over $X \cup \{y\}$ such that $\es'(y) = 0$ and $\es'(z) =
1$ for every $z \in X$.  Observe that~$\neg M$ is a formula in 2-POS-DNF, and
that~$\neg M \land y$ is a formula in 3-POS-DNF (obtained by adding~$y$ to
every term of~$\neg M$). Similarly, $\neg M' \land x$ is a formula in 3-POS-DNF.
Therefore, $(\neg M' \wedge x) \vee (\neg M \wedge y)$ is also a formula in
3-POS-DNF, and then Lemma~\ref{lem-reduction-comp-dist} can be used to establish the
polynomial-time many-one reduction.
\end{proof}
This concludes the first part of the proof. Next, we show
that~$\rshap$ is unlikely to be in~$\bpp$ for the class of Boolean
classifiers given as 2-NEG-CNF formulas.

\begin{sloppypar}
\paragraph*{Intractability of~$\rshap$ over 2-NEG-CNF.}
We now prove that if $\rshap(\twncnf)\in\bpp$, then $\np=\rp$. Notice
that this implies that Theorem~\ref{theo-comparing-bpp-np-rp} holds,
given that in the previous section we prove that if
$\compshap(\thpdnf) \in \bpp$, then
$\rshap(\twncnf) \in \bpp$. 
\end{sloppypar}

In what follows, we consider the same class of graphs as in the proof
of Theorem \ref{theo:non-FPRAS-2-POS-DNF}, that is, the class of all
undirected and loop-free graphs $G = (N,E)$ containing at least two
isolated nodes (see Condition \ref{cond:u}). Besides, recall the
definition of the 2-POS-DNF formula~$\theta(G)$ from
Equation~\eqref{eq-def-theta}. We start our proof with the following
lemma:

\begin{lemma}
\label{lem:shap-cliques-again}
Let~$G=(N,E)$ be a graph with $n = |N|$ nodes, $x$ be a propositional
variable not occurring in $\theta(G)$, and $\es$ be an entity such
that $\es(x) = 0$ and $\es(y)= 1$ for each variable~$y$ occurring in
$\theta(G)$. Then we have that:
\begin{eqnarray*}
  \shap(\neg \theta(G) \wedge \neg x, \es, x) & = & \frac{1}{2(n+1)} \sum_{k=0}^{n} \frac{k!
  (n-k)!}{n!} \cdot \frac{1}{2^{n-k}} \sum_{S \subseteq  N \,:\, |S|=k} \sclique(G,S).
\end{eqnarray*}
\end{lemma}
\begin{proof}
The proof is similar to that of Lemma~\ref{lem:shap-to-cliques}, but this time we use Lemma~\ref{lem-fpras-conj} instead of Lemma~\ref{lem-fpras-disj}.
\end{proof}
We will also need the fact that this quantity cannot be approximated:
\begin{lemma}
\label{lem:nopras}
For every $\varepsilon \in (0,1)$, the following problem does not
admit an $\varepsilon$-PRA, unless~$\np = \rp$.  Given as input a
graph~$G$,
compute
$\shap(\lnot \theta(G) \land \neg x, \es, x)$, where $x$ is a fresh
variable not occurring in $\theta(G)$ and~$\es$ is an entity such that $\es(x)
= 0$ and $\es(y)=1$ for each variable $y$ occurring in $\theta(G)$.
\end{lemma}
\begin{proof}
From the proof of Theorem~\ref{theo:non-FPRAS-2-POS-DNF}, we get that there is
no~$\varepsilon$-PRA for the following problem, unless~$\np = \rp$. Given as
input a graph~$G$,
compute the quantity:
\[\frac{1}{2(n+1)} \sum_{k=0}^{n} \frac{k!
  (n-k)!}{n!} \cdot \frac{1}{2^{n-k}} \sum_{S \subseteq  N \,:\, |S|=k} \sclique(G,S).\]
Hence, Lemma~\ref{lem:shap-cliques-again} allows us to conclude that Lemma \ref{lem:nopras} holds.
\end{proof}
The idea of our proof is then the following. We will prove that if
$\rshap(\twncnf)$ is in~$\bpp$, then there exists a $\frac{9}{10}$-PRA for the
problem in Lemma~\ref{lem:nopras}, hence deducing that $\np = \rp$.
To this end, we introduce a class of graphs and calculate the values of
$\shap(\lnot \theta(G) \land \neg x, \es, x)$ over them.  Given~$n,t \in
\mathbb{N}$ with~$t \leq n$, we define the graph~$G_{n,t} = (N_{n,t},E_{n,t})$
with $N_{n,t} = \{a_1,\ldots,a_t\} \cup \{b_1,\ldots,b_{n-t}\}$ and $E_{n,t} =
\{(b_i,b_j) \mid i,j \in \{1,\ldots n-t\} \text{ and } i\neq j\}$. In other words,
$G_{n,t}$ is the disjoint union of a clique with~$n-t$ nodes and of a graph
consisting of~$t$ isolated nodes. \rev{For instance, the graph~$G_{8,3}$ is illustrated
in Figure~\ref{fig:gnt}.}
We calculate the values~$\shap(\lnot \theta(G_{n,t}) \land \neg x, \es, x)$ in the next lemma.

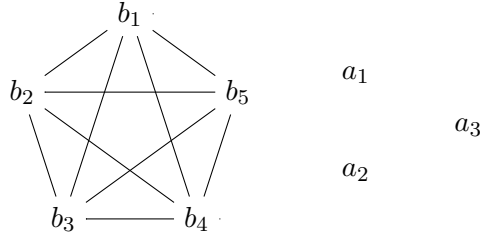
\begin{figure}
\begin{center}
\begin{tikzpicture}
\def\ngon{5}
\node[regular polygon,regular polygon sides=\ngon,minimum size=3cm] (p) {};
\foreach\x in {1,...,\ngon}{\node (p\x) at (p.corner \x){$b_\x$};}
\foreach\x in {1,...,\numexpr\ngon-1\relax}{
  \foreach\y in {\x,...,\ngon}{
    \draw (p\x) -- (p\y);
  }
}
\node at (3,.7) {$a_1$};
\node at (3,-.6) {$a_2$};
\node at (4.5,0) {$a_3$};
\end{tikzpicture}
\end{center}
\caption{\rev{Illustration of the graph~$G_{8,3}$.}}
\label{fig:gnt}
\end{figure}

\begin{lemma}\label{lem-graph-k-p}
Let $n,t \in \mathbb{N}$ such that~$t \leq n$, $x$ be a variable not occurring in~$\theta(G_{n,t})$ and~$\es$ an entity such $\es(x) =0$ and $\es(y) = 1$ for every variable $y$ occurring in $\theta(G_{n,t})$. Then:
\begin{eqnarray*}
\shap(\lnot \theta(G_{n,t}) \land \neg x, \es, x) & = &   \frac{t}{n(n+1)2^n} + \frac{1}{(t+1)2^{t+1}}.
\end{eqnarray*}
\end{lemma}

\begin{proof}
By Lemma~\ref{lem:shap-cliques-again} we have
\begin{eqnarray*}
  \shap(\neg \theta(G_{n,t}) \wedge \neg x, \es, x) & = & \frac{1}{2(n+1)} \sum_{k=0}^{n} \frac{k!
  (n-k)!}{n!} \cdot \frac{1}{2^{n-k}} \sum_{S \subseteq  N_{n,t} \,:\, |S|=k} \sclique(G,S).
\end{eqnarray*}
Let~$A$ be the set of the~$t$ isolated nodes of~$G_{n,t}$ and~$B$ be the set of
the~$n-t$ nodes of~$G_{n,t}$ that form a clique. In the above expression, it is
clear that if~$S$ intersects both~$A$ and~$B$ then $\sclique(G,S)$ is empty
(because~$S$ is already not a clique).  Moreover, if~$S$ is included in~$A$ and
is not empty, then $\sclique(G,S)$ is~$1$ if~$|S| = 1$ and is empty otherwise.
Finally, if~$S$ is included in~$B$ then we have $\sclique(G,S) = 2^{n-t-|S|}$.
Thus, we obtain

\begin{align*}
  \shap(\neg \theta(G_{n,t}) \wedge \neg x, \es, x) \ & = \  \frac{1}{2(n+1)}  \frac{1!
  (n-1)!}{n!} \cdot \frac{1}{2^{n-1}} \sum_{S \subseteq  A \,:\, |S|=1} 1\\
   & \hspace{3em} +  \frac{1}{2(n+1)} \sum_{k=0}^{n-t} \frac{k!
    (n-k)!}{n!} \cdot \frac{1}{2^{n-k}} \sum_{S \subseteq  B \,:\, |S|=k} 2^{n-t-k}\\
   & = \ \frac{t}{n(n+1)2^n}\\
   & \hspace{3em} +  \frac{1}{2(n+1)} \sum_{k=0}^{n-t} \frac{k!
    (n-k)!}{n!} \cdot \frac{1}{2^{n-k}} \binom{n-t}{k} 2^{n-t-k}\\
   & = \ \frac{t}{n(n+1)2^n} +  \frac{1}{2(n+1)2^t} \sum_{k=0}^{n-t} \frac{(n-t)! (n-k)!}{n! (n-t-k)!}
\end{align*}
We prove in Appendix \ref{app:sum-prod} that the sum on the right is equal to~$\frac{n+1}{t+1}$, so we get
\begin{eqnarray*}
  \shap(\neg \theta(G_{n,t}) \wedge \neg x, \es, x) & =&  \frac{t}{n(n+1)2^n} +  \frac{1}{2(n+1)2^t} \cdot \frac{n+1}{t+1}\\
 & = & \frac{t}{n(n+1)2^n} + \frac{1}{(t+1)2^{t+1}},
\end{eqnarray*}
which was to be shown.
\end{proof}

As a final ingredient, we will also need the following simple observation:
\begin{lemma}\label{lem-l-u-b-Phi}
Let $G = (N,E)$ be a graph with $n = |N|$, $x$ be a variable not occurring in~$\theta(G)$ and~$\es$ an entity such $\es(x) = 0$ and $\es(y) = 1$ for every variable $y$ occurring in $\theta(G)$. Assuming, without loss of generality, that $G_{n,t}$ has the same nodes as $G$, for each $t \in \{0, \ldots, n\}$, we have that:
\begin{align*}
 \shap(\neg \theta(G_{n,n}) \wedge \neg x, \es, x)\ \leq \ \shap(\lnot \theta(G) \land \neg x, \es, x) \ \leq \ \shap(\neg \theta(G_{n,2}) \wedge \neg x, \es, x).
\end{align*}
\end{lemma}
\begin{proof}
By looking at Lemma~\ref{lem:shap-cliques-again}, we see that~$\shap(\lnot
\theta(G) \land \neg x, \es, x)$ can only increase when edges are added to the
graph $G$. Therefore, this quantity is lower bounded by the quantity for
the graph with~$n$ isolated nodes (that is,~$G_{n,n}$) and upper bounded
by the quantity for the graph~$G_{n,2}$, given that we
only consider graphs with at least two isolated nodes.
\end{proof}

We have the necessary ingredients to finish the proof of Theorem~\ref{theo-comparing-bpp-np-rp}.  Assume that there exists a $\bpp$ algorithm $\A$ for the
problem $\rshap(\twncnf)$. Then the input of $\A$ is a pair of Boolean
classifiers~$M,M'$ over a set of features~$X$ given as 2-NEG-CNF
formulas, and the task is to verify whether $\shap(M \wedge \neg x,\es,x)
> \shap(M' \wedge \neg y,\es',y)$, 
where $x,y$ are two
features not occurring in $X$, $\es$ is an entity over $X \cup \{x\}$
such that $\es(x) = 0$ and $\es(z) = 1$ for every $z \in X$, and
$\es'$ is an entity over $X \cup \{y\}$ such that $\es'(y) = 0$ and
$\es'(z) = 1$ for every $z \in X$. Using the amplification lemma of $\bpp$~
\cite[p.~231]{goldreich2008computational}, 
we can ensure that~$\A$ satisfies the following conditions:
\begin{itemize}
\item if $\shap(M \wedge \neg x,\es,x) > \shap(M' \wedge \neg y,\es',y)$, then
\begin{eqnarray*}
\Pr(\A(M,M') \text{ outputs {\bf yes}}) & \geq & \bigg(1-\frac{1}{4(\|M\| + \|M'\|)}\bigg).
\end{eqnarray*}

\item If $\shap(M \wedge \neg x,\es,x) \leq \shap(M' \wedge \neg y,\es',y)$, then
\begin{eqnarray*}
\Pr(\A(M,M') \text{ outputs {\bf no}}) & \geq & \bigg(1-\frac{1}{4(\|M\| + \|M'\|)}\bigg).
\end{eqnarray*}
\end{itemize}
Moreover, $\A(M,M')$ works in time $O(p(\|M\|
+ \|M'\|))$, where $p(\cdot)$ is a fixed polynomial, and $\|M\|$,
$\|M'\|$ are the sizes of $M$ and $M'$ represented as input strings,
respectively. By using $\bpp$ algorithm $\A$, we will define an algorithm
$\B$ for approximating, given a graph $G=(N,E)$ containing at least~$2$ isolated nodes, the value
$\shap(\neg\theta(G) \wedge \neg x, \es, x)$, where~$x$ is a feature
not occurring in $N$ and $\es$ is the entity over $N \cup \{x\}$ such
that $\es(x) = 0$ and $\es(z) = 1$ for every $z \in N$.
More precisely, we will show that $\B$
is a $\frac{9}{10}$-PRA for this problem. As
mentioned above, this concludes
the proof of Theorem~\ref{theo-comparing-bpp-np-rp} by Lemma~\ref{lem:nopras} and Lemma~\ref{lem:dist-to-comp}. We now define~$\B$.

Given a graph $G = (N,E)$ containing at least~$2$ isolated nodes, and assuming that $n = |N|$, algorithm~$\B(G)$ performs the following steps:
\begin{enumerate}
\item If~$n < 4$ then compute the exact value of $\shap(\neg \theta(G) \wedge \neg x,\es,x)$ by using an exhaustive approach.
\item For $t = 2$ to $n$:
\begin{enumerate}
\item \label{alg-mid-p} if $\A(\neg\theta(G),\neg\theta(G_{n,t})) = {\bf yes}$, then return
\begin{align*}
\frac{1}{2} \bigg(\shap(\neg \theta(G_{n,t-1}) \wedge \neg x, \es, x) + \shap(\neg \theta(G_{n,t}) \wedge \neg x, \es, x)\bigg)
\end{align*}
\end{enumerate}
\item  \label{alg-last-p} Return $\shap(\neg \theta(G_{n,n}) \wedge \neg x, \es, x)$.
\end{enumerate}
Algorithm $\B$ works in polynomial time since each graph $G_{n,t}$ can
be constructed in polynomial time in $n$, 2-NEG-CNF formulas
$\neg\theta(G)$ and $\neg\theta(G_{n,t})$ can be constructed in
polynomial time from graphs $G$ and $G_{n,t}$, algorithm
$\A(\neg\theta(G),\neg\theta(G_{n,t}))$ works in time
$O(p(\|\neg\theta(G)\|+ \|\neg\theta(G_{n,t})\|))$, and the value
returned in either Step \ref{alg-mid-p} or Step \ref{alg-last-p} can
be computed in polynomial time in $n$ by Lemma~\ref{lem-graph-k-p}.

As the final step of this proof, we need to show that:
\begin{align}\label{eq-to-prove-e-rpa}
\Pr\bigg(\big|\B(G) - \shap(\neg \theta(G) \wedge \neg x,\es,x)\big| \leq \frac{9}{10} \cdot \big|\shap(\neg \theta(G) \wedge \neg x,\es,x)\big|\bigg) \ \geq \ \frac{3}{4}. \quad
\end{align}
We can assume without loss of generality that~$n \geq 4$, since otherwise the algorithm
has computed the exact value.
Assume initially that every call $\A(\neg\theta(G),\neg\theta(G_{n,t}))$ in
algorithm~$\B$ returns the correct result (we will come back to this assumption
later).  Because the values~$\shap(\neg \theta(G_{n,t}) \wedge \neg x,\es,x)$
are strictly decreasing with~$t$ (this is clear by looking at the expression in
Lemma~\ref{lem:shap-cliques-again}), by Lemma~\ref{lem-l-u-b-Phi} we have that
either $\shap(\neg \theta(G) \wedge \neg x,\es,x) = \shap(\neg \theta(G_{n,n})
\wedge \neg x,\es,x)$ or there exists $t \in \{3, \ldots, n\}$ such that
$\A(\neg\theta(G),\neg\theta(G_{n,t})) = {\bf yes}$ and
$\A(\neg\theta(G),\neg\theta(G_{n,t-1})) = {\bf no}$
(notice that 
$\A(\neg\theta(G),\neg\theta(G_{n,2})) = {\bf no}$ since $G$ contains at least 2 isolated nodes). In the former case, the
third step of our algorithm ensures that we return the correct value, so let
us focus on the latter case. For some~$t \in \{3,\ldots,n\}$, we have~that:
\begin{multline}
\label{eq:ineq}
\shap(\neg \theta(G_{n,t}) \wedge \neg x, \es, x)\ 
< \ \shap(\lnot \theta(G) \land \neg x, \es, x) \ 
\leq \\ \shap(\neg \theta(G_{n,t-1}) \wedge \neg x, \es, x).
\end{multline}
Therefore, by considering that the returned value in Step \ref{alg-mid-p} of invocation~$\B(G)$ is the 
\rev{average of the left and right terms of the above interval}, we obtain that:
\begin{multline*}
 \big|\B(G) - \shap(\neg\theta(G) \wedge \neg x,\es,x)\big| \ \leq \\ \frac{1}{2} \bigg(\shap(\neg \theta(G_{n,t-1}) \wedge \neg x, \es, x) - \shap(\neg \theta(G_{n,t}) \wedge \neg x, \es, x) \bigg).
\end{multline*}
Now, we have by Lemma \ref{lem-graph-k-p} that:
\begin{align*}
\frac{1}{2} \bigg(\shap(\neg \theta(&G_{n,t-1}) \wedge \neg x, \es, x) - \shap(\neg \theta(G_{n,t}) \wedge \neg x, \es, x) \bigg) \ = \\
& \frac{t-1}{n(n+1)2^{n+1}} + \frac{1}{t 2^{t+1}} - \frac{t}{n(n+1)2^{n+1}} - \frac{1}{(t+1)2^{t+2}} \ = \\
& \frac{1}{t 2^{t+1}} - \frac{1}{(t+1)2^{t+2}} - \frac{1}{n(n+1)2^{n+1}} \ \leq \\
& \frac{1}{t 2^{t+1}} - \frac{1}{(t+1)2^{t+2}} \ = \\
& \frac{t+2}{2t} \cdot \frac{1}{(t+1)2^{t+1}}
\end{align*}
But notice that~$\frac{t+2}{2t} = \frac{1}{2} + \frac{1}{t} \leq \frac{9}{10}$ since~$t \geq 3$. Therefore we have
\begin{eqnarray}
\notag \big|\B(G) - \shap(\neg\theta(G) \wedge \neg x,\es,x)\big| & \leq & \frac{9}{10} \cdot \frac{1}{(t+1)2^{t+1}} \\
\notag & \leq & \frac{9}{10} \cdot \bigg(\frac{t}{n(n+1)2^n} + \frac{1}{(t+1)2^{t+1}} \bigg) \\
\notag & = & \frac{9}{10} \cdot \shap(\neg\theta(G_{n,t}) \wedge \neg x,\es,x) \\
\label{eq-s-d-B-dnwtp} & \leq & \frac{9}{10} \cdot \shap(\neg\theta(G) \wedge \neg x,\es,x),
\end{eqnarray}
where the last inequality comes from Equation~\eqref{eq:ineq}.

To finish with the proof, we need to remove the assumption that every
call\linebreak $\A(\neg\theta(G),\neg\theta(G_{n,t}))$ in algorithm~$\B$ returns the
correct result, and instead compute a lower bound on the probability
that this assumption is true. More precisely, we need to show that the
probability that every call $\A(\neg\theta(G)$, $\neg\theta(G_{n,t}))$ in
algorithm~$\B$ returns the correct result is at least $\frac{3}{4}$,
as this lower bound together with \eqref{eq-s-d-B-dnwtp} imply
that \eqref{eq-to-prove-e-rpa} holds (given that $\shap(\neg\theta(G) \wedge \neg x,\es,x) = |\shap(\neg\theta(G) \wedge \neg x,\es,x)|$ by Lemma \ref{lem-l-u-b-Phi}).
For every $t \in \{2, \ldots, n\}$, we have that:
\begin{eqnarray*}
\Pr(\A(\neg\theta(G),\neg\theta(G_{n,t})) \text{ outputs the correct result}) & \geq & \bigg(1-\frac{1}{4(\|M\| + \|M'\|)}\bigg)\\
& \geq & \bigg(1-\frac{1}{4n}\bigg)
\end{eqnarray*}
since $\|M\|+\|M'\| \geq n$. Hence,
\begin{eqnarray*}
\Pr\bigg(\bigwedge_{t=2}^{n} \big(\A(\neg\theta(G),\neg\theta(G_{n,t})) \text{ outputs the correct result}\big)\bigg) & \geq & \bigg(1-\frac{1}{4n}\bigg)^{n-1}\\
& \geq & \bigg(1-\frac{1}{4n}\bigg)^{n}
\end{eqnarray*}
But it is well known that the function~$f:x \mapsto \big(1 - \frac{1}{4x}\big)^x$ is strictly increasing for~$x \geq 1$, and that~$f(4) > \frac{3}{4}$. Hence, since we assumed~$n \geq 4$, we obtain 
\begin{eqnarray*}
\Pr\bigg(\bigwedge_{t=2}^{n} \big(\A(\neg\theta(G),\neg\theta(G_{n,t})) \text{ outputs the correct result}\big)\bigg) & \geq & \frac{3}{4}.
\end{eqnarray*}
This concludes the proof of Theorem~\ref{theo-comparing-bpp-np-rp}.

%% file: discussion.tex
Our algorithm for computing the~$\shap$-score could be used in practical scenarios.
Indeed,  some classes of classifiers can be compiled
into tractable Boolean circuits.  This is the case, for instance, of Bayesian
Classifiers~\citep{shih2018symbolic}, Binary Neural
Networks~\citep{DBLP:conf/kr/ShiSDC20}, and Random
Forests~\citep{DBLP:journals/corr/abs-2007-01493}.  The idea is to start with a
Boolean classifier~$M$ given in a formalism that is hard to interpret -- for
instance a binary neural network -- and to compute a tractable Boolean
circuit~$M'$ that is equivalent to~$M$ (this computation can be expensive). One
can then use~$M'$ and the nice properties of tractable Boolean circuits to
explain the decisions of the model. Hence, this makes it possible to apply
the results in this paper on the $\shap$-score to those classes of classifiers.

A Boolean circuit representing (or compiled from) another opaque classifier, as those just mentioned, is also more interpretable \citep{rudin19}, adding a useful property to that of tractability.  Still, it would be interesting to compare the explanations obtained from the compiled circuit with those that can be obtain directly from, say, the original neural network. After all, the $\shap$-score can be applied to both. First and recent experiments of this kind are reported in \citet{DBLP:journals/corr/abs-2303-06516}. We should mention, however, that there are some recent methods to explain the results from a neural network that do not rely only on its input/output relation (c.f. \citep{DBLP:journals/pieee/SamekMLAM21} for a recent and thorough review). Quite recent approaches to the identification and modeling of causal structures in deep learning should open the door for new methodologies for interpretation and explanation \citep{DBLP:journals/corr/abs-2010-04296,DBLP:journals/corr/abs-2102-11107}.

To a large extent, explanations in ML are based on different forms of counterfactual- or causality-based approaches; and the concept of explanation as such is left rather implicit.  However, explanations have been treated in more explicit terms in many disciplines, and, in particular, in AI, under {\em consistency-based} and {\em abductive} diagnosis, the main two forms of {\em model-based diagnosis} \citep{DBLP:reference/fai/Struss08}. These are explanations in knowledge representation  with open models. Recent progress has been made in applying model-based diagnosis in Explainable ML. C.f. \cite{DBLP:journals/corr/abs-2303-02829} and \cite{DBLP:journals/corr/abs-2211-00541} for technical details and references.
However, a deeper investigation of this connection in the context of explanations for classification  is open.

\rev{We conclude this discussion by mentioning a few research directions that
our work opens. First, it would be interesting to extend our tractability
results of Sections~\ref{sec:shapscore-d-Ds} and~\ref{sec:nonbinary} to a
setting that could allow for \emph{continuous} variables, instead of the
discrete variables that we have considered so far. A starting point for such an
investigation could for instance be the framework of \emph{probabilistic
circuits} proposed by~\cite{probcircuitstutorial}: these are data structures that can be used to
represent probability distributions (possibly with continuous variables) so as to ensure the tractability of certain tasks, such as expectation computation or probabilistic inference. Since these circuits are
very similar in nature to the deterministic and decomposable circuit classes that we
consider here, it seems that they could also be used for the
SHAP-score.  Another natural direction would be to study the complexity of
approximately computing the SHAP-score under so-called \emph{empirical
distributions}, as is done in~\cite{VdBAAAI21,van2022tractability} for the case
of exact computation where they show that it is intractable in general. For
instance, is it the case that there is still no FPRAS for approximating the
SHAP-score of DNF formulas (or monotone DNFs) under these distributions?  Last,
we leave open the question of finding a natural class of Boolean classifiers
for which we could efficiently approximate the SHAP-score (via an FPRAS), or
efficiently compare the SHAP-score of different features, whereas computing
this score exactly would be intractable. As we mentioned in Section~\ref{sec:approx},
given that exact model counting for DNF formulas is intractable but has an
FPRAS, and given the strong connections between model counting and the SHAP-score
established in Section~\ref{sec:limits} or by~\cite{VdBAAAI21,van2022tractability}, DNF formulas were a natural candidate
for this. But, quite surprisingly, our non-approximability results indicate
that this is not the case.}

%% file: app_KC.tex
In this section we explain how FBDDs and binary decision trees can be encoded in linear time as
deterministic and decomposable Boolean circuits.
First we need to define these formalisms.

\paragraph*{Binary Decision Diagrams}
A \emph{binary decision diagram} (BDD) over a set of variables $X$ is
a rooted directed acyclic graph~$D$ such that: (i) each internal node
is labeled with a variable from~$X$, and has exactly two outgoing
edges: one labeled 0, the other one labeled 1; and (ii) each leaf is
labeled either~$0$ or~$1$.  Such a BDD represents a Boolean classifier
in the following way.  Let~$\es$ be an entity over~$X$, and let
$\pi_\es = u_1, \ldots, u_m$ be the unique path in~$D$ satisfying the
following conditions: (a)~$u_1$ is the root of~$D$; (b)~$u_m$ is a
leaf of~$D$; and (c) for every~$i \in \{1, \ldots, m-1\}$, if the
label of~$u_i$ is~$x \in X$, then the label of the edge~$(u_i,
u_{i+1})$ is equal to~$\es(x)$.  Then the {\em value of~$\es$ in~$D$},
denoted by~$D(\es)$, is defined as the label of the
leaf~$u_m$. Moreover, a binary decision diagram~$D$ is \emph{free}
(FBDD) if for every path from the root to a leaf, no two nodes on that
path have the same label, and {\em a binary decision tree} is an FBDD
whose underlying graph is a tree.

\paragraph*{Encoding FBDDs and binary decision trees into deterministic and decomposable Boolean circuits (Folklore).}
Given an FBDD~$D$ over a set of variables~$X$, we explain how~$D$ can
be encoded as a deterministic and decomposable Boolean circuit~$C$
over~$X$. Notice that the technique used in this example also apply
to binary decision trees, as they are a particular case of FBDDs.  The
construction of~$C$ is done by traversing the structure of~$D$ in a
bottom-up manner. In particular, for every node~$u$ of~$D$, we
construct
a deterministic and decomposable circuit~$\alpha(u)$ that is
equivalent to the FBDD represented by the subgraph of~$D$ rooted
at~$u$. More precisely, for a leaf~$u$ of~$D$ that is labeled
with~$\ell \in \{0,1\}$, we define~$\alpha(u)$ to be the Boolean
circuit consisting of only one constant gate with label
$\ell$.
For an internal node~$u$ of~$D$ labeled with
variable~$x \in X$, let~$u_0$ and~$u_1$ be the nodes that we reach
from~$u$ by following the~$0$- and~$1$-labeled edge, respectively.
Then~$\alpha(u)$ is the Boolean circuit
depicted in the following figure:
\begin{center}
\begin{tikzpicture}
  \node[circ, minimum size=7mm] (n1) {{$\lor$}};
  \node[circw, below=5mm of n1] (n2) {};
  \node[circ, left=12mm of n2, minimum size=7mm] (n3) {{$\land$}}
  edge[arrout] (n1);
  \node[circ, right=12mm of n2, minimum size=7mm] (n4) {{$\land$}}
  edge[arrout] (n1);
  \node[circw, below=5mm of n3] (n5) {};
  \node[circ, left=4mm of n5, minimum size=7mm] (n6) {{$\neg$}}
  edge[arrout] (n3);
  \node[circw, right=4mm of n5, inner sep=-2] (n7) {{$\alpha(u_0)$}}
  edge[arrout] (n3);
  \node[circw, below=5mm of n4] (n8) {};
  \node[circ, left=4mm of n8, minimum size=7mm] (n9) {{$x$}}
  edge[arrout] (n4);
  \node[circw, right=4mm of n8, inner sep=-2] (n10) {{$\alpha(u_1)$}}
  edge[arrout] (n4);
  \node[circ, below=5mm of n6, minimum size=7mm] (n11) {{$x$}}
  edge[arrout] (n6);

\end{tikzpicture}
\end{center}

It is clear that the circuit that we obtain is equivalent to the input FBDD. We now argue that this circuit is deterministic and decomposable.
For the~$\lor$-gate shown in the figure, if an entity~$\es$ is
accepted by the Boolean circuit in its left-hand size, then~$\es(x) =
0$, while if an entity~$\es$ is accepted by the Boolean circuit in its
right-hand size, then~$\es(x) = 1$. Hence, we have that this
$\lor$-gate is deterministic, from which we conclude that~$\alpha(u)$
is deterministic, as~$\alpha(u_0)$ and~$\alpha(u_1)$ are also
deterministic by construction. Moreover, the~$\land$-gates shown in
the figure are decomposable as variable~$x$ is mentioned neither
in~$\alpha(u_0)$ nor in~$\alpha(u_1)$: this is because~$D$ is a
\emph{free} BDD. Thus, we conclude that $\alpha(u)$ is decomposable,
as~$\alpha(u_0)$ and~$\alpha(u_1)$ are decomposable by
construction. Finally, if~$u_{\mathsf{root}}$ is the root of~$D$, then
by construction we have that $\alpha(u_{\mathsf{root}})$ is a
deterministic and decomposable Boolean circuit equivalent to~$D$.
Note that this encoding can trivially be done in linear time.  Thus, we
often say, by abuse of terminology, that ``FBDDs (or binary decision trees)
are restricted kinds of deterministic and decomposable circuits".

%% file: app_pras-distinguishes-zeros.tex
\begin{fact}[Folklore]
\label{fact:folk}
Let~$f$ be a function that admits an $\varepsilon$-PRA for $\varepsilon \in
(0,1)$. Then the problem of determining, given a string $x$,
whether~$f(x)=0$ is in \bpp.
\end{fact}
\begin{proof}
Let $\mathcal{A}$ be an $\varepsilon$-PRA for~$f$, that is, a randomized algorithm that
takes as input
a string~$x$, and computes in polynomial time
in the size of $x$ a value
$\mathcal{A}(x)$ such that ($\star$)~$\Pr\big(|f(x) - \mathcal{A}(x)| \, \leq \, \varepsilon \, |f(x)| \big)  \geq \frac{3}{4}$.
We claim that the following algorithm is a \bpp\ algorithm for deciding if~$f(x)=0$: compute~$\mathcal{A}(x)$, and if this is equal to zero
then accept, otherwise reject. We will assume without loss of generality that~$f(x) \geq 0$, as the case~$f(x) \leq 0$ can be handled in the same way.
Observe that ($\star$) can be equivalently rewritten as
($\dagger$)~$\Pr\big((1-\varepsilon) f(x) \leq \A(x) \leq (1+\varepsilon) f(x) \big)  \geq \frac{3}{4}$.
But then:
\begin{itemize}
	\item Assume first that~$f(x)=0$. Then by ($\dagger$) we have
that~$\Pr\big(\mathcal{A}(x)=0\big)  \geq \frac{3}{4}$ as well, so
that we accept with probability at least~$\frac{3}{4}$. 
	\item Assume now that~$f(x)\neq 0$. Then ($\dagger$) gives us
	$\Pr\big(\A(x) \geq  (1-\varepsilon) f(x) > 0\big) \geq \frac{3}{4}$, so that we reject with probability at least~$\frac{3}{4}$.
\end{itemize}
This concludes the proof.
\end{proof}

%% file: app_tech_lem.tex
We notice that the inequality holds if and only if
\begin{eqnarray*}
\frac{1+\varepsilon}{1-\varepsilon} \cdot (n^3+1) \cdot 2^{n} & < & 2^{mn^2 - 2\lfloor \frac{m}{3} \rfloor n^2}.
\end{eqnarray*}
Moreover, given that $m \geq 1$, we have that:
\begin{eqnarray*}
2^{mn^2 - 2\lfloor \frac{m}{3} \rfloor n^2} & \geq & 2^{mn^2 - \frac{2m}{3}n^2}\\
& = & 2^{\frac{m}{3}n^2}\\
& \geq & 2^{\frac{n^2}{3}}.
\end{eqnarray*}
Therefore, we conclude that the inequality in the statement of the lemma holds from the fact that there exists $z_0$ such that:
\begin{eqnarray*}
\frac{1+\varepsilon}{1-\varepsilon} \cdot (z^3+1) \cdot 2^{z} \ < \ 2^{\frac{z^2}{3}} \quad \quad \quad \text{for every } z \geq z_0.
\end{eqnarray*}
\rev{This is indeed clear, as the dominating factor of the left term is~$2^{O(z)}$ whereas
that of the the right term is~$2^{O(z^2)}$.}

%% file: app_diff_shap.tex
We prove the more general claim that for any Boolean classifier~$M$ over features~$X$, probability
distribution~$\mathcal{D}$ over~$\eset(X)$, entity~$\es \in \eset(X)$ and
features~$x,y\in X$, we have
\begin{multline*}
\shap_\mathcal{D}(M,\es,x) - \shap_\mathcal{D}(M,\es,y) \ = \\
\sum_{S \subseteq X \setminus \{x,y\}} \frac{|S|!\, (|X|-|S|-2)!}{(|X|-1)!} \big(\phi_\mathcal{D}(M, \es,S \cup \{x\}) - \phi_\mathcal{D}(M, \es,S \cup \{y\})\big).
\end{multline*}
Indeed, we have:
\begin{multline*}
\shap_\mathcal{D}(M,\es,x) - \shap_\mathcal{D}(M,\es,y) \ = \\
\sum_{S \subseteq X \setminus \{x\}} \frac{|S|!\, (|X|-|S|-1)!}{|X|!} \big(\phi_\mathcal{D}(M, \es,S \cup \{x\}) - \phi_\mathcal{D}(M, \es,S )\big)\\
- \sum_{S \subseteq X \setminus \{y\}} \frac{|S|!\, (|X|-|S|-1)!}{|X|!} \big(\phi_\mathcal{D}(M, \es,S \cup \{y\}) - \phi_\mathcal{D}(M, \es,S )\big)\\
\end{multline*}

We cut the first sum in half by considering (1) those~$S$ that are~$\subseteq X
\setminus \{x,y\}$ and (2) those~$S\subseteq X$ such that with~$y \in S$ and~$x
\notin S$; and do the same for the second sum. We end up with
\begin{multline*}
\shap_\mathcal{D}(M,\es,x) - \shap_\mathcal{D}(M,\es,y) \ = \\
\sum_{S \subseteq X \setminus \{x,y\}} \frac{|S|!\, (|X|-|S|-1)!}{|X|!} \big(\phi_\mathcal{D}(M, \es,S \cup \{x\}) - \phi_\mathcal{D}(M, \es,S \cup \{y\})\big)\\
+ \sum_{S \subseteq X \setminus \{x,y\}} \frac{(|S|+1)!\, (|X|-|S|-2)!}{|X|!} \big(\phi_\mathcal{D}(M, \es,S \cup \{x,y\}) - \phi_\mathcal{D}(M, \es,S \cup \{y\})\big)\\
- \sum_{S \subseteq X \setminus \{x,y\}} \frac{(|S|+1)!\, (|X|-|S|-2)!}{|X|!} \big(\phi_\mathcal{D}(M, \es,S \cup \{x,y\}) - \phi_\mathcal{D}(M, \es,S \cup \{x\})\big)\\
= \sum_{S \subseteq X \setminus \{x,y\}} \frac{|S|!\, (|X|-|S|-1)!}{|X|!} \big(\phi_\mathcal{D}(M, \es,S \cup \{x\}) - \phi_\mathcal{D}(M, \es,S \cup \{y\})\big)\\
+ \sum_{S \subseteq X \setminus \{x,y\}} \frac{(|S|+1)!\, (|X|-|S|-2)!}{|X|!} \big(\phi_\mathcal{D}(M, \es,S \cup \{x\}) - \phi_\mathcal{D}(M, \es,S \cup \{y\})\big)\\
= \sum_{S \subseteq X \setminus \{x,y\}} \frac{|S|!\, (|X|-|S|-2)!}{(|X|-1)!} \big(\phi_\mathcal{D}(M, \es,S \cup \{x\}) - \phi_\mathcal{D}(M, \es,S \cup \{y\})\big).
\end{multline*}
This concludes the proof.

%% file: app_sum_prod.tex
To ease the proof, we let $s = n-t$. Then we want to compute 
\begin{eqnarray*}
\sum_{k=0}^{n-t} \frac{(n-t)! (n-k)!}{n! (n-t-k)!} & = & 
\sum_{k=0}^{s} \frac{s! (s+t-k)!}{(s+t)! (s-k)!}.
\end{eqnarray*}
We have that:
\begin{eqnarray}
\notag \sum_{k=0}^{s} \frac{s! (s+t-k)!}{(s+t)! (s-k)!} &=&
\frac{s!}{(s+t)!} \cdot \sum_{k=0}^{s}
\frac{(s+t-k)!}{(s-k)!}\\ \notag &=& \frac{s!t!}{(s+t)!} \cdot
\sum_{k=0}^{s} \frac{(s+t-k)!}{t!(s-k)!}\\ \notag &=&
\frac{s!t!}{(s+t)!} \cdot \sum_{k=0}^{s} \binom{s+t-k}{s-k}\\
\label{eq-sum-prod} &=& \frac{s!t!}{(s+t)!} \cdot \sum_{k=0}^{s} \binom{t+k}{k}.
\end{eqnarray}
Next, we show that the following equality holds by induction on $s$:
\begin{eqnarray}
\label{eq-sum-prod-binom}
  \sum_{k=0}^{s} \binom{t+k}{k} &=& \binom{s+t+1}{s}.
\end{eqnarray}
First, consider $s=0$:
\begin{align*}
  \sum_{k=0}^{s} \binom{t+k}{k} \ = \ \sum_{k=0}^{0} \binom{t+k}{k} \ = \ 
  \binom{t+0}{0} \ = \ 1 \ = \ \binom{0+t+1}{0} \ = \   \binom{s+t+1}{s}. 
\end{align*}  
Now, assuming that \eqref{eq-sum-prod-binom} holds, we have that:
\begin{eqnarray*}
  \sum_{k=0}^{s+1} \binom{t+k}{k} &=& \sum_{k=0}^{s} \binom{t+k}{k} + \binom{t+s+1}{s+1}\\ 
&=&  \binom{s+t+1}{s} + \binom{t+s+1}{s+1} \quad \quad \text{by \eqref{eq-sum-prod-binom}}\\ 
  &=&  \binom{s+t+2}{s+1} \quad \quad \text{by Pascal's rule}\\
  &=& \binom{s+1+t+1}{s+1}.
  \end{eqnarray*}
Hence, we conclude from \eqref{eq-sum-prod} and \ref{eq-sum-prod-binom} that:
\begin{eqnarray*}
  \sum_{k=0}^{s} \frac{s! (s+t-k)!}{(s+t)! (s-k)!} &=&
\frac{s!t!}{(s+t)!} \cdot \binom{s+t+1}{s}\\
&=& \frac{s!t!}{(s+t)!} \cdot \frac{(s+t+1)!}{s!(t+1)!}\\
&=& \frac{s+t+1}{t+1},
\end{eqnarray*}
and this is~$\frac{n+1}{t+1}$, as promised.